\documentclass[10pt,a4paper]{article}
\usepackage{hyperref}
\usepackage{graphicx}
\usepackage{amsmath}
\usepackage{amssymb}
\usepackage{booktabs}
\usepackage{subcaption}
\usepackage{algorithmic}
\usepackage{algorithm}
\usepackage{xcolor}
\usepackage{enumerate}
\usepackage{tabularx}
\usepackage{psfrag}
\usepackage{xcolor}
\usepackage{amsfonts}
\usepackage{subcaption}
\usepackage{booktabs}
\usepackage{multirow}
\usepackage{amsthm}
\usepackage{cite}		% Doit être chargé avant babel

\graphicspath{{figures/}}
\definecolor{bleu}{rgb}{0.2,0.2,0.7020}
\definecolor{bclair}{rgb}{0.6235    0.6235    0.8980}
\definecolor{rose}{rgb}{0.9882,0.7882,0.7882}

\theoremstyle{theorem}
\newtheorem{theorem}{Theorem}                                     
\newtheorem{claim}{Proposition}     

\theoremstyle{definition}
\newtheorem{definition}{Definition}

\theoremstyle{remark}
\newtheorem{remark}{Remark}
\newtheorem{hypothesis}{Assumption}

\newcommand{\nablaL}{\partial}
\date{April 2020}

\title{Automated data-driven selection of the hyperparameters for Total-Variation based texture segmentation\thanks{Work supported by Defi Imag'in SIROCCO and by ANR-16-CE33-0020 MultiFracs, France and by ANR GraVa ANR-18-CE40-0005.}.}

\author{Barbara Pascal\thanks{Univ Lyon, ENS de Lyon, Univ Lyon 1, CNRS, Laboratoire de Physique, F-69342 Lyon, France (\texttt{firstname.lastname@ens-lyon.fr}).} \and Samuel Vaiter\thanks{CNRS \& Universit\'e de Bourgogne Franche-Comt\'e, Dijon, France. (\texttt{samuel.vaiter@u-bourgogne.fr})} \and Nelly Pustelnik\footnotemark[2] \and Patrice Abry\footnotemark[2]}

\begin{document}
\maketitle

\section{Introduction}

Numerous problems in signal and image processing consist in finding the best possible estimate $\widehat{\boldsymbol{x}}$ of a quantity $\bar{\boldsymbol{x}} \in \mathcal{H}$ from an observation $\boldsymbol{y}\in \mathcal{G}$ (where $\mathcal{H}$ and $\mathcal{G}$ are Hilbert spaces isomorphic to $\mathbb{R}^N$ and $\mathbb{R}^P$ respectively), potentially corrupted by a linear operator $\boldsymbol{\Phi} : \mathcal{H} \rightarrow \mathcal{G}$, which encapsulates deformation or information loss, and by some additive zero-mean Gaussian noise $\boldsymbol{\zeta} \sim \mathcal{N}(\boldsymbol{0}_{P},\boldsymbol{\mathcal{S}})$, with \textit{known} covariance matrix $\boldsymbol{\mathcal{S}}\in \mathbb{R}^{P \times P}$, leading to the general observation model
\begin{align}
\label{eq:obs_gen_model}
\boldsymbol{y} = \boldsymbol{\Phi} \bar{\boldsymbol{x}} + \boldsymbol{\zeta}.
\end{align}
Examples resorting to inverse problems include image restoration~\cite{Cai_JF_2012_j-ams_ima_rtv, pustelnik2012_j-ieee-tsp_sur_ads}, inpainting~\cite{chan_T_2005_wiley_var_ii}, texture-geometry decomposition~\cite{Aujol_J_2006_ijcv_str_tid}, but also texture segmentation as recently proposed in \cite{pascal2019nonsmooth}. 
A widely investigated path for the estimation of underlying $\bar{\boldsymbol{x}}$ is linear regression~\cite{lawson1995solving,bjorck1996numerical}, providing an unbiased linear regression estimator $\widehat{\boldsymbol{x}}_{\mathrm{LR}}$.
Yet corresponding estimates suffer from large variances, which can lead to dramatic errors in the presence of noise~$\boldsymbol{\zeta}$~\cite{berger1976minimax}.

An alternative relies on the construction of \textit{parametric estimators}
\begin{align}
\begin{array}{ccc}
 \mathcal{G} \times \mathbb{R}^L & \longrightarrow & \mathcal{H} \\
\left( \boldsymbol{y}, \boldsymbol{\Lambda} \right) & \longmapsto & \widehat{\boldsymbol{x}}(\boldsymbol{y};\boldsymbol{\Lambda} )
\end{array}
\end{align}
allowing some estimation bias, and thus leading to drastic decrease of the variance.
Given some prior knowledge about ground truth $\bar{\boldsymbol{x}}$, e.g. ~\cite{golub1999tikhonov}, either sparsity of the variable $\bar{\boldsymbol{x}}$~\cite{tibshirani1996regression}, of its derivative~\cite{Tikhonov_A_1963_j-sov-mat-dok_tikhonov_ripp,rudin1992nonlinear,hansen1993use} or of its wavelet transform~\cite{donoho1994ideal}, one can build \textit{parametric estimators} performing a compromise between fidelity to the model~\eqref{eq:obs_gen_model} and structure constraints on the estimation.
In general, the compromise is tuned by a small number $L = \mathcal{O}(1)$ parameters, stored in a vector $\boldsymbol{\Lambda} \in \mathbb{R}^L$.
A very popular class of \textit{parametric estimators} relies on a penalization of a least squares data fidelity term formulated as a minimization problem
\begin{align}
\label{eq:LS_pen}
\widehat{\boldsymbol{x}}(\boldsymbol{y};\boldsymbol{\Lambda} ) \in \underset{{\boldsymbol{x}\in \mathcal{H}}}{\mathrm{Argmin}} \, \Vert \boldsymbol{y} - \boldsymbol{\Phi} {\boldsymbol{x}}\Vert_{\boldsymbol{\mathcal{W}}}^2
+  \lVert \textbf{U}_{\boldsymbol{\Lambda}} \boldsymbol{x} \rVert_{q}^q
\end{align}
with $\Vert\cdot \Vert_{\boldsymbol{\mathcal{W}}}$ the Mahalanobis distance associated with $\boldsymbol{\mathcal{W}}\in \mathbb{R}^{P\times P}$, defined as
\begin{align}
\Vert \boldsymbol{y} - \boldsymbol{\Phi} {\boldsymbol{x}}\Vert_{\boldsymbol{\mathcal{W}}} \triangleq \sqrt{\left( \boldsymbol{y} - \boldsymbol{\Phi} {\boldsymbol{x}} \right)^{\top}\boldsymbol{\mathcal{W}} \left( \boldsymbol{y} - \boldsymbol{\Phi} {\boldsymbol{x}} \right)}.
\end{align}
$\textbf{U}_{\boldsymbol{\Lambda}} : \mathcal{H} \rightarrow \mathcal{Q}$ is a linear operator parametrized by $\boldsymbol{\Lambda}$ and $\lVert \cdot \rVert_{q}$ the $\ell_q$-norm with $q\geq 1$ in Hilbert space $\mathcal{Q}$.\\\

\noindent \textbf{Least Squares.} While Ordinary Least Squares involve usual $\ell_2$ squared norm as data-fidelity term, that is $\boldsymbol{\mathcal{W}} = \boldsymbol{I}_{P}$, Generalized Least Squares~\cite{strutz2010data} make use of the covariance structure of the noise through $\boldsymbol{\mathcal{W}} = \boldsymbol{\mathcal{S}}^{-1}$, encapsulating all the observation statistics in the case of Gaussian noise. 
This generalized approach is equivalent to decorrelating the data and equalizing noise levels before performing the regression. 
Further, the Gauss-Markov theorem~\cite{aitkin1935least} asserts that minimizing Weighted Least Squares provides the best linear estimator of $\bar{\boldsymbol{x}}$, advocating for the use of Mahalanobis distance as data fidelity term in penalized Least Squares. 
Yet, in practice, Generalized Least Squares (or Weighted Least Squares in the case when $\boldsymbol{\mathcal{S}}$ is diagonal) requires not only the knowledge of the covariance matrix, but also to be able to invert it. 
For uncorrelated data, $\boldsymbol{\mathcal{S}}$ is diagonal and, provided that it is well-conditioned, it is easy to invert numerically.
On the contrary, computing $\boldsymbol{\mathcal{S}}^{-1}$ might be extremely challenging for correlated data since $\boldsymbol{\mathcal{S}}$ is not diagonal anymore and has a size scaling like the square of the dimension of $\mathcal{G}$. Thus, to handle possibly correlated Gaussian noise $\boldsymbol{\zeta}$, 
using Ordinary Least Squares is often mandatory, even though it does not benefit from same theoretical guarantees that Generalized Least Squares. 
Nevertheless, we will show that the knowledge of $\boldsymbol{\mathcal{S}}$ is far from being useless, since it is possible to take advantage of it when estimating the quadratic risk. \\

\noindent \textbf{Penalization.} Appropriate choice of $q$ and $\textbf{U}_{\boldsymbol{\Lambda}}$ covers a large variety of well-known estimators.
Linear filtering is obtained for $q = 2$~\cite{elden1977algorithms}, the shape of the filter being encapsulated in operator $\textbf{U}_{\Lambda}$~\cite{hansen1993use}, the hyperparameters $\boldsymbol{\Lambda}$ tuning e.g. its band-width.
It is very common in image processing to impose priors on the spatial gradients of the image, using the finite discrete horizontal and vertical difference operator $\textbf{D}$ and one regularization parameter $\boldsymbol{\Lambda} = \lambda > 0$ ($L=1$). 
For example, smoothness of the estimate is favored using $\ell_2$ squared norm, performing Tikhonov regularization~\cite{Tikhonov_A_1963_j-sov-mat-dok_tikhonov_ripp,hansen1993use}, in which $q = 2$ and $\lVert \textbf{U}_{\boldsymbol{\Lambda}} \boldsymbol{x} \rVert_{q}^q \triangleq \lambda \lVert \textbf{D} \boldsymbol{x}\rVert_2^2$. 
Another standard penalization is the anisotropic total variation~\cite{rudin1992nonlinear} $\lVert \textbf{U}_{\boldsymbol{\Lambda}} \boldsymbol{x} \rVert_{q}^q \triangleq \lambda \lVert \textbf{D} \boldsymbol{x}\rVert_1$, corresponding to $q=1$, where the $\ell_1$-norm enforces sparsity of spatial gradients.  \\

\noindent \textbf{Risk estimation.} 
The purpose of Problem~\eqref{eq:LS_pen} is to obtain a faithful estimation $\widehat{\boldsymbol{x}}(\boldsymbol{y};\boldsymbol{\Lambda} )$ of ground truth $\bar{\boldsymbol{x}}$, the error being measured by the so-called \textit{quadratic risk}
\begin{align}
\label{eq:intro_risk}
\mathbb{E}\lVert \textbf{B} \widehat{\boldsymbol{x}}(\boldsymbol{y};\boldsymbol{\Lambda} ) - \textbf{B} \bar{\boldsymbol{x}} \rVert_{\boldsymbol{\mathcal{W}}}^2
\end{align}
with $\textbf{B}$ a linear operator, which enables to consider various types of risk.
For instance, when $\textbf{B} = \boldsymbol{\Pi}$ is a projector on a subset of $\mathcal{H}$~\cite{eldar2008generalized}, the \textit{projected quadratic risk}~\eqref{eq:intro_risk} measures the estimation error on the projected quantity $\boldsymbol{\Pi}\bar{\boldsymbol{x}}$. This case includes the usual quadratic risk when $\textbf{B} = \boldsymbol{I}_{N}$. Conversely, when $\textbf{B} = \boldsymbol{\Phi}$, the risk~\eqref{eq:intro_risk} quantifies the quality of the prediction $\widehat{\boldsymbol{y}}(\boldsymbol{y};\boldsymbol{\Lambda}) \triangleq \boldsymbol{\Phi} \widehat{\boldsymbol{x}}(\boldsymbol{y};\boldsymbol{\Lambda})$ with respect to the noise-free observation $\bar{\boldsymbol{y}}\triangleq\boldsymbol{\Phi}\bar{\boldsymbol{x}}$ lying in $\mathcal{G}$, and is known as the \textit{prediction risk}.\\
The main issue is that one does not have access to ground truth $\bar{\boldsymbol{x}}$. Hence, measuring the quadratic risk~\eqref{eq:intro_risk} first requires to derive an estimator of 
$$
\mathbb{E} \left\lVert \textbf{B} \widehat{\boldsymbol{x}}(\boldsymbol{y} ; \boldsymbol{\Lambda}) - \textbf{B} \bar{\boldsymbol{x}} \right\rVert_{\boldsymbol{\mathcal{W}}}^2
$$ 
\textit{not involving} $\bar{\boldsymbol{x}}$. \\
This problem was handled originally in the case of independent, identically distributed, (i.i.d.) Gaussian linear model, that is for scalar covariance matrix $\boldsymbol{\mathcal{S}} = \rho^2 \boldsymbol{I}_{P}$, by Stein~\cite{stein1981estimation, tibshirani2015stein}, performing a clever integration by part, leading to Stein's Unbiased Risk Estimate (SURE)~\cite{donoho1995adapting,meyer2000degrees, ramani2008monte, tibshirani2012degrees}, initially formulated for the \textit{prediction risk},
\begin{align}
\label{eq:intro_SURE}
\left\lVert  \left( \boldsymbol{\Phi}\widehat{\boldsymbol{x}}(\boldsymbol{y} ; \boldsymbol{\Lambda}) - \boldsymbol{y} \right) \right\rVert_{\boldsymbol{\mathcal{W}}}^2 + 2\rho^2 \mathrm{Tr} \left( \partial_{\boldsymbol{y}}  \widehat{\boldsymbol{x}}(\boldsymbol{y} ; \boldsymbol{\Lambda}) \right) - P \rho^2,
\end{align}
whose expected value equals quadratic risk~\eqref{eq:intro_risk} with $\textbf{B} = \boldsymbol{\Phi}$.
In the past years SURE was intensively used both in statistical, signal and image processing applications~\cite{donoho1995adapting, blu2007sure, pesquet2009sure}.
It was recently extended to the case of independent but not identically distributed noise~\cite{chaux2008nonlinear,xie2012sure}, corresponding to diagonal covariance matrix $\boldsymbol{\mathcal{S}} = \mathrm{diag}(\sigma_1^2, \hdots, \sigma_P^2)$, and to the case when the noise is Gaussian with potential correlations, with very general covariance matrix $\boldsymbol{\mathcal{S}}$. 
Yet, to the best of our knowledge, very few numerical assessments are available for Gaussian noise with non-scalar covariance matrices. 
A notable exception is~\cite{chaux2008nonlinear}, in which numerical experiments are run on uncorrelated multi-component data, the components experiencing different noise levels. The noise being assumed independent, this corresponds to a diagonal covariance matrix $\boldsymbol{\mathcal{S}} = \mathrm{diag}(\rho_1^2, \hdots, \rho_P^2)$, with $\rho_i^2$ the variance of the noise of the $i^{\mathrm{th}}$ component. \\
Further, in the case when the noise is neither independent identically distributed nor Gaussian, Generalized Stein Unbiased Risk Estimators were proposed, e.g. for Exponential Families~\cite{hudson1978natural,eldar2008generalized} or Poisson noise~\cite{hudson1998maximum, luisier2010image, le2014unbiased}. \\
As for practical evaluation of Stein estimator, more sophisticated tools might be required to evaluate the second term of~\eqref{eq:intro_SURE}, notably when $\widehat{\boldsymbol{x}}(\boldsymbol{y} ; \boldsymbol{\Lambda})$ is obtained from a proximal splitting algorithm~\cite{bauschke2011convex,Combettes2011,parikh2014proximal,condat2019proximal} solving Problem~\eqref{eq:LS_pen}.
Indeed Stein estimator involves the Jacobian of $\widehat{\boldsymbol{x}}(\boldsymbol{y}; \boldsymbol{\Lambda})$ with respect to observations $\boldsymbol{y}$, which might not be directly accessible in this case.
In order to manage this issue, Vonesch~\textit{et~al} proposed in~\cite{vonesch2008recursive} to perform recursive forward differentiation inside the splitting scheme solving~\eqref{eq:LS_pen}, which benefits from few theoretical results from~\cite{evans2015measure}. This approach, even if remaining partially heuristic, proved to be efficient for a large class of problems~\cite{deledalle2014stein}.\\

\noindent \textbf{Hyperparameter tuning.}
Equation~\eqref{eq:LS_pen} clearly shows that the estimate $\widehat{\boldsymbol{x}}(\boldsymbol{y};\boldsymbol{\Lambda}) $ drastically depends on the choice of regularization parameters $\boldsymbol{\Lambda}$. 
Thus, fine-tuning of regularization parameters is a long-standing problem in signal and image processing.
A common formulation of this problem consists in minimizing the quadratic risk with respect to regularization parameters $\boldsymbol{\Lambda}$, solving:
\begin{equation}
\label{eq:min_risk}
\underset{\boldsymbol{\Lambda}}{\mathrm{minimize}} \, \,   \mathbb{E} \left\lVert \textbf{B} \widehat{\boldsymbol{x}}(\boldsymbol{y} ; \boldsymbol{\Lambda}) - \textbf{B} \bar{\boldsymbol{x}} \right\rVert_{\boldsymbol{\mathcal{W}}}^2.
\end{equation}
As emphasized in~\cite{eldar2008generalized}, approximate solution of~\eqref{eq:min_risk} found selecting among the estimates $\left( \widehat{\boldsymbol{x}}(\boldsymbol{y};\boldsymbol{\Lambda})\right)_{\boldsymbol{\Lambda} \in \mathbb{R}^L}$ the one reaching lowest SURE~\eqref{eq:intro_SURE}, as proposed in pioneering work~\cite{stein1981estimation},
leads to lower mean square error than classical Maximum Likelihood approaches applied to Model~\eqref{eq:obs_gen_model}.\\
The most direct method solving~\eqref{eq:min_risk} consists in computing SURE~\eqref{eq:intro_SURE} 
over a grid of parameters~\cite{donoho1994ideal,ramani2008monte, eldar2008generalized}, and to select the parameter of the grid for which SURE is minimal.
Yet, grid search methods suffers from a high computation cost for several reasons.
First of all, the size of the grid scaling algebraically with the number of regularization parameters $L$, exhaustive grid search is often inaccessible. 
Recently, random strategies were proposed to improve grid search efficiency~\cite{bergstra2012random}.
Yet, for $L \geq 3$, it remains very challenging if not unfeasible.
Further, an additional difficulty might appear in the case when $\widehat{\boldsymbol{x}}(\boldsymbol{y} ; \boldsymbol{\Lambda})$ is obtained from a splitting algorithm solving Problem~\eqref{eq:LS_pen}. 
Indeed when the regularization term $\lVert \textbf{U}_{\boldsymbol{\Lambda}} \boldsymbol{x} \rVert_q^q$ is nonsmooth, the proximal algorithms solving~\eqref{eq:LS_pen} suffers from slow convergence rate, making the evaluation of Stein estimator at each point of the grid very time consuming. 
Although accelerated schemes were proposed~\cite{beck2009fast, chambolle2011first}, grid search with $L \geq 2$ remains very costly, preventing from practical use.\\
When a closed-form expression of Stein estimator is available, exact function minimization over the regularization parameters $\boldsymbol{\Lambda}$ might be possible. 
This is the case for instance for the Tikhonov penalization for which Thompson~\textit{et~al.}~\cite{thompson1991study}, Galatsanos~\textit{et~al.} in~\cite{galatsanos1992methods}, and Desbat~\textit{et~al.}~in~\cite{desbat1995minimum} took advantage of the linear closed-form expression of $\widehat{\boldsymbol{x}}(\boldsymbol{y};\boldsymbol{\Lambda} ) $ to find the ``best" regularization parameter, i.e. to solve~\eqref{eq:min_risk}. 
Another well-known closed-form expression holds for soft-thresholding, which is widely used for wavelet-shrinkage denoising e.g.~\cite{donoho1994ideal,luisier2007new}.
Note  that Generalized Cross Validation~\cite{golub1979generalized} also makes use of closed-form expression for parameters tuning, but in a slightly different way, working on prediction risk, solving~\eqref{eq:min_risk} for $\textbf{B} = \boldsymbol{\Phi}$.
Generalized Cross Validation and Stein-based estimators were compared independently by Li~\cite{li1985stein}, Thompson~\cite{thompson1991study}, and Desbat~\textit{et~al.}~in~\cite{desbat1995minimum}.
Further, Bayesian methods were proposed to deal with very large number of hyperparameters $L \gg 1$, among which Sequential Model-Based Optimization (SMBO), providing smart sampling of the hyperparameter domain~\cite{bergstra2013making}. 
Such methods are particularly adapted to machine learning, as they manage huge amount of hyperparameters without requiring knowledge of the gradient of the cost function~\cite{bertr2020implicit}.
\\
In order to go further than (random) sampling methods, elaborated approaches relying on minimization schemes were proposed, requiring sufficiently smooth risk estimator, as well as access to its derivative with respect to $\boldsymbol{\Lambda}$. 
From a $C^{\infty}$ closed-form expression of Poisson Unbiased Risk Estimate, Deledalle~\textit{et~al.}~\cite{deledalle2010poisson} proposed a Newton algorithm solving~\eqref{eq:min_risk}. 
Nevertheless, it does not generalize, since it is very rare that one has access to all the derivatives of the risk estimator. 
In the case when the noise is Gaussian i.i.d., Chaux~\textit{et~al.}~\cite{chaux2008nonlinear} proposed and assessed numerically an empirical descent algorithm for automatic choice of regularization parameter, but with no convergence guarantee. 
For i.i.d. Gaussian noise and estimators built as the solution of~\eqref{eq:LS_pen}, Deledalle~\textit{et~al.}~\cite{deledalle2014stein} proposed sufficient conditions so that $\widehat{\boldsymbol{x}}(\boldsymbol{y};\boldsymbol{\Lambda})$ is differentiable with respect to $\boldsymbol{\Lambda}$, and then derived the differentiability of Stein's Unbiased Risk Estimate. 
Further, they elaborated a Stein Unbiased GrAdient estimator of the Risk (SUGAR) with the aim of performing a quasi-Newton descent solving~\eqref{eq:min_risk} using BFGS strategy. 
SUGAR proved its efficiency in the automated hyperparameter selection in a spatial-spectral deconvolution method for large multispectral data corrupted by i.i.d. Gaussian noise~\cite{ammanouil2019parallel}\\

\noindent \textbf{Contributions and outline.} 
We propose a Generalized Stein Unbiased GrAdient estimator of the Risk, for the case of Gaussian noise $\boldsymbol{\zeta}$ with \textit{any} covariance matrix $\boldsymbol{\mathcal{S}}$, using the framework of Ordinary Least Squares, that is~\eqref{eq:intro_risk} with $\boldsymbol{\mathcal{W}} = \boldsymbol{I}_{P}$, enabling to manage different noise levels and correlations in the observed data. \\
Section~\ref{sec:SURE} revisits Stein's Unbiased Estimator of the Risk in the particular case of correlated Gaussian noise with covariance matrix $\boldsymbol{\mathcal{S}}$ and derives the Finite Difference Monte Carlo SURE for this framework, extending
~\cite{deledalle2014stein}.
Further, we include a projection operator $\textbf{B} = \boldsymbol{\Pi}$ making the model versatile enough to fit various applications.\\
In this context, Finite Difference Monte Carlo SURE is differentiated with respect to regularization parameters leading to Finite Difference Monte Carlo Generalized Stein Unbiased GrAdient estimator of the Risk, whose asymptotic unbiasedness is demonstrated in~Section~\ref{sec:SUGAR}.
Generalized Stein Unbiased Risk Estimate and Generalized Stein Unbiased GrAdient estimate of the Risk are embedded in a quasi-Newton optimization scheme for automatic parameters tuning, presented in Section~\ref{subsec:BFGS}. 
Moreover, the case of sequential estimators is discussed in Section~\ref{subsec:seq_est}. \\
Then, in Section~\ref{sec:text_seg}, the entire proposed procedure is particularized to an original application to texture segmentation based on a wavelet (multiscale) estimation of fractal attributes, proposed in~\cite{pascal2018joint,pascal2019nonsmooth}. 
The texture model is cast into the general formulation~\eqref{eq:obs_gen_model}, $\boldsymbol{y}$ corresponding to a nonlinear multiscale transform of the image to be segmented. Hence the noise $\boldsymbol{\zeta}$ presents both inter-scale and intra-scale correlations, leading to a non-diagonal covariance matrix $\boldsymbol{\mathcal{S}}$.
Both Stein Unbiased Risk Estimate and Stein Unbiased GrAdient estimate of the Risk are evaluated with a Finite Difference Monte Carlo strategy, all steps of which are made explicit for the texture segmentation problem.\\
Finally, Section~\ref{sec:numerics} is devoted to exhaustive numerical simulations assessing the performance of the proposed texture segmentation with automatic regularization parameters tuning. We notably emphasize the importance of taking into account the \textit{full} covariance structure into account in Stein-based approaches.\\

\section{Stein Unbiased Risk Estimate (SURE) with correlated noise}
\label{sec:SURE}

This Section details the extension of Stein Unbiased Risk Estimator~\eqref{eq:intro_SURE} when $\boldsymbol{\mathcal{W}} = \boldsymbol{I}_P$ to the case when observations evidence correlated noise, leading to the Finite Difference Monte Carlo Generalized Stein Unbiased Risk Estimator, $\widehat{R}_{\nu, \boldsymbol{\varepsilon}}(\boldsymbol{y}; \boldsymbol{\Lambda} \lvert \boldsymbol{\mathcal{S}})$, defined in~\eqref{eq:SURE_FDMC_def}.\\

\noindent \textbf{Notations.} For a linear operator $\boldsymbol{\Phi} : \mathcal{H} \rightarrow \mathcal{G}$, the \textit{adjoint} operator is denoted $\boldsymbol{\Phi}^*$ and characterized by: $\text{for every } \boldsymbol{x} \in \mathcal{H}, \, \text{and } \boldsymbol{y} \in \mathcal{G}, \, \, \langle \boldsymbol{y}, \boldsymbol{\Phi} \boldsymbol{x} \rangle = \langle \boldsymbol{\Phi}^* \boldsymbol{y}, \boldsymbol{x} \rangle$.\\
The Jacobian  with respect to observations $\boldsymbol{y}$ of a differentiable estimator $\widehat{\boldsymbol{x}}(\boldsymbol{y};\boldsymbol{\Lambda})$ is denoted $\partial_{\boldsymbol{\Lambda}} \widehat{\boldsymbol{x}}(\boldsymbol{y};\boldsymbol{\Lambda})$.
\subsection{Observation model}

In this work, we consider observations $\boldsymbol{y}$, supposed to follow Model~\eqref{eq:obs_gen_model}, as stated in Assumption~\ref{hyp:gauss_noise} with a degradation operator $\boldsymbol{\Phi}$ assumed to be full-rank, as stated in Assumption~\ref{hyp:full_rank}.

\begin{hypothesis}[Gaussianity]
\label{hyp:gauss_noise}
The additive noise $\boldsymbol{\zeta} \in \mathcal{G}$ is Gaussian: $\boldsymbol{\zeta} \sim \mathcal{N}\left(\boldsymbol{0}_{P}, \boldsymbol{\mathcal{S}} \right)$, where $\boldsymbol{0}_{P}$ is the null vector of $\mathcal{G}$ and $\boldsymbol{\mathcal{S}} \in \mathbb{R}^{P \times P}$ is the covariance matrix of the noise, where $P = \dim(\mathcal{G})$.
Thus, the density probability law associated with the model~\eqref{eq:obs_gen_model} writes
\begin{align}
\label{eq:gauss_noise}
\boldsymbol{y} \sim \frac{1}{\sqrt{(2\pi)^P \lvert \mathrm{det} \boldsymbol{\mathcal{S}} \rvert}} \exp \left( - \frac{\lVert \boldsymbol{y}-\boldsymbol{\Phi} \bar{\boldsymbol{x}} \rVert_{\boldsymbol{\mathcal{S}}^{-1}}^2}{2}\right).
\end{align}
\end{hypothesis}

\begin{hypothesis}[Full-rank]
\label{hyp:full_rank}
The linear operator $\boldsymbol{\Phi}: \mathcal{H} \rightarrow \mathcal{G}$ is full rank, or equivalently $\boldsymbol{\Phi}^* \boldsymbol{\Phi}$ is invertible.
\end{hypothesis}

\subsection{Estimation problem}

Let $\widehat{\boldsymbol{x}}(\boldsymbol{y};\boldsymbol{\Lambda})$ be a parametric estimator of ground truth $\bar{\boldsymbol{x}} \in \mathcal{H}$, defined in a unique manner from observations $\boldsymbol{y} \in \mathcal{G}$ and hyperparameters $\boldsymbol{\Lambda} \in \mathbb{R}^L$. 
\begin{remark}
\label{rq:full_rank_use}
For instance, $\widehat{\boldsymbol{x}}(\boldsymbol{y} ; \boldsymbol{\Lambda})$ can be the Penalized Ordinary Least Squares estimator, defined in~\eqref{eq:LS_pen}. 
In this case full-rank Assumption~\ref{hyp:full_rank} ensures the unicity of the minimizer. 
Nevertheless, we emphasize that Sections~\ref{sec:SURE}~and~\ref{sec:SUGAR} address Problem~\eqref{eq:min_risk} in a more general framework. 
\end{remark}

The possibility that the quantity of interest might be a projection of $\bar{\boldsymbol{x}}$ on a the subspace $\mathcal{I}$ of $\mathcal{H}$ is considered. 
One can think for instance of physics problems, in which only part of variables have a physical interpretation.
\begin{definition}
\label{def:Pi}
The linear operator $\boldsymbol{\Pi} : \mathcal{H} \rightarrow \mathcal{H}$ performs the orthogonal projection on subspace $\mathcal{I}$ capturing relevant information about $\bar{\boldsymbol{x}}$.
Moreover, from both Assumption~\ref{hyp:full_rank} and the projection operator $\boldsymbol{\Pi}$, we define the linear operator $\textbf{A} : \mathcal{G} \rightarrow \mathcal{H}$ as the composition 
\begin{align}
\label{eq:def_A}
\textbf{A} \triangleq \boldsymbol{\Pi} \left(\boldsymbol{\Phi}^{*}\boldsymbol{\Phi}\right)^{-1}\boldsymbol{\Phi}^{*}. 
\end{align}
\end{definition}

The risk is defined as the \textit{projected estimation} error made on the quantity of interest $\boldsymbol{\Pi} \bar{\boldsymbol{x}}$ by the estimator, measured via an ordinary squared $\ell_2$-norm.
\begin{align}
\label{eq:risk_def}
 R[\widehat{\boldsymbol{x}}]( \boldsymbol{\Lambda}) \triangleq \mathbb{E}_{\boldsymbol{\zeta}}  \left\lVert \boldsymbol{\Pi} \widehat{\boldsymbol{x}}(\boldsymbol{y} ; \boldsymbol{\Lambda}) - \boldsymbol{\Pi}\bar{\boldsymbol{x}} \right\rVert_2^2.
\end{align}

\begin{remark}
Another usual definition of the risk involves the inverse of the covariance matrix~\cite{eldar2008generalized} through a Mahalanobis distance writing
\begin{align}
\label{eq:risk_eldar}
R_{\mathrm{M}}[\widehat{\boldsymbol{x}}](\boldsymbol{\Lambda}) \triangleq \mathbb{E}_{\boldsymbol{\zeta}}  \left\lVert \boldsymbol{\Pi} \widehat{\boldsymbol{x}}(\boldsymbol{y} ; \boldsymbol{\Lambda}) - \boldsymbol{\Pi}\bar{\boldsymbol{x}} \right\rVert_{\boldsymbol{\mathcal{S}}^{-1}}^2.
\end{align}
requiring the knowledge of $\boldsymbol{\mathcal{S}}^{-1}$, which might be non-trivial or even inaccessible for correlated noise presenting non-diagonal covariance matrix. Hence our approach uses exclusively ordinary quadratic risk defined in~\eqref{eq:risk_def}. Nevertheless, these two approaches, even though being different, shares interesting common points which will be mentioned briefly in the following (see Remark~\ref{rq:eldar_sure}).
\end{remark}

The aim of this work is automatic fine-tuning the regularization parameters $\boldsymbol{\Lambda}$ in order to minimize the ordinary risk~\eqref{eq:risk_def} defined above. 
Yet in practice, the optimal regularization parameters~$\boldsymbol{\Lambda}^\dagger$ satisfying
\begin{align}
\label{eq:risk_min}
\boldsymbol{\Lambda}^\dagger \in \underset{\boldsymbol{\Lambda} \in \mathbb{R}^L}{\mathrm{Argmin}} \,\, R[\widehat{\boldsymbol{x}}]( \boldsymbol{\Lambda}) 
\end{align}
is inaccessible. In the following, we propose a detailed procedure to closely approach $\boldsymbol{\Lambda}^\dagger$, by minimizing a Generalized Stein Unbiased Risk Estimator approximating $R[\widehat{\boldsymbol{x}}]( \boldsymbol{\Lambda})$.

\subsection{Generalized Stein Unbiased Risk Estimator}
\label{subsec:SURE}

The risk defined in~\eqref{eq:risk_def} depends explicitly on ground truth $\bar{\boldsymbol{x}}$ and hence is inaccessible. 
Stein proposed an unbiased estimator of this risk, known as Stein Unbiased Risk Estimator (SURE) in the case of i.i.d. Gaussian noise, recalled in Equation~\eqref{eq:intro_SURE}.
This estimator was then extended to very general noise distributions (see e.g.~\cite{eldar2008generalized} for Exponential Families, including Gaussian densities). 
In particular, when the noise $\boldsymbol{\zeta}$ is Gaussian, with possible non-trivial covariance matrix, Theorem~\ref{thm:SURE} provides a generalization of Stein's original estimator, which constitutes the starting point of this work. \\
Stein's approach for risk estimation crucially relies the following hypothesis on estimator $\widehat{\boldsymbol{x}}(\boldsymbol{y} ; \boldsymbol{\Lambda})$:
\begin{hypothesis}[Regularity and integrability]
\label{hyp:reg_int}
The estimator $\widehat{\boldsymbol{x}}(\boldsymbol{y};\boldsymbol{\Lambda})$ is continuous and weakly differentiable with respect to observations $\boldsymbol{y}$. Moreover, the quantities $\left\langle \textbf{A}^* \boldsymbol{\Pi}\widehat{\boldsymbol{x}}(\boldsymbol{y};\boldsymbol{\Lambda}), \boldsymbol{\zeta} \right\rangle$ and $\partial_{\boldsymbol{y}} \widehat{\boldsymbol{x}}(\boldsymbol{y};\boldsymbol{\Lambda})$ are integrable against the Gaussian density:
\begin{align*}
\frac{1}{\sqrt{(2\pi)^P \lvert \mathrm{det} \boldsymbol{\mathcal{S}} \rvert}} \exp \left( - \frac{\lVert \boldsymbol{\zeta} \rVert_{\boldsymbol{\mathcal{S}}^{-1}}^2}{2}\right) \, \mathrm{d}\boldsymbol{\zeta}.
\end{align*}
\end{hypothesis}
\begin{theorem}
\label{thm:SURE}
Consider Model~\eqref{eq:obs_gen_model}, together with Assumptions~\ref{hyp:gauss_noise}~(Gaussianity), \ref{hyp:full_rank}~(Full-rank), \ref{hyp:reg_int}~(Integrability), and linear operator $\textbf{A}$ defined in~\eqref{eq:def_A}.
Then generalized Stein's lemma applies, and leads to 
\begin{align}
\label{eq:SURE_thm}
R[\widehat{\boldsymbol{x}}]( \boldsymbol{\Lambda} )  = \mathbb{E}_{\boldsymbol{\zeta}} \left[\left\lVert  \textbf{A}\left( \boldsymbol{\Phi}\widehat{\boldsymbol{x}}(\boldsymbol{y} ; \boldsymbol{\Lambda}) - \boldsymbol{y} \right) \right\rVert_2^2 + 2 \mathrm{Tr} \left( \boldsymbol{\mathcal{S}}\textbf{A}^* \boldsymbol{\Pi} \partial_{\boldsymbol{y}}  \widehat{\boldsymbol{x}}(\boldsymbol{y} ; \boldsymbol{\Lambda}) \right) - \mathrm{Tr}(\textbf{A} \boldsymbol{\mathcal{S}} \textbf{A}^*  ) \right],
\end{align}
the quantity in the brackets being the so-called Generalized Stein Unbiased Risk Estimator.
\end{theorem}

\begin{proof}
A detailed proof is provided in Appendix~\ref{app:SURE}.
\end{proof}

\begin{remark}
\label{rq:eldar_sure}
Interestingly, when considering the squared Mahalanobis distance in the defintion of the risk~\eqref{eq:risk_eldar}, Stein Unbiased Risk Estimator has the same global structure, yet, instead of involving the covariance matrix $\boldsymbol{\mathcal{S}}$ it involves its inverse writing
\begin{align}
\widetilde{R}[\widehat{\boldsymbol{x}}]( \boldsymbol{\Lambda} )  = \mathbb{E}_{\boldsymbol{\zeta}} \left[\left\lVert  \textbf{A}\left( \boldsymbol{\Phi}\widehat{\boldsymbol{x}}(\boldsymbol{y} ; \boldsymbol{\Lambda}) - \boldsymbol{y} \right) \right\rVert_{\boldsymbol{\mathcal{S}}^{-1}}^2 + 2 \mathrm{Tr} \left( \textbf{A}^* \boldsymbol{\Pi} \partial_{\boldsymbol{y}}  \widehat{\boldsymbol{x}}(\boldsymbol{y} ; \boldsymbol{\Lambda}) \right) - \mathrm{Tr}(\textbf{A} \textbf{A}^*  ) \right].
\end{align}
\end{remark}

\subsection{Finite Difference Monte Carlo SURE}
\label{sec:est_dof}

In the proposed SURE expression~\eqref{eq:SURE_thm}, the quantity $ \mathrm{Tr} \left( \boldsymbol{\mathcal{S}}\textbf{A}^* \boldsymbol{\Pi} \partial_{\boldsymbol{y}}  \widehat{\boldsymbol{x}}(\boldsymbol{y};\boldsymbol{\Lambda}) \right)$ appearing in~\eqref{eq:SURE_thm}, called the \textit{degrees of freedom}, concentrates the major difficulties in computing Stein's estimator in data processing problems, as evidenced by the prolific literature addressing this issue in the case $\boldsymbol{\mathcal{S}} \propto \boldsymbol{I}_P$~\cite{kato2009degrees,tibshirani2012degrees,dossal2013degrees,vaiter2017degrees}.
Indeed, it involves the product of the $P \times P$ matrix $\boldsymbol{\mathcal{S}} \textbf{A}^* \boldsymbol{\Pi}$ with the $P \times P$ Jacobian matrix $\partial_{\boldsymbol{y}} ( \widehat{\boldsymbol{x}}(\boldsymbol{y}, \boldsymbol{\Lambda})) $. 
Not only the product of two $P \times P$ matrices might be extremely costly in computational efforts  but also the Jacobian matrix, because of its large size, $P \gg 1$, might also be very demanding to compute (or even to estimate).
Two-step Finite Difference Monte Carlo strategy together with Assumption~\ref{hyp:lip_L1} presented below, enable to overcome theses difficulties and to built a usable Stein Unbiased Risk Estimator, denoted $\widehat{R}_{\nu, \boldsymbol{\varepsilon}}(\boldsymbol{y}; \boldsymbol{\Lambda} \lvert \boldsymbol{\mathcal{S}}) $, defined in Equation~\eqref{eq:SURE_FDMC_def}. \\ 

\begin{hypothesis}[Lipschitzianity w.r.t. observations]
\label{hyp:lip_L1}
Let $\widehat{\boldsymbol{x}}(\boldsymbol{y};\boldsymbol{\Lambda})$ an estimator of $\bar{\boldsymbol{x}}$, depending on observations $\boldsymbol{y}$, and parametrized by $\boldsymbol{\Lambda}$.  \\ 
\textit{(i)}~The mapping $\boldsymbol{y} \mapsto \widehat{\boldsymbol{x}}(\boldsymbol{y};\boldsymbol{\Lambda})$ is uniformly $L_1$-Lipschitz .\\
\textit{(ii)}~$\forall \, \boldsymbol{\Lambda} \in \mathbb{R}^L, \, \widehat{\boldsymbol{x}}(\boldsymbol{0}_{P}; \boldsymbol{\Lambda}) = \boldsymbol{0}_{N}$, with $\boldsymbol{0}_{N}$ (resp. $\boldsymbol{0}_{P}$) the null vector of $\mathcal{H}$ (resp. $\mathcal{G}$).\\
\end{hypothesis}

\noindent \textbf{Step 1.}
\underline{Trace estimation via Monte Carlo:}\\
In the way to practical degrees of freedom estimation, the first step is to remark that it far less costly to compute the product of the $P \times P$ matrix $\boldsymbol{\mathcal{S}} \textbf{A}^* \boldsymbol{\Pi}$ with $\partial_{\boldsymbol{y}}  ( \widehat{\boldsymbol{x}}(\boldsymbol{y}, \boldsymbol{\Lambda})) [\boldsymbol{\varepsilon}] \in \mathbb{R}^P $, the Jacobian matrix applied on a vector $\boldsymbol{\varepsilon} \in \mathbb{R}^P$.
Further, straightforward computation shows that if $\boldsymbol{\varepsilon} \in \mathbb{R}^P$ is a normalized random variable $\boldsymbol{\varepsilon} \sim \mathcal{N}(\boldsymbol{0}_P, \boldsymbol{I}_P)$, and $\textbf{M} \in \mathbb{R}^{P\times P}$ any matrix, then 
\begin{align*}
\mathrm{Tr}(\textbf{M}) = \mathbb{E}_{\boldsymbol{\varepsilon}} \langle \textbf{M}\boldsymbol{\varepsilon}, \boldsymbol{\varepsilon}\rangle. 
\end{align*}
Thus, following the suggestion of~\cite{girard1989fast, ramani2008monte, deledalle2014stein}, if one has access to $\partial_{\boldsymbol{y}} ( \widehat{\boldsymbol{x}}(\boldsymbol{y}, \boldsymbol{\Lambda})) \left[ \boldsymbol{\varepsilon} \right]$, then, since $\boldsymbol{\mathcal{S}}$ is a covariance matrix and hence is symmetric,
\begin{align}
\label{eq:DOF_MC}
\mathrm{Tr} \left( \boldsymbol{\mathcal{S}}\textbf{A}^* \boldsymbol{\Pi} \partial_{\boldsymbol{y}}  \widehat{\boldsymbol{x}}(\boldsymbol{y} ; \boldsymbol{\Lambda})  \right) = \mathbb{E}_{\boldsymbol{\varepsilon}}\left\langle \boldsymbol{\mathcal{S}}\textbf{A}^* \boldsymbol{\Pi} \partial_{\boldsymbol{y}}  \widehat{\boldsymbol{x}}(\boldsymbol{y} ; \boldsymbol{\Lambda}) [\boldsymbol{\varepsilon}], \boldsymbol{\varepsilon} \right\rangle \nonumber \\
= \mathbb{E}_{\boldsymbol{\varepsilon}}\left\langle \textbf{A}^* \boldsymbol{\Pi} \partial_{\boldsymbol{y}}  \widehat{\boldsymbol{x}}(\boldsymbol{y} ; \boldsymbol{\Lambda}) [\boldsymbol{\varepsilon}], \boldsymbol{\mathcal{S}}\boldsymbol{\varepsilon} \right\rangle,
\end{align}
and
$
\left\langle \textbf{A}^* \boldsymbol{\Pi} \partial_{\boldsymbol{y}}  \widehat{\boldsymbol{x}}(\boldsymbol{y} ; \boldsymbol{\Lambda}) [\boldsymbol{\varepsilon}], \boldsymbol{\mathcal{S}}\boldsymbol{\varepsilon} \right\rangle
$
provides an estimator of degrees of freedom.\\

\vspace{2mm}
\noindent \textbf{Step 2.}
 \underline{First-order derivative estimation with Finite Differences:} \\
 Second step consists in tackling the problem of estimating $\partial_{\boldsymbol{y}} ( \widehat{\boldsymbol{x}}(\boldsymbol{y}, \boldsymbol{\Lambda})) \left[ \boldsymbol{\varepsilon} \right]$ when no direct access to the Jacobian $\partial_{\boldsymbol{y}} ( \widehat{\boldsymbol{x}}(\boldsymbol{y}, \boldsymbol{\Lambda})) $ is possible. In this case, the derivative can be estimated using the normalized random variable $\boldsymbol{\varepsilon}$ and a step $\nu>0$ making use of Taylor expansion 
\begin{align*}
&\widehat{\boldsymbol{x}}(\boldsymbol{y} + \nu \boldsymbol{\varepsilon}; \boldsymbol{\Lambda}) - \widehat{\boldsymbol{x}}(\boldsymbol{y} ; \boldsymbol{\Lambda})  \underset{\nu \rightarrow 0}{\simeq}  \partial_{\boldsymbol{y}} (\widehat{\boldsymbol{x}}(\boldsymbol{y}; \boldsymbol{\Lambda}))  \left[ \nu\boldsymbol{\varepsilon} \right] \\
\Longleftrightarrow  \quad & \partial_{\boldsymbol{y}}  \widehat{\boldsymbol{x}}(\boldsymbol{y} ; \boldsymbol{\Lambda}) \left[ \boldsymbol{\varepsilon} \right] = \lim\limits_{\nu \rightarrow 0} \frac{1}{\nu} \left( \widehat{\boldsymbol{x}}(\boldsymbol{y} + \nu \boldsymbol{\varepsilon}; \boldsymbol{\Lambda}) - \widehat{\boldsymbol{x}}(\boldsymbol{y} ; \boldsymbol{\Lambda})  \right).
\end{align*}

It follows
\begin{align}
\mathrm{Tr} \left( \boldsymbol{\mathcal{S}}\textbf{A}^* \boldsymbol{\Pi} \partial_{\boldsymbol{y}}  \widehat{\boldsymbol{x}}(\boldsymbol{y} ; \boldsymbol{\Lambda}) \right) &= \mathbb{E}_{\boldsymbol{\varepsilon}} \lim\limits_{\nu \rightarrow 0} \frac{1}{\nu} \left\langle \boldsymbol{\mathcal{S}}\textbf{A}^* \boldsymbol{\Pi} \left( \widehat{\boldsymbol{x}}(\boldsymbol{y} + \nu \boldsymbol{\varepsilon}; \boldsymbol{\Lambda}) - \widehat{\boldsymbol{x}}(\boldsymbol{y} ; \boldsymbol{\Lambda})\right), \boldsymbol{\varepsilon}  \right\rangle \nonumber\\
\label{eq:DOF_FDMC} &= \mathbb{E}_{\boldsymbol{\varepsilon}} \lim\limits_{\nu \rightarrow 0}  \frac{1}{\nu} \left\langle \textbf{A}^* \boldsymbol{\Pi} \left( \widehat{\boldsymbol{x}}(\boldsymbol{y} + \nu \boldsymbol{\varepsilon}; \boldsymbol{\Lambda}) - \widehat{\boldsymbol{x}}(\boldsymbol{y} ; \boldsymbol{\Lambda})\right), \boldsymbol{\mathcal{S}}\boldsymbol{\varepsilon}  \right\rangle.
\end{align}
Elaborating on Formula~\eqref{eq:DOF_FDMC} and Assumption~\ref{hyp:lip_L1}, the following theorem provides an asymptotically unbiased Finite Differences Monte Carlo estimator of the risk, which can be used in a vast variety of estimation problems.\\

\vspace{2mm}

\begin{theorem}
\label{thm:SURE_FDMC}
Consider the observation Model~\eqref{eq:obs_gen_model}, the operator $\textbf{A} $ defined in~\eqref{eq:def_A} together with Assumptions~\ref{hyp:gauss_noise}~(Gaussianity), \ref{hyp:full_rank}~(Full-rank), \ref{hyp:reg_int}~(Integrability), and \ref{hyp:lip_L1}~(Lipschitzianity w.r.t. $\boldsymbol{y}$). 
Generalized Finite Differences Monte Carlo SURE, writing
\begin{align}
\label{eq:SURE_FDMC_def}
& \widehat{R}_{\nu, \boldsymbol{\varepsilon}}(\boldsymbol{y}; \boldsymbol{\Lambda} \lvert \boldsymbol{\mathcal{S}}) \triangleq  \left\lVert  \textbf{A}\left( \boldsymbol{\Phi}\widehat{\boldsymbol{x}}(\boldsymbol{y} ; \boldsymbol{\Lambda}) - \boldsymbol{y} \right) \right\rVert_2^2\\
& + \frac{2}{\nu} \left\langle \textbf{A}^* \boldsymbol{\Pi} \left( \widehat{\boldsymbol{x}}(\boldsymbol{y} + \nu \boldsymbol{\varepsilon}; \boldsymbol{\Lambda}) - \widehat{\boldsymbol{x}}(\boldsymbol{y} ; \boldsymbol{\Lambda})\right), \boldsymbol{\mathcal{S}}\boldsymbol{\varepsilon}  \right\rangle- \mathrm{Tr}(\textbf{A} \boldsymbol{\mathcal{S}} \textbf{A}^*  ), \nonumber
\end{align}
is an asymptotically unbiased estimator of the risk $R[\widehat{\boldsymbol{x}}]( \boldsymbol{\Lambda}) $ as $\nu \rightarrow 0$, meaning that
\begin{align}
 \lim\limits_{\nu \rightarrow 0} \mathbb{E}_{\boldsymbol{\zeta}, \boldsymbol{\varepsilon}} \widehat{R}_{\nu, \boldsymbol{\varepsilon}}(\boldsymbol{y} ; \boldsymbol{\Lambda}\lvert \boldsymbol{\mathcal{S}}) = R[\widehat{\boldsymbol{x}}]( \boldsymbol{\Lambda}) .
\end{align}
\end{theorem}

\begin{proof}
The proof of Theorem~\ref{thm:SURE_FDMC} is postponed to Appendix~\ref{app:SURE_FDMC}.
\end{proof}

\begin{remark}
\label{rq:sparsity_S}
The use of Monte Carlo strategy is advocated in~\cite{deledalle2014stein} so that to reduce the complexity of SURE evaluation, replacing costly $P \times P$ matrices product by products of $P \times P$ matrix by vector of size $P$. 
Yet, the product of $\boldsymbol{S} \in \mathbb{R}^{P\times P}$ with $\boldsymbol{\varepsilon} \in \mathbb{R}^P$, as well as the product of $\textbf{A}^* \boldsymbol{\Pi} \in \mathbb{R}^{P \times N}$ with $\widehat{\boldsymbol{x}}(\boldsymbol{y}; \boldsymbol{\Lambda}) \in \mathbb{R}^N$, might still be extremely costly. Hopefully, we will see that, in data processing problems (e.g. for texture segmentation in Section~\ref{sec:text_seg}), both the covariance matrix $\boldsymbol{\mathcal{S}}$ and linear operator $\textbf{A}$ (through the degradation $\boldsymbol{\Phi}$) benefit from sufficient sparsity so that the calculations can be handled at a reasonable cost.
\end{remark}

\section{Stein's Unbiased GrAdient estimator of the Risk (SUGAR)}
\label{sec:SUGAR}

From the estimator of the risk $\widehat{R}_{\nu, \boldsymbol{\varepsilon}}(\boldsymbol{y}; \boldsymbol{\Lambda}\lvert \boldsymbol{\mathcal{S}})$ provided in previous Section~\ref{subsec:SURE}, basic grid search approach could be performed, in order to estimate the optimal $\boldsymbol{\Lambda}^\dagger$, as defined in~\eqref{eq:risk_min}. 
Yet, the exploration of a fine grid of $\boldsymbol{\Lambda} \in \mathbb{R}^L$ might be time consuming if the evaluation of $\widehat{\boldsymbol{x}}(\boldsymbol{y};\boldsymbol{\Lambda})$ is costly, which is the case when $\widehat{\boldsymbol{x}}(\boldsymbol{y};\boldsymbol{\Lambda})$ is \textit{sequential}, \textit{i.e.} obtained from an optimization scheme. 
Moreover, the size of a grid in $\mathbb{R}^L$ with given step size grows algebraically with $L$.
Altogether, this precludes grid search when $L > 2$.\\ 
Inspiring from~\cite{deledalle2014stein}, this section addresses this issue in the extended case of correlated noise.
We provide in Equation~\eqref{eq:SUGAR_FDMC} a generalized estimator $\nablaL_{\boldsymbol{\Lambda}}\widehat{R}_{\nu, \boldsymbol{\varepsilon}}(\boldsymbol{y}; \boldsymbol{\Lambda}\lvert \boldsymbol{\mathcal{S}})\in \mathbb{R}^L$ of the gradient of the 
risk with respect to hyperparameters $\boldsymbol{\Lambda}$.
Further, we demonstrate that the Finite Difference Monte Carlo estimator $\nablaL_{\boldsymbol{\Lambda}}\widehat{R}_{\nu, \boldsymbol{\varepsilon}}(\boldsymbol{y}; \boldsymbol{\Lambda}\lvert \boldsymbol{\mathcal{S}})$ is an asymptotically unbiased estimator of the gradient of the risk~\eqref{eq:risk_def} with respect to $\boldsymbol{\Lambda}$.\\
In Algorithm~\ref{alg:PD}, we provide an example of sequential estimator, relying on an accelerated primal-dual scheme, designed to solve~\eqref{eq:LS_pen}, with its differentiated 
counterpart, providing both $\widehat{\boldsymbol{x}}(\boldsymbol{y};\boldsymbol{\Lambda})$ and its Jacobian $\nablaL_{\boldsymbol{\Lambda}}\widehat{\boldsymbol{x}}(\boldsymbol{y};\boldsymbol{\Lambda}) \in \mathbb{R}^{N \times L}$ with respect to $\boldsymbol{\Lambda}$.\\
Hence, costly grid search can be avoided, the estimation of $\boldsymbol{\Lambda}^\dagger$ being performed by a quasi-Newton descent, described in Algorithm~\ref{alg:BFGS}, which minimizes the estimated risk $\widehat{R}_{\nu, \boldsymbol{\varepsilon}}(\boldsymbol{y}; \boldsymbol{\Lambda}\lvert \boldsymbol{\mathcal{S}})$, making use of its gradient $\nablaL_{\boldsymbol{\Lambda}}\widehat{R}_{\nu, \boldsymbol{\varepsilon}}(\boldsymbol{y}; \boldsymbol{\Lambda}\lvert \boldsymbol{\mathcal{S}})$.

\subsection{Differentiation of Stein Unbiased Risk Estimate}
\label{subsec:SUGAR}

\begin{claim}
\label{claim:SUGAR}
Consider the observation Model~\eqref{eq:obs_gen_model}, the operator $\textbf{A} $ defined in~\eqref{eq:def_A} together with Assumptions~\ref{hyp:gauss_noise}~(Gaussianity), \ref{hyp:full_rank}~(Full-rank), \ref{hyp:reg_int}~(Integrability), \ref{hyp:lip_L1}~(Lipschitzianity w.r.t. $\boldsymbol{y}$), and \ref{hyp:lip_L2}~(Lipschitzianity w.r.t. $\boldsymbol{\Lambda}$) Assumptions. 
Then the Finite Difference Monte Carlo SURE $\widehat{R}_{\nu, \boldsymbol{\varepsilon}}(\boldsymbol{y}; \boldsymbol{\Lambda}\lvert \boldsymbol{\mathcal{S}})$, defined in~\eqref{eq:SURE_FDMC_def}, is weakly differentiable with respect to both observations $\boldsymbol{y}$ and parameters $\boldsymbol{\Lambda}$, and its gradient with respect to $\boldsymbol{\Lambda}$, as an element of $\mathbb{R}^L$, is given by 
\begin{align}
\label{eq:SUGAR_FDMC}
\nablaL_{\boldsymbol{\Lambda}} \left[\widehat{R}_{\nu, \boldsymbol{\varepsilon}}(\boldsymbol{y}; \boldsymbol{\Lambda} \lvert \boldsymbol{\mathcal{S}}) \right] \triangleq
2 \left( \textbf{A} \boldsymbol{\Phi} \nablaL_{\boldsymbol{\Lambda}} \widehat{\boldsymbol{x}}(\boldsymbol{y}; \boldsymbol{\Lambda}) \right)^*  \textbf{A}\left( \boldsymbol{\Phi}\widehat{\boldsymbol{x}}(\boldsymbol{y}; \boldsymbol{\Lambda}) - \boldsymbol{y} \right) \\ +   \frac{2}{\nu} \left( \textbf{A}^* \boldsymbol{\Pi} \left( \nablaL_{\boldsymbol{\Lambda}}\widehat{\boldsymbol{x}}(\boldsymbol{y} + \nu \boldsymbol{\varepsilon}; \boldsymbol{\Lambda}) - \nablaL_{\boldsymbol{\Lambda}}\widehat{\boldsymbol{x}}(\boldsymbol{y} ; \boldsymbol{\Lambda})\right)\right)^* \boldsymbol{\mathcal{S}}\boldsymbol{\varepsilon} , \nonumber
\end{align}
\end{claim}

\begin{proof}
The Finite Difference Monte Carlo SURE $\widehat{R}_{\nu, \boldsymbol{\varepsilon}}(\boldsymbol{y}; \boldsymbol{\Lambda}\lvert \boldsymbol{\mathcal{S}})$, defined by Formula~\eqref{eq:SURE_FDMC_def} is a combination of continuous and weakly differentiable functions with respect to both observations $\boldsymbol{y}$ and parameters $\boldsymbol{\Lambda}$, composed with (bounded) linear operators, and thus is continuous and weakly differentiable. Further, the derivation rules apply and lead to the expression of Finite Difference Monte Carlo SUGAR estimator given in~Formula~\eqref{eq:SUGAR_FDMC}.\\
\end{proof}

\begin{hypothesis}[Lipschitzianity w.r.t. hyperparameters]
\label{hyp:lip_L2}
Let $\widehat{\boldsymbol{x}}(\boldsymbol{y};\boldsymbol{\Lambda})$ be an estimator of $\bar{\boldsymbol{x}}$, depending on observations $\boldsymbol{y}$, and parametrized by $\boldsymbol{\Lambda}$.   
The mapping $\boldsymbol{\Lambda} \mapsto \widehat{\boldsymbol{x}}(\boldsymbol{y};\boldsymbol{\Lambda})$ is uniformly $L_2$-Lipschitz continuous with constant $L_2$ being independent of $\boldsymbol{y}$.
\end{hypothesis}

\begin{remark}
\label{rq:check_L2}
As argued in~\cite{deledalle2014stein}, when the estimator $\widehat{\boldsymbol{x}}(\boldsymbol{y};\boldsymbol{\Lambda})$ can be expressed as a (composition of) proximal operator(s) of gauge(s) of compact set(s)\footnote{For $\mathcal{C} \subset \mathcal{G}$ a non-empty closed convex set containing $\boldsymbol{0}_{\mathcal{G}}$, the gauge of $\mathcal{C}$ is defined as $\gamma_{\mathcal{C}}(\boldsymbol{y}) \triangleq \inf \left\lbrace \omega > 0 \, | \, \boldsymbol{y} \in \omega \mathcal{C}\right\rbrace$.}, Assumption~\ref{hyp:lip_L2} holds. 
Thus, in the case of~\eqref{eq:LS_pen} when $\mathcal{G} = \mathcal{H}$, $\boldsymbol{\Phi} =\boldsymbol{I}_{\mathcal{H}}$, and $\lVert \textbf{U}_{\boldsymbol{\Lambda}} \boldsymbol{x} \rVert_q^q= \lambda \lVert \boldsymbol{x}  \rVert_q^q$, for any $q \geq 1$ the Lipschitzianity w.r.t. $\boldsymbol{\Lambda}$ is ensured.
Moreover, in the case of Tikhonov regularization, i.e. $q = 2$ and $\lVert \textbf{U}_{\boldsymbol{\Lambda}} \boldsymbol{x} \rVert_{q}^q = \lambda \lVert \textbf{D} \boldsymbol{x}\rVert_2^2$ in~\eqref{eq:LS_pen}, if $\boldsymbol{\Phi} = \boldsymbol{I}_{\mathcal{H}}$ and $\textbf{D}^*\textbf{D}$ is diagonalizable with strictly positive eigenvalues, then Assumption~\ref{hyp:lip_L2} is verified.
Apart from these two well-known examples, proving the validity of Assumption~\ref{hyp:lip_L2} in the general case of Penalized Least Square is a difficult problem and is foreseen for future work.
\end{remark}

\begin{theorem}
\label{thm:SUGAR}
Consider the observation Model~\eqref{eq:obs_gen_model}, the operator $\textbf{A} $ defined in~\eqref{eq:def_A} together with Gaussianity~\ref{hyp:gauss_noise}, Full-rank~\ref{hyp:full_rank}, Integrability~\ref{hyp:reg_int}, Lipschitzianity w.r.t. $\boldsymbol{y}$~\ref{hyp:lip_L1}, and Lipschitzianity w.r.t. $\boldsymbol{\Lambda}$~\ref{hyp:lip_L2} Assumptions. 
Then generalized Finite Difference Monte Carlo SUGAR, $ \nablaL_{\boldsymbol{\Lambda}}\widehat{R}_{\nu, \boldsymbol{\varepsilon}}(\boldsymbol{y}; \boldsymbol{\Lambda})$ defined in~Equation~\eqref{eq:SUGAR_FDMC}, is an asymptotically unbiased estimate of the gradient of the risk as $\nu \rightarrow 0$, that is
\begin{align}
\nablaL_{\boldsymbol{\Lambda}} R [\widehat{\boldsymbol{x}}](\boldsymbol{\Lambda}) = \lim\limits_{\nu \rightarrow 0}  \mathbb{E}_{\boldsymbol{\zeta}, \boldsymbol{\varepsilon}} \nablaL_{\boldsymbol{\Lambda}}\widehat{R}_{\nu, \boldsymbol{\varepsilon}}(\boldsymbol{y}; \boldsymbol{\Lambda}\lvert \boldsymbol{\mathcal{S}})
\end{align}
\end{theorem}

\begin{proof}
The proof of Theorem~\ref{thm:SUGAR} is postponed to Appendix~\ref{app:SUGAR_FDMC}.
\end{proof}

\begin{remark}
\label{rq:low_dim}
Finite Difference Monte Carlo estimator of the gradient of the risk, $\nablaL_{\boldsymbol{\Lambda}}\widehat{R}_{\nu, \boldsymbol{\varepsilon}}(\boldsymbol{y}; \boldsymbol{\Lambda}\lvert \boldsymbol{\mathcal{S}})$, defined in Equation~\eqref{eq:SUGAR_FDMC}, involves the Jacobian $\nablaL_{\boldsymbol{\Lambda}}\widehat{\boldsymbol{x}}(\boldsymbol{y} ; \boldsymbol{\Lambda}) \in \mathbb{R}^{N \times L}$ which could be a very large matrix, raising difficulties for practical use.
Nevertheless, in most applications, the regularization hyperparameters $\boldsymbol{\Lambda} \in \mathbb{R}^L$, have a ``low" dimensionality $L = \mathcal{O}(1) \ll N$.
Thus, it is reasonable to expect that the Jacobian matrix $\nablaL_{\boldsymbol{\Lambda}}\widehat{\boldsymbol{x}}(\boldsymbol{y} ; \boldsymbol{\Lambda}) \in \mathbb{R}^{N \times L}$ can be stored and manipulated, with similar memory and computational costs than for $\widehat{\boldsymbol{x}}(\boldsymbol{y};\boldsymbol{\Lambda})$ (see Section~\ref{subsec:seq_est}). 
\end{remark}

\subsection{Sequential estimators and forward iterative differentiation}
\label{subsec:seq_est}

The evaluation of $\partial_{\boldsymbol{\Lambda}} \widehat{R}_{\nu, \boldsymbol{\varepsilon}}(\boldsymbol{y}; \boldsymbol{\Lambda})$ from Formula~\eqref{eq:SUGAR_FDMC} requires the Jacobian $\partial_{\boldsymbol{\Lambda}}\widehat{x}(\boldsymbol{y};\boldsymbol{\Lambda})$.
Yet, when no closed-form expression of estimator $\widehat{\boldsymbol{x}}(\boldsymbol{y}; \boldsymbol{\Lambda})$ is available, computing the gradient $\partial_{\boldsymbol{\Lambda}}\widehat{\boldsymbol{x}}(\boldsymbol{y};\boldsymbol{\Lambda})$ might be a complicated task.
A large class of estimators $\widehat{\boldsymbol{x}}(\boldsymbol{y}; \boldsymbol{\Lambda})$ lacking closed-form expression are those obtained as the limit of iterates as
\begin{align}
\label{eq:seq_estim}
\widehat{\boldsymbol{x}}(\boldsymbol{y};\boldsymbol{\Lambda}) = \lim\limits_{k \rightarrow \infty} \boldsymbol{x}^{[k]}(\boldsymbol{y};\boldsymbol{\Lambda}),
\end{align}
for instance when $\widehat{\boldsymbol{x}}(\boldsymbol{y}; \boldsymbol{\Lambda})$ is defined as the solution of a minimization problem, e.g.~\eqref{eq:LS_pen}. \\
In the case when $\widehat{\boldsymbol{x}}(\boldsymbol{y}; \boldsymbol{\Lambda})$ is a sequential estimator, given an observation $\boldsymbol{y}$, it is only possible to sample  the function $\boldsymbol{\Lambda} \mapsto \widehat{\boldsymbol{x}}(\boldsymbol{y};\boldsymbol{\Lambda}), $ for a \textit{discrete} set of regularization hyperparameters $\lbrace \boldsymbol{\Lambda}_1, \boldsymbol{\Lambda}_2, \hdots \rbrace$, running the minimization algorithm for each hyperparameters $\boldsymbol{\Lambda}_1,\boldsymbol{\Lambda}_2,\hdots$. 
It is a classical fact in signal processing that no robust estimator of the differential can be built from samples of the function, thus more sophisticated tools are needed.
Provided some smoothness conditions on the iterations of the minimization algorithm, iterative differentiation strategy~\cite{deledalle2014stein} gives access to a sequence of Jacobian $ \nablaL_{\boldsymbol{\Lambda}} \widehat{\boldsymbol{x}}^{[k]}(\boldsymbol{y};\boldsymbol{\Lambda})$, relying on \textit{chain rule} differentiation presented in Proposition~\ref{claim:chain_rule}.\\

Considering Problem~\eqref{eq:LS_pen}, splitting algorithms~\cite{parikh2014proximal,Combettes2011,bauschke2011convex} are advocated to perform the minimization.
We chose the primal-dual scheme proposed in~\cite{chambolle2011first}, Algorithm~2, taking advantage of closed-form expressions of the proximal operators~\cite{rockafellar1970convex} of both the data-fidelity term and the penalization\footnote{see \texttt{http://proximity-operator.net} for numerous proximal operator closed-form expressions}.
Chambolle-Pock algorithm, particularized to~\eqref{eq:LS_pen}, is presented in Algorithm~\ref{alg:PD}.
Further, $\boldsymbol{\Phi}$ being full-rank (Assumption~\ref{hyp:full_rank}), denoting by $\mathrm{Sp}(\boldsymbol{\Phi}^*\boldsymbol{\Phi})$ the spectrum of $\boldsymbol{\Phi}^*\boldsymbol{\Phi}$, $\gamma = 2\min \mathrm{Sp}(\boldsymbol{\Phi}^*\boldsymbol{\Phi})$ is strictly positive.
Hence the data-fidelity in~\eqref{eq:LS_pen} term turns out to be $\gamma$-strongly convex, and the primal-dual algorithm can be accelerated thank to Step~\eqref{eq:pd_tausig_up} of Algorithm~\ref{alg:PD}, following~\cite{chambolle2011first}.
The iterative differentiation strategy providing $\nablaL_{\boldsymbol{\Lambda}} \boldsymbol{x}^{[k]}(\boldsymbol{y};\boldsymbol{\Lambda})$ is presented in the second part of Algorithm~\ref{alg:PD}.
Other iterative differentiation schemes are detailed in~\cite{deledalle2014stein}.

\begin{claim}
\label{claim:chain_rule}
Let $\boldsymbol{\Psi} : \mathbb{R}^{N \times L} \rightarrow \mathbb{R}^N$ be a differentiable function of variable $\boldsymbol{x} \in \mathcal{H}$, differentiably parametrized by $\boldsymbol{\Lambda} \in \mathbb{R}^L$, and $\left( \boldsymbol{x}^{[k]} \right)_{k\in \mathbb{N}}$ the sequential estimator defined by iterations of the form
\begin{align}
\boldsymbol{x}^{[k+1]} = \boldsymbol{\Psi}(\boldsymbol{x}^{[k]} ; \boldsymbol{\Lambda}).
\end{align}
The gradient of $\boldsymbol{x}^{[k]}$ with respect to $\boldsymbol{\Lambda}$ can be computed making use of the \textit{chain rule} differentiation
\begin{align}
\label{eq:chain_rule}
\nablaL_{\boldsymbol{\Lambda}} \boldsymbol{x}^{[k+1]} = \nablaL_{\boldsymbol{\Lambda}} \left( \boldsymbol{\Psi}(\boldsymbol{x}^{[k]} ; \boldsymbol{\Lambda})\right) = \partial_{\boldsymbol{x}} \boldsymbol{\Psi}(\boldsymbol{x}^{[k]} ; \boldsymbol{\Lambda})[ \nablaL_{\boldsymbol{\Lambda}} \boldsymbol{x}^{[k]}] + \nablaL_{\boldsymbol{\Lambda}} \boldsymbol{\Psi}(\boldsymbol{x}^{[k]} ; \boldsymbol{\Lambda}),
\end{align}
where $\partial_{\boldsymbol{x}} \boldsymbol{\Psi}(\boldsymbol{x}; \boldsymbol{\Lambda})[ \boldsymbol{\delta}] $ denotes the differential of $\boldsymbol{\Psi}$ with respect to variable $\boldsymbol{x}$ applied on vector $\boldsymbol{\delta}$, and $\nablaL_{\boldsymbol{\Lambda}} \boldsymbol{\Psi}(\boldsymbol{x}; \boldsymbol{\Lambda})$ the gradient of $\boldsymbol{\Psi}$ with respect to $\boldsymbol{\Lambda}$. The differentiability of $\boldsymbol{\Psi}$ should be understood in the weak sense.
\end{claim}

\begin{remark}
Two particular cases are often encountered in iterative differentiation (see Algorithm~\ref{alg:PD}):\\
\textit{(i)~Linear operator} $\boldsymbol{\Psi} (\boldsymbol{x} ;\boldsymbol{\Lambda}) \triangleq \textbf{U}_{\boldsymbol{\Lambda}} \boldsymbol{x}$. Assuming that $\left( \boldsymbol{x} \mapsto \textbf{U}_{\boldsymbol{\Lambda}}\boldsymbol{x} \right)_{\boldsymbol{\Lambda}}$ is a family of linear operators, with a differentiable parametrization by $\boldsymbol{\Lambda}$, the \textit{chain rule} writes
\begin{align}
\nablaL_{\boldsymbol{\Lambda}} \boldsymbol{x}^{[k+1]} = \textbf{U}_{\boldsymbol{\Lambda}} \nablaL_{\boldsymbol{\Lambda}} \boldsymbol{x}^{[k]} + \left( \nablaL_{\boldsymbol{\Lambda}} \textbf{U}_{\boldsymbol{\Lambda}} \right) \boldsymbol{x}^{[k]},
\end{align}
since the differential of the linear operator $\textbf{U}_{\boldsymbol{\Lambda}}$ with respect to $\boldsymbol{x}$ is itself. See  \eqref{eq:linear_chain_z}~and~\eqref{eq:linear_chain_x}, in Algorithm~\ref{alg:PD} for applications of the \textit{chain rule} with linear operators.\\
\textit{(ii)~Proximal operator $\boldsymbol{\Psi} (\boldsymbol{x} ;\boldsymbol{\Lambda}) \triangleq \mathrm{prox}_{\tau \lVert \cdot\rVert_{2,1}}( \boldsymbol{x})$.} The proximal operator being independent of $\boldsymbol{\Lambda}$, the \textit{chain rule} simplifies to
\begin{align}
\nablaL_{\boldsymbol{\Lambda}} \boldsymbol{x}^{[k+1]} = \partial_{\boldsymbol{x}} \mathrm{prox}_{\tau \lVert \cdot\rVert_{2,1}}(\boldsymbol{x}^{[k]})[ \nablaL_{\boldsymbol{\Lambda}} \boldsymbol{x}^{[k]}]
\end{align}
with the differential of the so-called $\ell_2-\ell_1$ \textit{soft-thresholding} $\mathrm{prox}_{\tau \lVert \cdot\rVert_{2,1}}$ with respect to $\boldsymbol{x}= \left( x_1, x_2\right)$, applied on $\boldsymbol{\delta} = (\delta_1, \delta_2)$ having the closed-form expression
\begin{align}
\nablaL_{\boldsymbol{x}} \mathrm{prox}_{\tau \lVert \cdot\rVert_{2,1}}( \boldsymbol{x}) [\boldsymbol{\delta}] = \left\lbrace  \begin{array}{ll}
\boldsymbol{0} & \mathrm{if} \quad \lVert \boldsymbol{x} \rVert_2 \leq \tau \\
\boldsymbol{\delta} - \frac{\tau}{\lVert \boldsymbol{x}\rVert_2} \left( \boldsymbol{\delta} - \frac{\langle \boldsymbol{\delta}, \boldsymbol{x} \rangle}{\lVert \boldsymbol{x} \rVert_2^2} \boldsymbol{x} \right) & \mathrm{else}.
\end{array} 
\right.
\end{align}
See  \eqref{eq:prox_chain_z}~and~\eqref{eq:prox_chain_x}, in Algorithm~\ref{alg:PD} for applications of the \textit{chain rule} with proximal operators.\\
\end{remark}

\begin{algorithm}[h!]
\caption{\label{alg:PD} Accelerated primal-dual scheme for solving~\eqref{eq:LS_pen} with iterative differentiation with respect to regularization parameters $\boldsymbol{\Lambda}$.}
\begin{algorithmic}
\STATE{
\textbf{Routines:}
 \raisebox{-2.25mm}{\parbox{0.78\linewidth}{$\begin{array}{ccc}
  \widehat{\boldsymbol{x}}(\boldsymbol{y}; \boldsymbol{\Lambda}) &=&  \mathrm{PD}(\boldsymbol{y}, \boldsymbol{\Lambda})\\ 
\left( \widehat{\boldsymbol{x}}(\boldsymbol{y}; \boldsymbol{\Lambda}) , \nablaL_{\boldsymbol{\Lambda}} \widehat{\boldsymbol{x}}(\boldsymbol{y}; \boldsymbol{\Lambda})\right) &=& \partial \mathrm{PD}(\boldsymbol{y}, \boldsymbol{\Lambda})
\end{array}$}}
}
\STATE{
\textbf{Inputs:}
\raisebox{-0.63cm}{\parbox{0.8\linewidth}{Observations $\boldsymbol{y}$\\
Regularization hyperparameters $\boldsymbol{\Lambda}$ \\ 
Strong-convexity modulus of data-fidelity term $\gamma = 2\min \mathrm{Sp}(\boldsymbol{\Phi}^*\boldsymbol{\Phi})$\\
}}}

\STATE{
\textbf{Initialization:}
\raisebox{-0.325cm}{\parbox{0.8\linewidth}{Descent steps $\boldsymbol{\tau}^{[0]} = (\tau^{[0]} _1, \tau^{[0]} _2)$ such that $\tau^{[0]} _1 \tau^{[0]} _2 \lVert \textbf{U}_{\boldsymbol{\Lambda}} \rVert^2 < 1$\\
Primal, auxiliary and dual variables $\boldsymbol{x}^{[0]} \in \mathcal{H}$, $\boldsymbol{w}^{[0]} = \boldsymbol{x}^{[0]}$, $\boldsymbol{z}^{[0]} \in \mathcal{Q}$\\
}}}

\STATE{
\textbf{Preliminaries:} Jacobian with respect to $\boldsymbol{\Lambda}$: 
$\nablaL_{\boldsymbol{\Lambda}} \boldsymbol{x}^{[0]}$, $\nablaL_{\boldsymbol{\Lambda}} \boldsymbol{w}^{[0]} $, $\nablaL_{\boldsymbol{\Lambda}} \boldsymbol{z}^{[0]} $
}

\vspace{2mm}
\FOR{$k = 1$ \TO $K_{\max}$}
\STATE{
\vspace{2mm}
\COMMENT{Accelerated Primal-Dual}
\begin{align}
& \widetilde{\boldsymbol{z}}^{[k]} = \boldsymbol{z}^{[k]} + \tau^{[k]}_1 \textbf{U}_{\boldsymbol{\Lambda}}  \boldsymbol{w}^{[k]}\\
&\boldsymbol{z}^{[k+1]} = \mathrm{prox}_{\tau^{[k]}_1 \left( \lVert \cdot \rVert_q^q \right)^*} \left( \widetilde{\boldsymbol{z}}^{[k]} \right) \\
& \widetilde{\boldsymbol{x}}^{[k]} = \boldsymbol{x}^{[k]} - \tau^{[k]}_2 \textbf{U}^*_{\boldsymbol{\Lambda}}  \boldsymbol{z}^{[k+1]}\\
&\boldsymbol{x}^{[k+1]} = \mathrm{prox}_{\tau^{[k]}_2 \lVert \boldsymbol{y} - \boldsymbol{\Phi} \cdot \rVert_{2}^2} \left(  \widetilde{\boldsymbol{x}}^{[k]}  \right)\\
& \label{eq:pd_tausig_up}\theta^{[k]} = 1/\sqrt{1 + 2 \gamma \tau^{[k]}_2}, \quad \tau_1^{[k+1]} = \tau_1^{[k]} / \theta^{[k]}, \quad \tau_2^{[k+1]} = \theta^{[k]} \tau_2^{[k]}\\
&\boldsymbol{w}^{[k+1]} =  \boldsymbol{x}^{[k]} + \theta^{[k]} \left( \boldsymbol{x}^{[k+1]}  - \boldsymbol{x}^{[k]}  \right)
\end{align}
\vspace{0mm}

\COMMENT{Accelerated Differentiated Primal-Dual}
\begin{align}
& \label{eq:linear_chain_z} \nablaL_{\boldsymbol{\Lambda}} \widetilde{\boldsymbol{z}}^{[k]} = \nablaL_{\boldsymbol{\Lambda}}\boldsymbol{z}^{[k]} + \tau^{[k]}_1 \textbf{U}_{\boldsymbol{\Lambda}}  \nablaL_{\boldsymbol{\Lambda}}\boldsymbol{w}^{[k]} +  \tau^{[k]}_1 \frac{\partial\textbf{U}_{\boldsymbol{\Lambda}} }{\partial\boldsymbol{\Lambda}} \boldsymbol{w}^{[k]}\\
& \label{eq:prox_chain_z} \nablaL_{\boldsymbol{\Lambda}}  \boldsymbol{z}^{[k+1]} = \partial_{\widetilde{\boldsymbol{z}}} \mathrm{prox}_{\tau^{[k]}_1 \left( \lVert \cdot \rVert_q^q \right)^*} \left( \widetilde{\boldsymbol{z}}^{[k]} \right)  \left[ \nablaL_{\boldsymbol{\Lambda}} \widetilde{\boldsymbol{z}}^{[k]} \right] \\
& \label{eq:linear_chain_x}\nablaL_{\boldsymbol{\Lambda}} \widetilde{\boldsymbol{x}}^{[k]} = \nablaL_{\boldsymbol{\Lambda}} \boldsymbol{x}^{[k]} - \tau^{[k]}_2 \textbf{U}^*_{\boldsymbol{\Lambda}}  \nablaL_{\boldsymbol{\Lambda}} \boldsymbol{z}^{[k+1]} - \tau^{[k]}_2  \frac{\partial\textbf{U}_{\boldsymbol{\Lambda}} }{\partial\boldsymbol{\Lambda}} \boldsymbol{z}^{[k+1]}\\
&\label{eq:prox_chain_x} \nablaL_{\boldsymbol{\Lambda}}\boldsymbol{x}^{[k+1]} = \partial_{\widetilde{\boldsymbol{x}}}\mathrm{prox}_{\tau^{[k]}_2 \lVert \boldsymbol{y} - \boldsymbol{\Phi} \cdot \rVert_{2}^2} \left(  \widetilde{\boldsymbol{x}}^{[k]}  \right)  \left[\nablaL_{\boldsymbol{\Lambda}}  \widetilde{\boldsymbol{x}}^{[k]}  \right]\\
&\nablaL_{\boldsymbol{\Lambda}} \boldsymbol{w}^{[k+1]} =  \nablaL_{\boldsymbol{\Lambda}} \boldsymbol{x}^{[k]} + \theta^{[k]} \left( \nablaL_{\boldsymbol{\Lambda}}\boldsymbol{x}^{[k+1]}  - \nablaL_{\boldsymbol{\Lambda}}\boldsymbol{x}^{[k]}  \right)
\end{align}
}
\ENDFOR
\vspace{2mm}
\STATE{
\textbf{Outputs:}
 \raisebox{-0.2cm}{\parbox{0.85\linewidth}{\textit{Finite-time} solution of Problem~\eqref{eq:LS_pen} \hspace{1.45cm} $\widehat{\boldsymbol{x}}(\boldsymbol{y}; \boldsymbol{\Lambda}) \triangleq \widehat{\boldsymbol{x}}^{[K_{\max}]}$\\
 \textit{Finite-time} Jacobian w.r.t. hyperparameters $\nablaL_{\boldsymbol{\Lambda}} \widehat{\boldsymbol{x}}(\boldsymbol{y}; \boldsymbol{\Lambda}) \triangleq \nablaL_{\boldsymbol{\Lambda}} \widehat{\boldsymbol{x}}^{[K_{\max}]}$}}
}
\end{algorithmic}
\end{algorithm}

\begin{definition}[Generalized SURE and SUGAR for sequential estimators]
Let $\widehat{\boldsymbol{x}}(\boldsymbol{\ell}; \boldsymbol{\Lambda})$ be a sequential estimator in the sense of~\eqref{eq:seq_estim}.
The associated risk estimate $\widehat{R}_{\nu, \boldsymbol{\varepsilon}}(\boldsymbol{y}; \boldsymbol{\Lambda}\lvert \boldsymbol{\mathcal{S}})$ and gradient of the risk estimate $\nablaL_{\boldsymbol{\Lambda}}\widehat{R}_{\nu, \boldsymbol{\varepsilon}}(\boldsymbol{y}; \boldsymbol{\Lambda}\lvert \boldsymbol{\mathcal{S}})$ are computed running Algorithm~\ref{alg:PD} twice: first with input $\boldsymbol{y}$~(observations), second with input $\boldsymbol{y} + \nu \boldsymbol{\varepsilon}$~(perturbed observations).
Then, generalized SURE 
is computed from Formula~\eqref{eq:SURE_FDMC_def}, and generalized SUGAR 
from Formula~\eqref{eq:SUGAR_FDMC}.
These steps are summarized into routines respectively called ``SURE" and ``SUGAR", detailed in Algorithm~\ref{alg:SURE_SUGAR}.
\end{definition}

\begin{algorithm}[h!]
\caption{\label{alg:SURE_SUGAR} Generalized SURE and SUGAR for sequential $\widehat{\boldsymbol{x}}(\boldsymbol{y}; \boldsymbol{\Lambda})$.}
\begin{algorithmic}
\STATE{
\vspace{-3mm}
\textbf{Routines:}
 \raisebox{-1.9mm}{\parbox{0.6\linewidth}{
 \begin{align*}
 \begin{array}{ccc} \widehat{R}_{\nu, \boldsymbol{\varepsilon}}(\boldsymbol{y}; \boldsymbol{\Lambda} \lvert \boldsymbol{\mathcal{S}}) &=&  \mathrm{SURE}(\boldsymbol{y}, \boldsymbol{\Lambda}, \boldsymbol{\mathcal{S}}, \nu, \boldsymbol{\varepsilon}) \\
\nablaL_{\boldsymbol{\Lambda}}\widehat{R}_{\nu, \boldsymbol{\varepsilon}}(\boldsymbol{y}; \boldsymbol{\Lambda} \lvert \boldsymbol{\mathcal{S}}) &= &\mathrm{SUGAR}(\boldsymbol{y}, \boldsymbol{\Lambda}, \boldsymbol{\mathcal{S}}, \nu, \boldsymbol{\varepsilon})
\end{array}
\end{align*}}}
}
\vspace{1mm}
\STATE{
\textbf{Inputs:}
\raisebox{-1.05cm}{\parbox{0.7\linewidth}{Observations $\boldsymbol{y}$ \\ 
Regularization hyperparameters $\boldsymbol{\Lambda}$\\
Covariance matrix $\boldsymbol{\mathcal{S}}$\\
Monte Carlo vector $\boldsymbol{\varepsilon} \in \mathbb{R}^P \sim \mathcal{N}(\boldsymbol{0}_{P}, \boldsymbol{I}_{P})$\\
Finite Difference step $\nu > 0$\\
}}}
\vspace{2mm}
\STATE{
\COMMENT{Solution of~\eqref{eq:LS_pen} from Algorithm~\ref{alg:PD}}
\begin{align}
\widehat{\boldsymbol{x}}(\boldsymbol{y}; \boldsymbol{\Lambda}) &= \mathrm{PD}(\boldsymbol{y}, \boldsymbol{\Lambda})\\
\widehat{\boldsymbol{x}}(\boldsymbol{y} + \nu \boldsymbol{\varepsilon}; \boldsymbol{\Lambda}) &=  \mathrm{PD}(\boldsymbol{y}+ \nu \boldsymbol{\varepsilon}, \boldsymbol{\Lambda})
\end{align}
}
\vspace{1mm}
\STATE{
\COMMENT{Finite Difference Monte Carlo SURE~\eqref{eq:SURE_FDMC_def}}
\begin{align}
& \label{eq:SURE_eval} \widehat{R}_{\nu, \boldsymbol{\varepsilon}}(\boldsymbol{y}; \boldsymbol{\Lambda} \lvert \boldsymbol{\mathcal{S}}) =  \left\lVert  \textbf{A}\left( \boldsymbol{\Phi}\widehat{\boldsymbol{x}}(\boldsymbol{y} ; \boldsymbol{\Lambda}) - \boldsymbol{y} \right) \right\rVert_2^2   \\ \nonumber
&+ \frac{2}{\nu} \left\langle \textbf{A}^* \boldsymbol{\Pi} \left( \widehat{\boldsymbol{x}}(\boldsymbol{y} + \nu \boldsymbol{\varepsilon}; \boldsymbol{\Lambda}) - \widehat{\boldsymbol{x}}(\boldsymbol{y} ; \boldsymbol{\Lambda})\right), \boldsymbol{\mathcal{S}}\boldsymbol{\varepsilon}  \right\rangle- \mathrm{Tr}(\textbf{A} \boldsymbol{\mathcal{S}} \textbf{A}^*  ) 
\end{align}
}
\vspace{1mm}
\STATE{
\COMMENT{Solution of~\eqref{eq:LS_pen} and its differential w.r.t. $\boldsymbol{\Lambda}$ from Algorithm~\ref{alg:PD}}
\begin{align}
\left( \widehat{\boldsymbol{x}}(\boldsymbol{y}; \boldsymbol{\Lambda}) ,  \nablaL_{\boldsymbol{\Lambda}} \widehat{\boldsymbol{x}}(\boldsymbol{y}; \boldsymbol{\Lambda}) \right)  &= \partial \mathrm{PD}(\boldsymbol{y}, \boldsymbol{\Lambda})\\
\left( \widehat{\boldsymbol{x}}(\boldsymbol{y} + \nu \boldsymbol{\varepsilon}; \boldsymbol{\Lambda}) ,  \nablaL_{\boldsymbol{\Lambda}} \widehat{\boldsymbol{x}}(\boldsymbol{y}+ \nu \boldsymbol{\varepsilon}; \boldsymbol{\Lambda})  \right)&= \partial \mathrm{PD}(\boldsymbol{y}+ \nu \boldsymbol{\varepsilon}, \boldsymbol{\Lambda})
\end{align}
}
\vspace{1mm}
\STATE{
\COMMENT{Finite Difference Monte Carlo estimators~\eqref{eq:SURE_FDMC_def}~and~\eqref{eq:SUGAR_FDMC}}
\begin{align}
& \label{eq:SUGAR_eval} \widehat{R}_{\nu, \boldsymbol{\varepsilon}}(\boldsymbol{y}; \boldsymbol{\Lambda} \lvert \boldsymbol{\mathcal{S}}) =  \left\lVert  \textbf{A}\left( \boldsymbol{\Phi}\widehat{\boldsymbol{x}}(\boldsymbol{y} ; \boldsymbol{\Lambda}) - \boldsymbol{y} \right) \right\rVert_2^2   \\ \nonumber
&+ \frac{2}{\nu} \left\langle \textbf{A}^* \boldsymbol{\Pi} \left( \widehat{\boldsymbol{x}}(\boldsymbol{y} + \nu \boldsymbol{\varepsilon}; \boldsymbol{\Lambda}) - \widehat{\boldsymbol{x}}(\boldsymbol{y} ; \boldsymbol{\Lambda})\right), \boldsymbol{\mathcal{S}}\boldsymbol{\varepsilon}  \right\rangle- \mathrm{Tr}(\textbf{A} \boldsymbol{\mathcal{S}} \textbf{A}^*  ) \\
&\nablaL_{\boldsymbol{\Lambda}} \widehat{R}_{\nu,\boldsymbol{\varepsilon}}(\boldsymbol{y}; \boldsymbol{\Lambda} \lvert \boldsymbol{\mathcal{S}}) = 2 \left( \textbf{A} \boldsymbol{\Phi} \nablaL_{\boldsymbol{\Lambda}} \widehat{\boldsymbol{x}}(\boldsymbol{y}; \boldsymbol{\Lambda}) \right)^*  \textbf{A}\left( \boldsymbol{\Phi}\widehat{\boldsymbol{x}}(\boldsymbol{y} ; \boldsymbol{\Lambda}) - \boldsymbol{y} \right)  \\ \nonumber
 &+   \frac{2}{\nu} \left( \textbf{A}^* \boldsymbol{\Pi} \left( \nablaL_{\boldsymbol{\Lambda}}\widehat{\boldsymbol{x}}(\boldsymbol{y} + \nu \boldsymbol{\varepsilon}; \boldsymbol{\Lambda}) - \nablaL_{\boldsymbol{\Lambda}}\widehat{\boldsymbol{x}}(\boldsymbol{y} ; \boldsymbol{\Lambda})\right)\right)^* \boldsymbol{\mathcal{S}}\boldsymbol{\varepsilon} 
\end{align}
}
\vspace{1mm}
\STATE{
\textbf{Output:}
\raisebox{-2.15mm}{\parbox{0.7\linewidth}{
Risk estimate $\widehat{R}_{\nu, \boldsymbol{\varepsilon}}(\boldsymbol{y}; \boldsymbol{\Lambda} \lvert \boldsymbol{\mathcal{S}})$ \\
Gradient of the risk estimate $\nablaL_{\boldsymbol{\Lambda}}\widehat{R}_{\nu, \boldsymbol{\varepsilon}}(\boldsymbol{y}; \boldsymbol{\Lambda} \lvert \boldsymbol{\mathcal{S}})$} 
}}
\end{algorithmic}
\end{algorithm}

\subsection{Automatic risk minimization}
\label{subsec:BFGS}

Theorem~\ref{thm:SURE_FDMC} provides an asymptotically unbiased estimator of the risk $R[\widehat{\boldsymbol{x}}](\boldsymbol{\Lambda})$, denoted $\widehat{R}_{\nu, \boldsymbol{\varepsilon}}(\boldsymbol{y}; \boldsymbol{\Lambda}\lvert \boldsymbol{\mathcal{S}})$, based on Finite Difference Monte Carlo strategy.
Hence, for sufficiently small Finite Difference step $\nu > 0$, we can expect that the solution $\boldsymbol{\Lambda}^\dagger$ of Problem~\eqref{eq:min_risk}, minimizing the true risk, is well approximated by the hyperparameters $\widehat{\boldsymbol{\Lambda}}_{\nu, \boldsymbol{\varepsilon}}^\dagger$ minimizing the estimated risk
\begin{align}
\label{eq:min_sure}
\widehat{\boldsymbol{\Lambda}}_{\nu, \boldsymbol{\varepsilon}}^\dagger (\boldsymbol{y} \lvert \boldsymbol{\mathcal{S}}) \in \underset{\boldsymbol{\Lambda} \in \mathbb{R}^L}{\mathrm{Argmin}} \,\, \widehat{R}_{\nu, \boldsymbol{\varepsilon}}(\boldsymbol{y}; \boldsymbol{\Lambda}\lvert \boldsymbol{\mathcal{S}}) .
\end{align}
Then, since the dimensionality of $\boldsymbol{\Lambda} \in \mathbb{R}^L$ is ``low" enough (see Remark~\ref{rq:low_dim}), Problem~\eqref{eq:min_sure} is addressed performing a quasi-Newton descent algorithm, using 
the estimated gradient of the risk $\nablaL_{\boldsymbol{\Lambda}}\widehat{R}_{\nu, \boldsymbol{\varepsilon}}(\boldsymbol{y}; \boldsymbol{\Lambda}\lvert \boldsymbol{\mathcal{S}})$, provided by Theorem~\ref{thm:SUGAR}. \\
\begin{algorithm}[h!]
\caption{\label{alg:BFGS} Automated selection of hyperparameters minimizing quadratic risk.}
\begin{algorithmic}
\STATE{
\textbf{Inputs:}
\raisebox{-0.83cm}{\parbox{0.7\linewidth}{Observations $\boldsymbol{y}$ \\ 
Covariance matrix $\boldsymbol{\mathcal{S}}$\\
Monte Carlo vector $\boldsymbol{\varepsilon} \in \mathbb{R}^P \sim \mathcal{N}(\boldsymbol{0}_{P}, \boldsymbol{I}_{P})$\\
Finite Difference step $\nu > 0$\\
}}}
\STATE{
\textbf{Initialization:}
\raisebox{-1.75mm}{\parbox{0.7\linewidth}{ $\boldsymbol{\Lambda}^{[0]} \in \mathbb{R}^L$, Inverse Hessian $\boldsymbol{H}^{[0]} \in \mathbb{R}^{L\times L}$,\\
Gradient $\nablaL_{\boldsymbol{\Lambda}}\widehat{R}^{[0]} =  \mathrm{SUGAR}(\boldsymbol{y}, \boldsymbol{\Lambda}^{[0]}, \boldsymbol{\mathcal{S}}, \nu, \boldsymbol{\varepsilon}) $
}}}
\vspace{2mm}
\FOR{$t = 0$ \TO $ T_{\max}-1$}
\STATE{
\vspace{2mm}
\COMMENT{Descent direction from gradient of the risk estimate:}
\begin{align}
\label{eq:dir}
\boldsymbol{d}^{[t]} = - \boldsymbol{H}^{[t]} \nablaL_{\boldsymbol{\Lambda}}\widehat{R} ^{[t]}
\end{align}
\vspace{2mm}
 \COMMENT{Line search to find descent step:}
\begin{align}
\label{eq:line_search} 
\alpha^{[t]} &\in \underset{\alpha \in \mathbb{R}}{\mathrm{Argmin}} \, \widehat{R}( \boldsymbol{\Lambda}^{[t]} + \alpha \boldsymbol{d}^{[t]}),\quad  \text{with }  \widehat{R}( \boldsymbol{\Lambda}  ) = \mathrm{SURE}(\boldsymbol{y}, \boldsymbol{\Lambda}   , \boldsymbol{\mathcal{S}}, \nu , \boldsymbol{\varepsilon}) 
\end{align}
\vspace{2mm}
\COMMENT{Quasi-Newton descent step on $\boldsymbol{\Lambda}$:}
\begin{align}
\label{eq:grd_dsct}
\boldsymbol{\Lambda}^{[t+1]} = \boldsymbol{\Lambda}^{[t]} + \alpha^{[t]} \boldsymbol{d}^{[t]}
\end{align}
\vspace{2mm}
\COMMENT{Gradient update:}
\begin{align}
\nablaL_{\boldsymbol{\Lambda}}\widehat{R}^{[t+1]} =  \mathrm{SUGAR}(\boldsymbol{y}, \boldsymbol{\Lambda}^{[t+1]}, \boldsymbol{\mathcal{S}}, \nu, \boldsymbol{\varepsilon}) 
\end{align}
\vspace{2mm}
\COMMENT{Gradient increment}
\begin{align}
\boldsymbol{u}^{[t]} &=\nablaL_{\boldsymbol{\Lambda}}\widehat{R}^{[t+1]}  - \nablaL_{\boldsymbol{\Lambda}}\widehat{R}^{[t]}
\end{align}
\vspace{2mm}
\COMMENT{BFGS update of inverse Hessian~\eqref{eq:BFGS_update}:}
\begin{align}
\boldsymbol{H}^{[t+1]} &= \mathrm{BFGS}(\boldsymbol{H}^{[t]} ,  \boldsymbol{d}^{[t]} , \boldsymbol{u}^{[t]})
\end{align}
}
\ENDFOR
\vspace{2mm}
\STATE{
\textbf{Outputs:}
 \raisebox{-1.85mm}{\parbox{0.86\linewidth}{\textit{Finite-time} solution of Problem~\eqref{eq:min_sure} \hspace{1.5mm} $\widehat{\boldsymbol{\Lambda}}_{\nu, \boldsymbol{\varepsilon}}^\mathrm{BFGS} (\boldsymbol{y}\lvert \boldsymbol{\mathcal{S}}) \triangleq \boldsymbol{\Lambda}^{[T_{\max}]}$ \\
 Estimate with automated selection of $\boldsymbol{\Lambda}$ \hspace{1mm} $\widehat{\boldsymbol{x}}^{\mathrm{BFGS}}_{\nu,\boldsymbol{\varepsilon}}(\boldsymbol{y}\lvert \boldsymbol{\mathcal{S}}) \triangleq  \mathrm{PD}(\boldsymbol{y}, \boldsymbol{\Lambda}^{[T_{\max}]})$}} 
}
\end{algorithmic}
\end{algorithm}

A sketch of quasi-Newton descent, particularized to Problem~\eqref{eq:min_sure}, is detailed in Algorithm~\ref{alg:BFGS}.
It generates a sequence $\left(\boldsymbol{\Lambda}^{[t]}\right)_{t \in \mathbb{N}}$ converging toward a minimizer of $\widehat{R}_{\nu, \boldsymbol{\varepsilon}}(\boldsymbol{y}; \boldsymbol{\Lambda} \lvert \boldsymbol{\mathcal{S}})$.
This algorithm relies on a gradient descent step~\eqref{eq:grd_dsct} involving a descent direction $\boldsymbol{d}^{[t]}$ obtained from the product of BFGS approximated inverse Hessian matrix $\boldsymbol{H}^{[t]}$ and the gradient $\partial_{\boldsymbol{\Lambda}} \widehat{R}_{\nu, \boldsymbol{\varepsilon}}(\boldsymbol{y}; \boldsymbol{\Lambda} \lvert \boldsymbol{\mathcal{S}})$ obtained from SUGAR (see Algorithm~\ref{alg:SURE_SUGAR}).
The descent step size $\alpha^{[t]}$ is obtained from a line search, derived in~\eqref{eq:line_search}, which stops when Wolfe conditions are fulfilled~\cite{nocedal2006numerical, curtis2017bfgs}.
Finally, the approximated inverse Hessian matrix $\boldsymbol{H}^{[t]}$ is updated according to 
Definition~\ref{def:BFGS_update}.

\begin{remark}
The line search, Step~\eqref{eq:line_search}, is the most time consuming.
Indeed, the routines $\mathrm{SURE}$ and $\mathrm{SUGAR}$ are called for several hyperparameters of the form $\boldsymbol{\Lambda}^{[t]} + \alpha \boldsymbol{d}^{[t]}$, each call requiring to run differentiated primal-dual scheme twice.
\end{remark}

\begin{definition}[Broyden–Fletcher–Goldfarb–Shanno (BFGS)]
\label{def:BFGS_update}
Let $\boldsymbol{d}^{[t]}$ be the descent direction and $\boldsymbol{u}^{[t]}$ the gradient increment at iteration $t$, the approximated inverse Hessian matrix $\boldsymbol{H}^{[t]}$ BFGS update writes
\begin{align}
\label{eq:BFGS_update}
\boldsymbol{H}^{[t+1]}  &=\left( \boldsymbol{I}_L  - \frac{\boldsymbol{d}^{[t]} \left( \boldsymbol{u}^{[t]} \right)^\top }{\left( \boldsymbol{u}^{[t]} \right)^\top \boldsymbol{d}^{[t]} }\right)\boldsymbol{H}^{[t]} \left( \boldsymbol{I}_L  - \frac{\boldsymbol{u}^{[t]} \left( \boldsymbol{d}^{[t]} \right)^\top }{\left( \boldsymbol{u}^{[t]} \right)^\top \boldsymbol{d}^{[t]} }\right) + \alpha^{[t]}  \frac{\boldsymbol{d}^{[t]} \left( \boldsymbol{d}^{[t]}\right)^\top  }{  \left( \boldsymbol{u}^{[t]} \right)^\top \boldsymbol{d}^{[t]} }.
\end{align} 
This step constitutes a routine, named ``BFGS", defined as
\begin{align}
\boldsymbol{H}^{[t+1]}\triangleq \mathrm{BFGS}(\boldsymbol{H}^{[t]} , \boldsymbol{d}^{[t]}, \boldsymbol{u}^{[t]}).
\end{align}
\end{definition}
For detailed discussions on low memory implementations of BFGS, box constraints management, and others algorithmic tricks the interested reader is referred to~\cite{byrd1995limited, nocedal2006numerical, curtis2017bfgs}.\\

Convergence conditions for quasi-Newton algorithms relies on the behavior of second derivatives of 
the objective function~\cite{nocedal2006numerical}. 
Most of the time, when it comes to sequential estimators, one has no information about the twice differentiability of generalized SURE with respect to hyperparameters.
Hence, the convergence of Algorithm~\ref{alg:BFGS} will be assessed numerically.
Further, quasi-Newton algorithms being known to be sensitive to initialization, special attention needs to be paid to the initialization of both hyperparameters $\boldsymbol{\Lambda}$ and approximated inverse Hessian $\boldsymbol{H}$~(see Section~\ref{subsec:BFGS_num}).

\begin{remark}
Given a parametric estimator $\widehat{\boldsymbol{x}}(\boldsymbol{y};\boldsymbol{\Lambda})$, possibly obtained by another routine than $\mathrm{PD}$, Algorithms~\ref{alg:SURE_SUGAR}~and~\ref{alg:BFGS} can be used, provided that one has a routine equivalent to $\partial \mathrm{PD}$, computing $\nablaL_{\boldsymbol{\Lambda}}\widehat{\boldsymbol{x}}(\boldsymbol{y};\boldsymbol{\Lambda})$.
The reader can find other differentiated proximal algorithms in~\cite{deledalle2014stein}.
\end{remark}

\section{Hyperparameter tuning for texture segmentation}
\label{sec:text_seg}

The formalism proposed above for the automated selection of the regularization hyperparameters is now specified to total-variation based texture segmentation. 
Section~\ref{sec:tv_seg} formulates the texture segmentation problem as the minimization of a convex objective function. 
Then, in Section~\ref{sec:cast}, this segmentation procedure is cast into the general formalism of Sections~\ref{sec:SURE}~and~\ref{sec:SUGAR}. 
The hypothesis needed to apply Theorems~\ref{thm:SURE_FDMC}~and~\ref{thm:SUGAR} are discussed one by one in the context of texture segmentation. 
Finally, the practical evaluation of the estimators of the risk $\widehat{R}_{\nu, \boldsymbol{\varepsilon}}(\boldsymbol{\ell}; \boldsymbol{\Lambda}\lvert \boldsymbol{\mathcal{S}})$ and of the gradient of the risk $\nablaL_{\boldsymbol{\Lambda}} \widehat{R}_{\nu, \boldsymbol{\varepsilon}}(\boldsymbol{\ell}; \boldsymbol{\Lambda}\lvert \boldsymbol{\mathcal{S}})$ is discussed in Section~\ref{sec:practical}.

\subsection{Total-variation based texture segmentation}
\label{sec:tv_seg}

\subsubsection{Piecewise homogeneous fractal texture model}
\label{subsec:monofractal}

Let  $X \in \mathbb{R}^{N_1\times N_2}$ denote the texture to be segmented, consisting of a real-valued discrete field defined on a grid of pixels $\Omega = \lbrace 1, \hdots, N_1\rbrace \times \lbrace 1, \hdots, N_2\rbrace$. 
Texture $X$ is assumed to be formed as the union of $M$ independent Gaussian textures, existing on a set of disjoint supports,
\begin{align}
\label{eq:partition}
\Omega = \overline{\Omega}_1 \cup \cdots \cup \overline{\Omega}_M,\quad \text{with} \quad \overline{\Omega}_m \cap \overline{\Omega}_{m'} = \emptyset \quad \text{if} \quad m \neq m'.
\end{align}
Each homogeneous Gaussian texture, defined on $\Omega_m$ is characterized by two global fractal features, the scaling (or Hurst) exponent $\overline{H}_m$ and the variance $\overline{\Sigma}_m^2$, that fully control its statistics. 
Interested readers are referred to e.g., \cite{pascal2019nonsmooth} for the detailed definition of Gaussian fractal textures. 
Figures~\ref{fig:configA} and \ref{fig:configB} propose examples of such piecewise Gaussian fractal textures, with $M = 2$ and mask shown in Figure~\ref{fig:mask_vh_ellipse}.

\begin{figure}[t!]
\centering
\begin{subfigure}{3.5cm}
\centering
\includegraphics[width = 3cm]{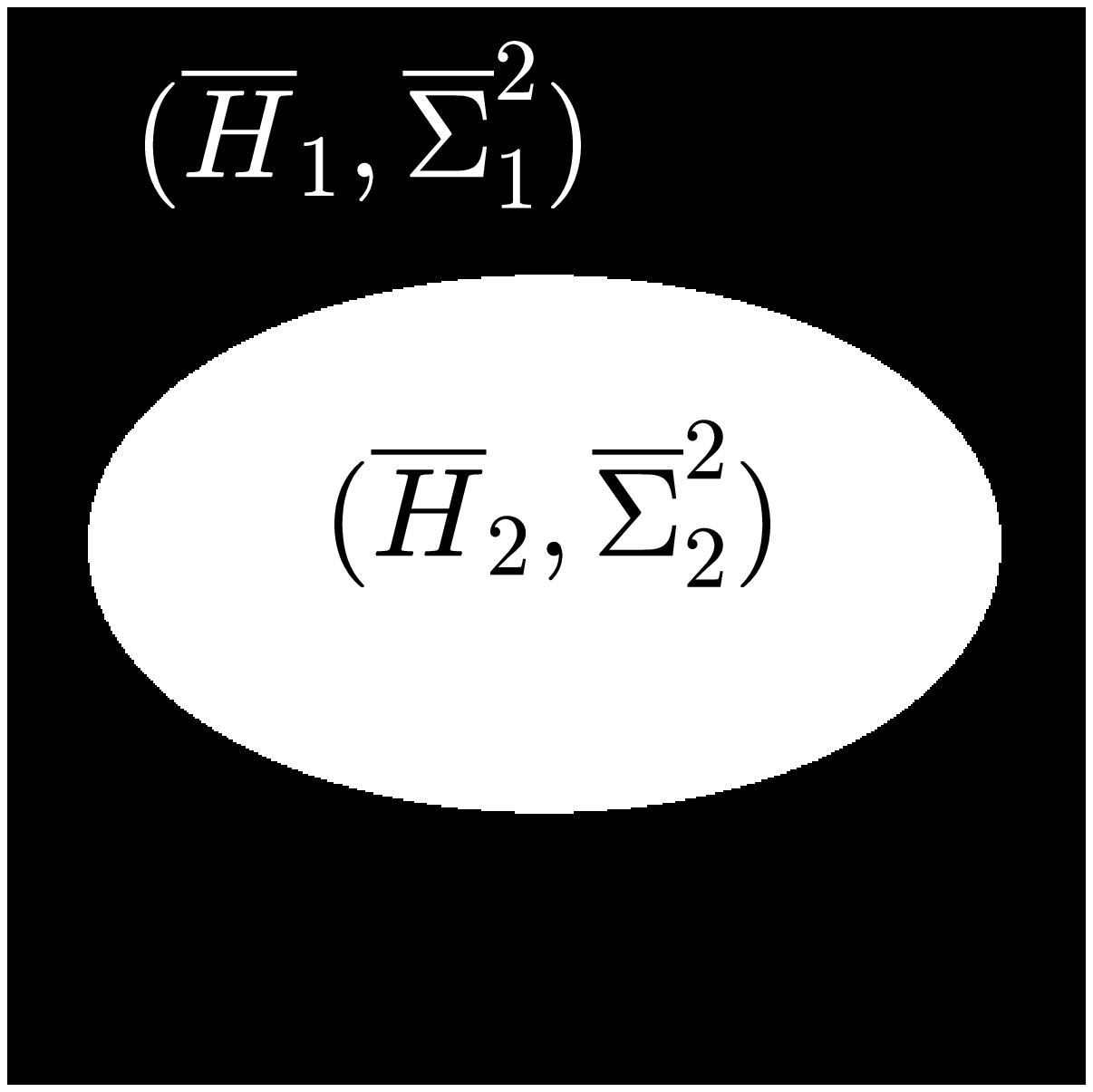}
\subcaption{ \label{fig:mask_vh_ellipse} Elliptic mask}
\end{subfigure}
\begin{subfigure}{3.5cm}
\centering
\includegraphics[width = 3cm]{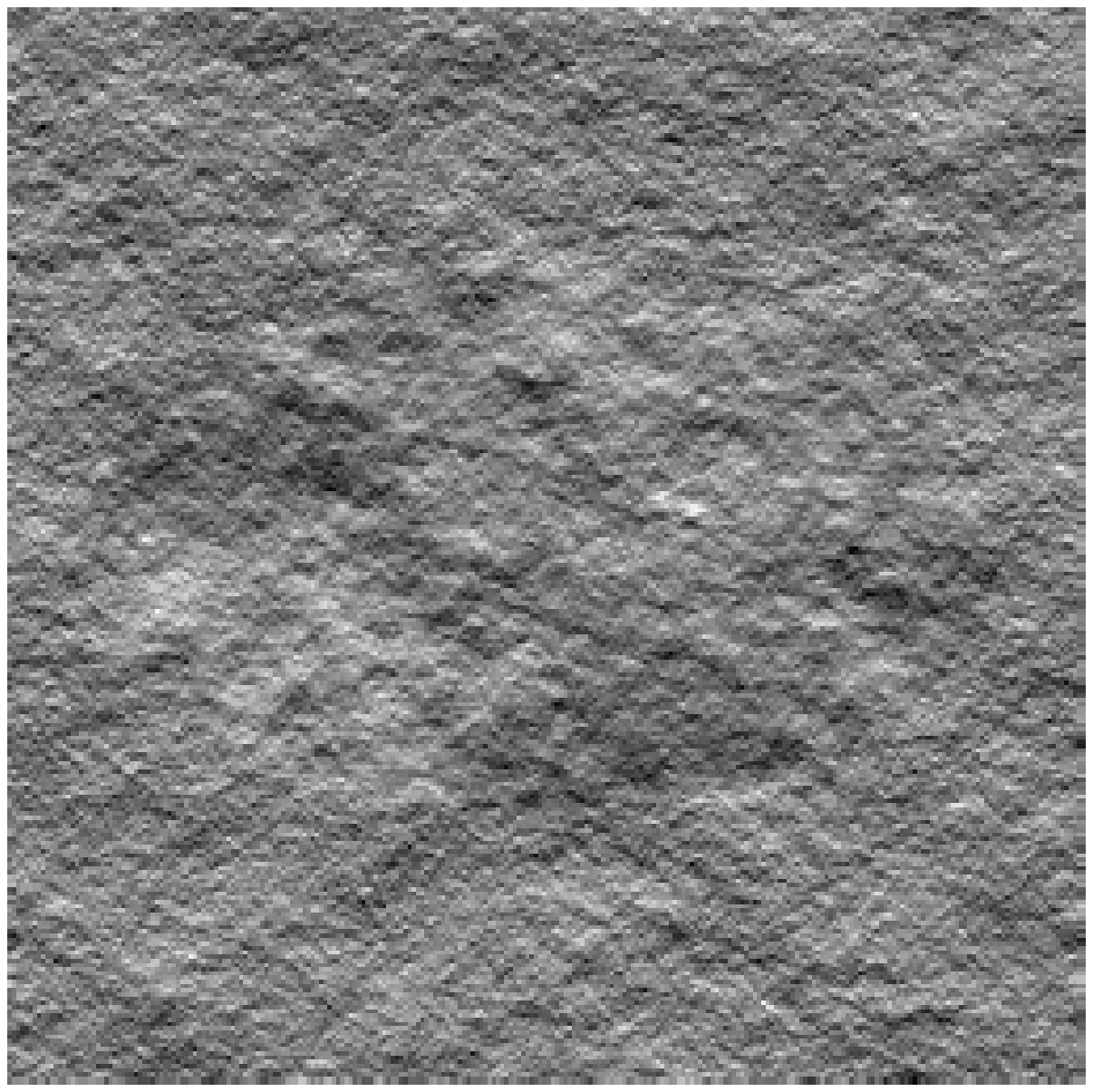}
\subcaption{\label{fig:configA} $ X$: Texture~``\textbf{D}"}
\end{subfigure}
\begin{subfigure}{3.5cm}
\centering
\includegraphics[width = 3cm]{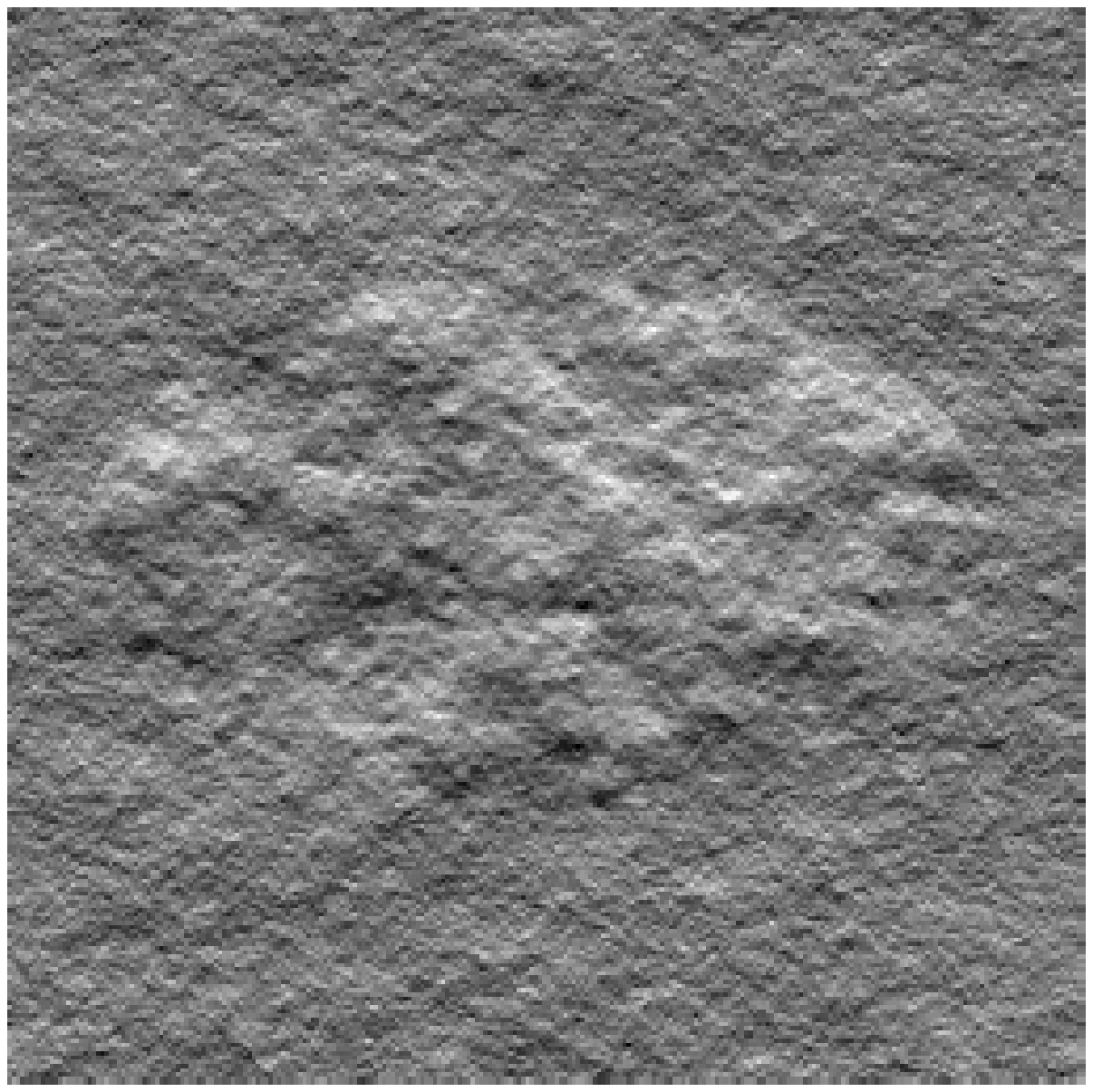}
\subcaption{\label{fig:configB} $X$: Texture~``\textbf{E}"}
\end{subfigure}
\caption{\label{fig:seg_setup}(a)~Mask for piecewise textures composed of two regions: ``background" (in black) on which the texture is characterized by homogeneous local regularity~$\bar{\boldsymbol{h}} \equiv \overline{H}_0$ and local variance~$\bar{\boldsymbol{\sigma}}^2\equiv \overline{\Sigma}^2_0$ and ``foreground" (in white) on which the texture is characterized by homogeneous local regularity~$\bar{\boldsymbol{h}} \equiv \overline{H}_1$ and local variance~$\bar{\boldsymbol{\sigma}}^2\equiv \overline{\Sigma}^2_1$. (b)~and~(c)~Synthetic piecewise homogeneous textures used for performance assessment, with resolution $256\times 256$ pixels.}
\end{figure}

\subsubsection{Local regularity and wavelet \textit{leader} coefficients}

It was abundantly discussed in the literature (cf. e.g.~\cite{Wendt2008c, Wendt2009b, wendt2009wavelet, pont2011optimized, nelson2016semi}) that textures can be well-analyzed by local fractal features (local regularity and local variance), that can be accurately estimated from wavelet \textit{leader} coefficients, as extensively described and studied in e.g.~\cite{Wendt2009b,pustelnik_combining_2016}, to which the reader is referred for a detailed presentation. 

Let  $\chi^{(d)}_{j,\underline{n}}$ 
denote the coefficients of the undecimated 2D Discrete Wavelet Transform of image $X$, at octave $ j = j_1, \hdots, j_2$ and pixel $\underline{n} \in \Omega  $, with 
the 2D-wavelet basis being defined from the 4 combination (hence the orientations $d \in \lbrace 0, 1, 2, 3 \rbrace$) of 1D wavelet $\psi $ and scaling functions. 
Interested readers are referred to e.g., \cite{Mallat:2008:WTS:1525499} for a full definition of the $\chi^{(d)}_{j,\underline{n}}$.
Wavelet \textit{leaders}, $\lbrace \mathcal{L}_{j,\underline{n}}, \, j = j_1,\hdots, j_2, \, \underline{n} \in \Omega \rbrace$, are further defined as local suprema over a spatial neighborhood and across all finest scales of the $\chi^{(d)}_{j,\underline{n}}$~\cite{Wendt2009b}:
\begin{align}
\label{eq:leaders_def}
\mathcal{L}_{j, \underline{n}}= \underset{\begin{array}{c}
d = \lbrace 1, 2, 3 \rbrace \\
\lambda_{j', \underline{n}'} \subset 3\lambda_{j,\underline{n}}
\end{array}}{\sup} \, \left \lvert 2^{j } \chi_{j', \underline{n}'}^{(d)} \right\rvert, \quad \text{where} \quad \left\lbrace \begin{array}{l}
\lambda_{j, \underline{n}} = \left[ \underline{n}, \underline{n}+2^j \right[,\\
3\lambda_{j, \underline{n}} =\underset{\underline{p} \in \lbrace -2^j, 0, 2^j\rbrace^2}{ \cup} \lambda_{j, \underline{n} + \underline{p}} .
\end{array}\right.
\end{align}
Local regularity $\bar{h}_{\underline{n}} $  and local variance $ \bar{\sigma}^2_{\underline{n}} $ at pixel $n$ can be defined via the local power law behavior of the wavelet leaders across scales \cite{Wendt2009b,wendt2009wavelet}: 
\begin{align}
\label{eq:model}
 \mathcal{L}_{j, \underline{n}} =\bar{\sigma}_{\underline{n}}  2^{j \bar{h}_{\underline{n}}}  \beta_{j,\underline{n}}, \quad \text{as} \, \, 2^j \rightarrow 0,
\end{align}
where $\beta_{j,\underline{n}}$ can be well approximated for large classes of textures~\cite{Wendt2008c} as log-normal random variables, with log-mean $\mu = 0$.
For piecewise fractal textures $X$ described in Section~\ref{subsec:monofractal}, local regularity $\bar{\boldsymbol{h}}\in \mathbb{R}^{N_1 \times N_2}$ and local variance $\bar{\boldsymbol{\sigma}}^2\in \mathbb{R}^{N_1 \times N_2}$ maps are piecewise constant, reflecting the global scaling exponent $H$ and variance $\Sigma^2$ of the homogeneous textures as: 
\begin{align}
\label{eq:leaders_model}
\left( \forall m \in \lbrace 1, \hdots, M \rbrace\right) \,  \left( \forall \underline{n} \in \overline{\Omega}_m \right) \quad\bar{h}_{\underline{n}} \equiv \overline{H}_m \text{ and } \bar{\sigma}_{\underline{n}}^2 \equiv \overline{\Sigma}_m^2 F(\overline{H}_m,\psi),
\end{align}
with $F(\overline{H}_m,\psi)$ a deterministic function studied in \cite{veitch1999wavelet} and not of interest here.\\
Taking the logarithm of Equation~\eqref{eq:model} leads to the following linear formulation
\begin{align}
\label{eq:logmodel} \ell_{j,\underline{n}} =  \bar{v}_{ \underline{n}} + j \bar{h}_{\underline{n}} + \zeta_{j,\underline{n}}, \quad \text{as} \, 2^j \rightarrow 0
\end{align}
with log-\textit{leaders} $\ell_{j,\underline{n}}=\log_2(\mathcal{L}_{j, \underline{n}})$, log-variance $ \bar{v}_{ \underline{n}} = \log_2 \bar{\sigma}_{\underline{n}}$ and zero-mean Gaussian noise $\zeta_{j,\underline{n}} = \log_2(\beta_{j,\underline{n}})$.
In the following, the \textit{leader} coefficients at scale $2^j$ are denoted $\boldsymbol{\ell}_j \in \mathbb{R}^{N_1N_2}$, and the complete collection of \textit{leaders} is stored in $\boldsymbol{\ell} \in \mathbb{R}^{JN_1N_2}$.

\subsubsection{Total variation regularization and iterative thresholding}
\label{subsec:joint}
The linear regression estimator inspired by~\eqref{eq:logmodel} 
\begin{align}
\label{eq:h_v_LR}
&\begin{pmatrix}
\widehat{\boldsymbol{h}}_{\mathrm{LR}}(\boldsymbol{\ell})\\
\widehat{\boldsymbol{v}}_{\mathrm{LR}}(\boldsymbol{\ell})
\end{pmatrix} = \underset{ \scriptsize \begin{pmatrix} \boldsymbol{h}\\ \boldsymbol{v} \end{pmatrix} \in \mathbb{R}^{2N_1N_2}}{\mathrm{argmin}} \,  \sum_{j = j_1}^{j_2} \lVert  j \boldsymbol{h} + \boldsymbol{v} - \boldsymbol{\ell}_j \rVert_2^2
\end{align}
achieves poor performance in estimating piecewise constant local regularity and local power, hence precluding an accurate segmentation of the piecewise homogeneous textures.
Thus, a functional for \textit{joint} attribute estimation and segmentation was proposed in~\cite{pascal2019nonsmooth}, leading to the following Penalized Least Squares~\eqref{eq:LS_pen}: 
\begin{align}
\label{eq:jointpb}
&\begin{pmatrix}
\widehat{\boldsymbol{h}}(\boldsymbol{\ell} ;\boldsymbol{\Lambda}) \\
\widehat{\boldsymbol{v}}(\boldsymbol{\ell} ; \boldsymbol{\Lambda})
\end{pmatrix} \in \underset{ \scriptsize \begin{pmatrix} \boldsymbol{h}\\ \boldsymbol{v} \end{pmatrix} \in \mathbb{R}^{2N_1N_2}}{\mathrm{Argmin}} \,  \sum_{j = j_1}^{j_2} \lVert  j \boldsymbol{h} + \boldsymbol{v} - \boldsymbol{\ell}_j \rVert_2^2 + \lambda_h \mathrm{TV}(\boldsymbol{h}) + \lambda_v \mathrm{TV}(\boldsymbol{v}),
\end{align}
where TV stands for the well-known isotropic Total Variation, 
defined as a mixed $\ell_{2,1}$-norm composed with spatial gradient operators
\begin{align}
\mathrm{TV}(\boldsymbol{h})  = \sum_{\underline{n} \in \Omega} \sqrt{\left( \textbf{D}_1 \boldsymbol{h} \right)^2_{\underline{n}} + \left( \textbf{D}_2 \boldsymbol{h} \right)^2_{\underline{n}}} = \sum_{\underline{n} \in \Omega} \lVert\left( \textbf{D} \boldsymbol{h} \right)_{\underline{n}}\rVert_2,
\end{align}
where $\textbf{D}_1 : \mathbb{R}^{N_1N_2} \rightarrow \mathbb{R}^{N_1N_2}$ (resp. $\textbf{D}_2 : \mathbb{R}^{N_1N_2} \rightarrow \mathbb{R}^{N_1N_2}$) stand for the discrete spatial horizontal (resp. vertical) gradient operator.
This TV-penalized least square estimator is designed to favor piecewise constancy of the estimates $\widehat{\boldsymbol{h}}$ and $\widehat{\boldsymbol{v}}$, making used of $\ell_1$-norm, i.e. $q=1$ in~\eqref{eq:LS_pen}. \\
Finally, following~\cite{cai2013multiclass,cai2018linkage}, the estimate $\widehat{\boldsymbol{h}}(\boldsymbol{\ell};\boldsymbol{\Lambda})$ is \textit{thresholded} to yield a posterior \textit{piecewise constant} map of local regularity $T \widehat{\boldsymbol{h}}(\boldsymbol{\ell};\boldsymbol{\Lambda})$, taking \textit{exactly} $M$ different values $\widehat{H}_1(\boldsymbol{\ell};\boldsymbol{\Lambda}), \hdots, \widehat{H}_M(\boldsymbol{\ell};\boldsymbol{\Lambda})$.
The resulting segmentation 
\begin{align}
\label{eq:estim_partition}
\Omega = \widehat{\Omega}_1(\boldsymbol{\ell};\boldsymbol{\Lambda}) \cup \cdots \cup \widehat{\Omega}_M(\boldsymbol{\ell};\boldsymbol{\Lambda})
\end{align}
is deduced from $T \widehat{\boldsymbol{h}}(\boldsymbol{\ell};\boldsymbol{\Lambda})$, defining
\begin{align}
\label{eq:cai_partition}
\left(\forall m\in\lbrace 1, \hdots, M\rbrace\right), \,   \widehat{\Omega}_m(\boldsymbol{\ell};\boldsymbol{\Lambda}) = \left\lbrace  \underline{n} \in \Omega \left\lvert
\left(  T \widehat{\boldsymbol{h}}(\boldsymbol{\ell};\boldsymbol{\Lambda}) \right)_{\underline{n}} \equiv \widehat{H}_m(\boldsymbol{\ell};\boldsymbol{\Lambda}) \right. \right\rbrace.
\end{align}
This is illustrated in Figure~\ref{fig:cai}, for a two-region synthetic texture with ground truth piecewise constant local regularity $\bar{\boldsymbol{h}}$~in Figure~\ref{fig:true_h}.

\begin{figure}[h!]
\centering
\begin{subfigure}{0.32\linewidth}
\flushright
\includegraphics[width = 3.5cm]{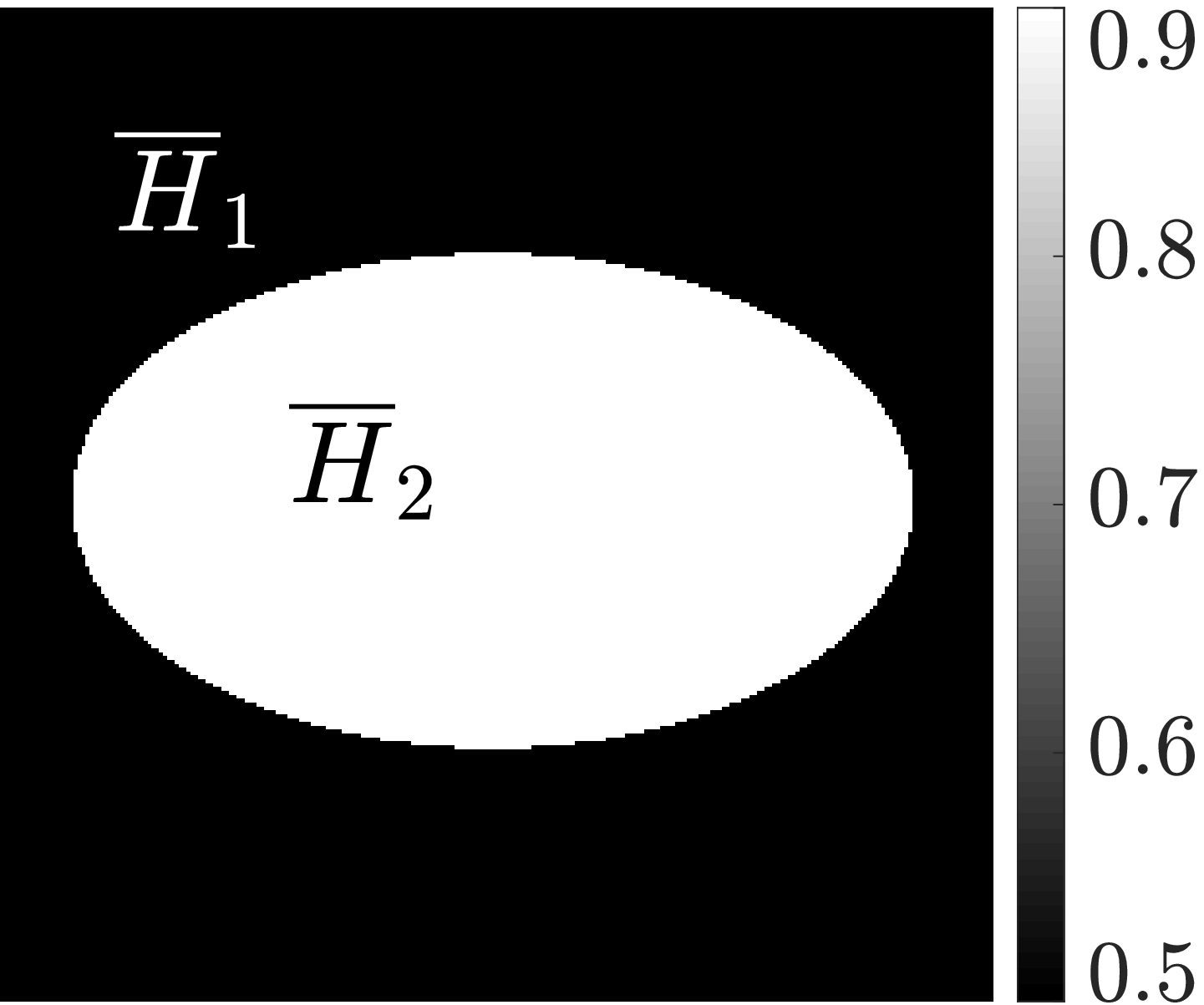}
\subcaption{\label{fig:true_h}Ground truth $\bar{\boldsymbol{h}}$}
\end{subfigure}
\begin{subfigure}{0.32\linewidth}
\flushright
\includegraphics[width = 3.5cm]{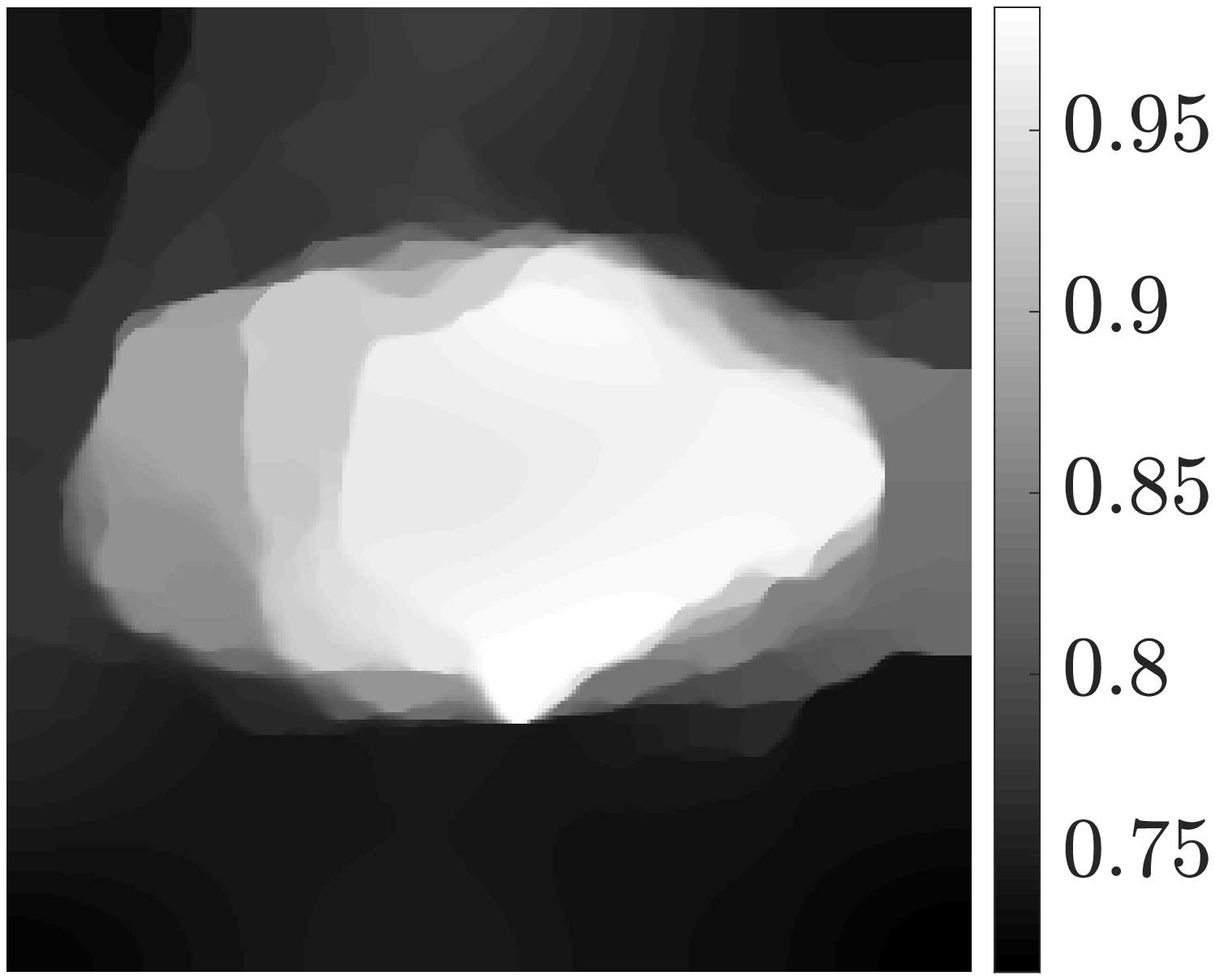}
\subcaption{Estimate $\widehat{\boldsymbol{h}}(\boldsymbol{\ell};\boldsymbol{\Lambda})$}
\end{subfigure}
\begin{subfigure}{0.32\linewidth}
\flushright
\includegraphics[width = 3.5cm]{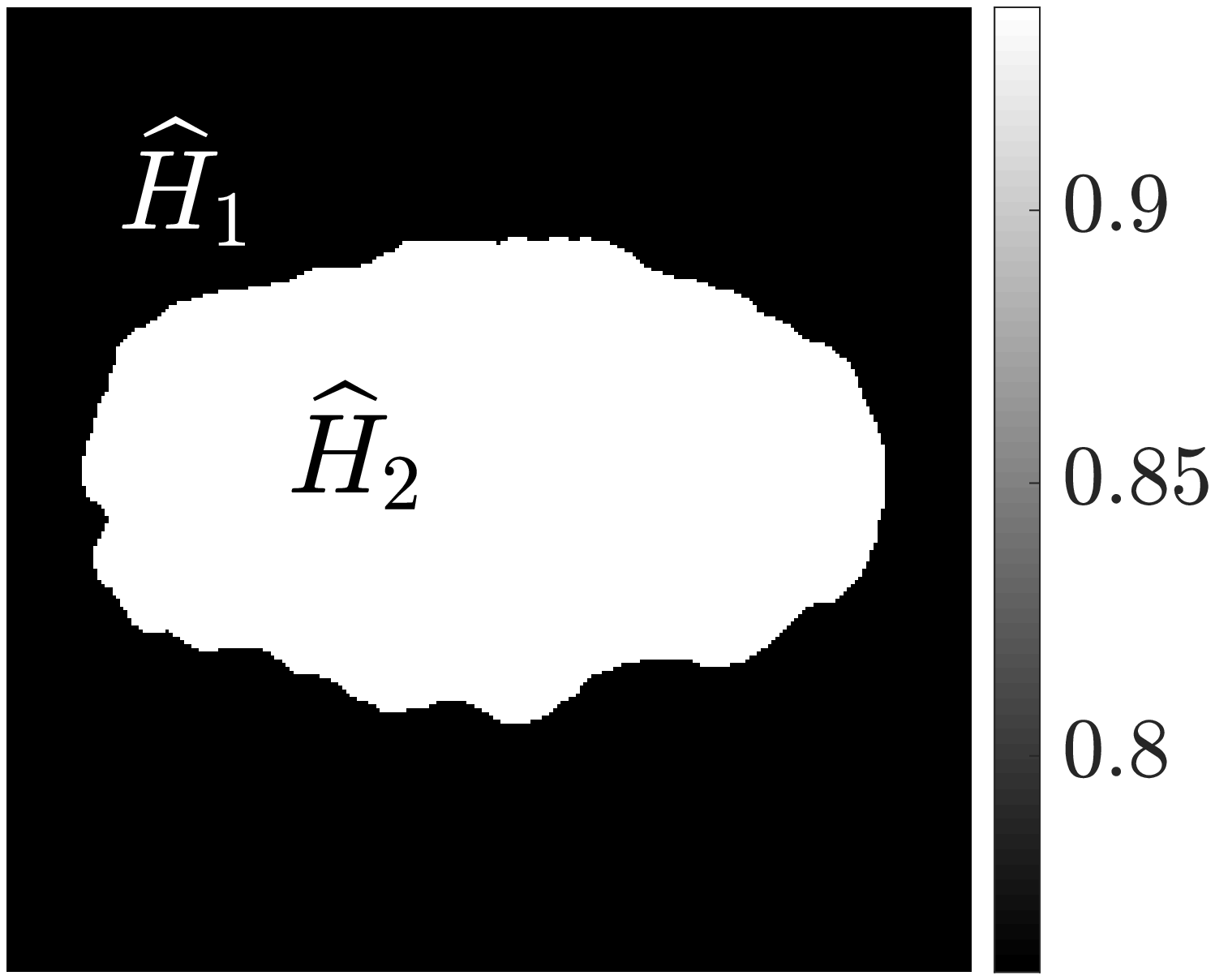}
\subcaption{Thresholded $T \widehat{\boldsymbol{h}}(\boldsymbol{\ell};\boldsymbol{\Lambda})$}
\end{subfigure}
\caption{\label{fig:cai} Example of the thresholding of $\widehat{\boldsymbol{h}}(\boldsymbol{\ell};\boldsymbol{\Lambda})$ to obtain a two-region segmentation where $\boldsymbol{\ell}$ denotes the wavelet \textit{learders} associated with the Texture $X$:~``\textbf{E}" displayed in Figure~\ref{fig:configB}}
\end{figure}

\subsection{Reformulation in term of Model~\eqref{eq:obs_gen_model}}
\label{sec:cast}
\subsubsection{Observation $\boldsymbol{y}$}

To cast the log-linear behavior~\eqref{eq:logmodel} into the general model~\eqref{eq:obs_gen_model}, vectorized quantities for $\bar{\boldsymbol{v}}$, $\bar{\boldsymbol{h}}$ and $\boldsymbol{\ell}$ are used.
The $N_1 \times N_2$ maps $\bar{\boldsymbol{h}}$ and $\bar{\boldsymbol{v}}$ are reshapped into vectors $\bar{\boldsymbol{x}}\in \mathbb{R}^{N}$, with $N = 2N_1N_2$, ordering the pixels in the lexicographic order.
The log-\textit{leaders} $\boldsymbol{\ell} = \left( \boldsymbol{\ell}_{j}\right)_{j_1\leq j \leq j_2}$, composed of $J \triangleq j_2-j_1 + 1$ octaves of resolution $N_1 \times N_2$  are vectorized, octaves by octaves, with lexical ordering of pixels, $\boldsymbol{\ell} \in \mathbb{R}^{P}$, with $P = JN_1N_2$.
 
Equation~\eqref{eq:logmodel} can then be cast into general model~\eqref{eq:obs_gen_model} as: 
\begin{align}
\label{eq:part_y}
&\text{Observations } \hspace{16mm} \boldsymbol{y} = \boldsymbol{\ell} \in \mathbb{R}^{P}, \hspace{22mm}   P = JN_1N_2\\
\label{eq:part_bx}
&\text{Ground truth } \hspace{15mm}\bar{\boldsymbol{x}} = \begin{pmatrix} \bar{\boldsymbol{h}} \\ \bar{\boldsymbol{v}} \end{pmatrix} \in \mathbb{R}^{N}, \hspace{15mm} N = 2N_1N_2\\
&\text{Linear degradation } \hspace{7mm} \label{eq:def_Phi} \boldsymbol{\Phi} : \left\lbrace \begin{array}{ll}
\mathbb{R}^{N} \rightarrow \mathbb{R}^{P} \\
\begin{pmatrix}\bar{\boldsymbol{h}} \\\bar{\boldsymbol{v}} \end{pmatrix} \mapsto \begin{pmatrix}  j\bar{\boldsymbol{h}} + \bar{\boldsymbol{v}} \end{pmatrix}_{j_1 \leq j \leq j_2}.
\end{array} \right. 
\end{align}

\subsubsection{Full-rank operator $\boldsymbol{\Phi}$}

Proposition~\ref{claim:phi_full} asserts that $\boldsymbol{\Phi}^*\boldsymbol{\Phi}$ is invertible (Assumption~\ref{hyp:full_rank}). 

\begin{claim}
\label{claim:phi_full}
The linear operator $\boldsymbol{\Phi}$ defined in~\eqref{eq:def_Phi} is bounded and its adjoint writes
\begin{align}
\label{eq:def_Phi*}
 \boldsymbol{\Phi}^* : \left\lbrace \begin{array}{ll}
\mathbb{R}^{P} \rightarrow  \mathbb{R}^{N} \\
\left(\boldsymbol{\ell}_j  \right)_{j_1 \leq j \leq j_2} \mapsto \begin{pmatrix}  \sum_{j=j_1}^{j_2} j\boldsymbol{\ell}_j \\  \\ \sum_{j=j_1}^{j_2} \boldsymbol{\ell}_j \end{pmatrix} 
\end{array} \right. 
\end{align}
Further, $\boldsymbol{\Phi}$ is full rank, and the following inversion formula holds
\begin{align}
\label{eq:phi*phi_inv}
 \left( \boldsymbol{\Phi}^*\boldsymbol{\Phi} \right)^{-1} = \frac{1}{F_2 F_0 - F_1^2} \begin{pmatrix}
 F_0 \boldsymbol{I}_{N/2} & - F_1 \boldsymbol{I}_{N/2} \\
- F_1 \boldsymbol{I}_{N/2} & F_2 \boldsymbol{I}_{N/2} 
  \end{pmatrix}, \, \, F_\alpha \triangleq \sum_{j=j_1}^{j_2} j^\alpha, \, \,\alpha \in \lbrace 0, 1, 2 \rbrace.
\end{align}
\end{claim}

\begin{proof}
Formula~\eqref{eq:def_Phi*} is obtained from straightforward computations.
Then, combining~\eqref{eq:def_Phi}~and~\eqref{eq:def_Phi*}, leads to 
\begin{align}
\label{eq:phi*phi}
\boldsymbol{\Phi}^*\boldsymbol{\Phi} = 
\begin{pmatrix}
F_2 \boldsymbol{I}_{N/2} & F_1 \boldsymbol{I}_{N/2} \\
 F_1 \boldsymbol{I}_{N/2} & F_0 \boldsymbol{I}_{N/2} 
  \end{pmatrix}
\end{align}
which is finally inverted using the $2\times 2$ cofactor matrix formula.
\end{proof}

\subsubsection{Projection operator}

Performing texture segmentation the discriminant attribute is the local regularity~$\boldsymbol{h}$, while local power $\boldsymbol{v}$ is an auxiliary feature.
Hence the \textit{projected} quadratic risk~\eqref{eq:risk_def} customized to texture segmentation reads:
\begin{align}
\label{eq:part_risk}
 R[\widehat{\boldsymbol{h}}]( \boldsymbol{\Lambda}) \triangleq \mathbb{E}_{\boldsymbol{\zeta}}  \left\lVert  \widehat{\boldsymbol{h}}(\boldsymbol{\ell} ; \boldsymbol{\Lambda}) -\bar{\boldsymbol{h}} \right\rVert_2^2,
\end{align}
with $\widehat{\boldsymbol{h}}(\boldsymbol{\ell} ; \boldsymbol{\Lambda})$ defined in~\eqref{eq:jointpb}.  

Then, the particularized projection operator in Definition~\ref{def:Pi} takes the matrix form
\begin{align}
\label{eq:part_pi}
\boldsymbol{\Pi} \triangleq \begin{pmatrix}
\boldsymbol{I}_{N/2} &\boldsymbol{Z}_{N/2} \\
\boldsymbol{Z}_{N/2} & \boldsymbol{Z}_{N/2}
\end{pmatrix} \quad \text{so that} \quad \boldsymbol{\Pi}\begin{pmatrix} \bar{\boldsymbol{h}} \\ \bar{\boldsymbol{v}} \end{pmatrix}  = \begin{pmatrix} \bar{\boldsymbol{h}} \\ \boldsymbol{0}_{N/2}\end{pmatrix}
\end{align}
where $\boldsymbol{I}_{N/2}$ (resp. $\boldsymbol{Z}_{N/2}$) denotes the identity (resp. null) matrix of size $N/2\times N/2$ and $\boldsymbol{0}_{N/2} $ the null vector of $\mathbb{R}^{N/2}$.

\subsubsection{Regularity of the estimates}

\begin{claim}
\label{claim:joint_hyp}
Problem~\eqref{eq:jointpb} has a unique solution $\begin{pmatrix}
\widehat{\boldsymbol{h}}(\boldsymbol{\ell} ;\boldsymbol{\Lambda}) \\
\widehat{\boldsymbol{v}}(\boldsymbol{\ell} ; \boldsymbol{\Lambda})
\end{pmatrix}$. \\
This solution is continuous and weakly differentiable w.r.t. $\boldsymbol{\ell}$ and integrable against the Gaussian probability density function~(Assumption~\ref{hyp:reg_int}). 
Further, both $\widehat{\boldsymbol{h}}(\boldsymbol{\ell};\boldsymbol{\Lambda})$ and $\widehat{\boldsymbol{v}}(\boldsymbol{\ell};\boldsymbol{\Lambda})$ are uniformly $L_1$-Lipschitz w.r.t. $\boldsymbol{\ell}$~(Assumption~\ref{hyp:lip_L1}). 
\end{claim}

\begin{proof}
As shown in~\cite{pascal2019nonsmooth}, the objective function 
\begin{align}
\label{eq:obj_fun}
\left(\boldsymbol{h}, \boldsymbol{v} \right) \mapsto \underbrace{\sum_{j = j_1}^{j_2} \lVert  j \boldsymbol{h} + \boldsymbol{v} - \boldsymbol{\ell}_j \rVert_2^2}_{\left\lVert  \boldsymbol{\Phi} {\scriptsize\begin{pmatrix}  \boldsymbol{h}\\ \boldsymbol{v} \end{pmatrix}}   - \boldsymbol{\ell} \right\rVert_2^2} + \lambda_h \mathrm{TV}(\boldsymbol{v}) + \lambda_v \mathrm{TV}(\boldsymbol{v})
\end{align}
is convex, being the sum of convex terms. Further, computing the eigenvalues of $\boldsymbol{\Phi}^*\boldsymbol{\Phi}$ shows that the least squares data fidelity term is $\gamma$-strongly convex, with 
\begin{align}
\label{eq:gamma_J}
\gamma = 2 \min \mathrm{Sp}(\boldsymbol{\Phi}^*\boldsymbol{\Phi}) > 0
\end{align}
where $\mathrm{Sp}(\boldsymbol{\Phi}^*\boldsymbol{\Phi})$ stand for the spectrum of the (bounded) linear operator $\boldsymbol{\Phi}^*\boldsymbol{\Phi}$. 
Hence, the objective function~\eqref{eq:obj_fun} has a unique minimum, being the unique solution of Problem~\eqref{eq:jointpb}, as mentioned in Remark~\ref{rq:full_rank_use}.\\
Further, \eqref{eq:jointpb} falls under the general formulation of Penalized Least Squares~\eqref{eq:LS_pen}, which can be written
\begin{align}
\label{eq:LS_pen_J}
\widehat{\boldsymbol{x}}(\boldsymbol{y};\boldsymbol{\Lambda} ) = \underset{{\boldsymbol{x}\in \mathcal{H}}}{\mathrm{argmin}} \, \Vert \boldsymbol{y} - \boldsymbol{\Phi} {\boldsymbol{x}}\Vert_{\boldsymbol{\mathcal{W}}}^2
+ \mathcal{J}_{\boldsymbol{\Lambda}}(\boldsymbol{x}),
\end{align}
where $\mathcal{J}_{\boldsymbol{\Lambda}}(\boldsymbol{x}) =  \lVert \textbf{U}_{\boldsymbol{\Lambda}} \boldsymbol{x} \rVert_{1}$ is built from a linear operator $\textbf{U}_{\boldsymbol{\Lambda}}$ depending on regularization parameters $\lambda_h$ and $\lambda_v$ as
\begin{align}
\boldsymbol{\Lambda} = \begin{pmatrix}
\lambda_h \\
\lambda_v
\end{pmatrix} \in \mathbb{R}_+^2, \quad \text{and} \quad \textbf{U}_{\boldsymbol{\Lambda}} = \left\lbrace \begin{array}{ll}
\mathbb{R}^{2N_1N_2} \rightarrow \mathbb{R}^{2N_1N_2} \times \mathbb{R}^{2N_1N_2} \\
\begin{pmatrix}
\boldsymbol{h} \\
\boldsymbol{v}
\end{pmatrix} \mapsto \left( \begin{pmatrix}
\lambda_h \textbf{D}_1 \boldsymbol{h} \\ 
\lambda_v \textbf{D}_1 \boldsymbol{v} 
\end{pmatrix}, 
\begin{pmatrix}
\lambda_h \textbf{D}_2 \boldsymbol{h}\\
\lambda_v \textbf{D}_2 \boldsymbol{v}
\end{pmatrix}\right).
\end{array}\right.
\end{align}
$\mathcal{J}_{\boldsymbol{\Lambda}}$ is convex, proper and lower semicontinuous, then, following~\cite{vaiter2017degrees}, \eqref{eq:LS_pen_J} can be rewritten as a constrained optimization problem
\begin{align}
\label{eq:x_Phix} \left( \widehat{\boldsymbol{x}}(\boldsymbol{y};\boldsymbol{\Lambda} ),  \widehat{\boldsymbol{z}}(\boldsymbol{y};\boldsymbol{\Lambda} ) \right) &= \underset{{\boldsymbol{x}\in \mathcal{H}, \boldsymbol{z}\in \mathcal{G}}}{\mathrm{argmin}} \, \Vert \boldsymbol{y} - \boldsymbol{z}\Vert_{\boldsymbol{\mathcal{W}}}^2
+ \mathcal{J}_{\boldsymbol{\Lambda}}(\boldsymbol{x}), \, \text{such that} \, \boldsymbol{z} = \boldsymbol{\Phi} \boldsymbol{x}\\
\label{eq:z_prox} \Longleftrightarrow \quad  \widehat{\boldsymbol{z}}(\boldsymbol{y};\boldsymbol{\Lambda} ) &= \underset{{\boldsymbol{z}\in \mathcal{G}}}{\mathrm{argmin}} \, \Vert \boldsymbol{y} -\boldsymbol{z} \Vert_{\boldsymbol{\mathcal{W}}}^2
+ \left( \boldsymbol{\Phi} \mathcal{J}_{\boldsymbol{\Lambda}}\right) (\boldsymbol{z}) \\
\label{eq:z_prox2}\Longleftrightarrow \quad  \widehat{\boldsymbol{z}}(\boldsymbol{y};\boldsymbol{\Lambda} ) &= \mathrm{prox}_{(1/2) \boldsymbol{\Phi} \mathcal{J}_{\boldsymbol{\Lambda}}}(\boldsymbol{y})
\end{align}
where 
\begin{align}
\left( \boldsymbol{\Phi} \mathcal{J}_{\boldsymbol{\Lambda}}\right) (\boldsymbol{z}) \triangleq \underset{\lbrace\boldsymbol{x} \lvert \boldsymbol{\Phi}\boldsymbol{x} = \boldsymbol{z}\rbrace}{\min} \, \mathcal{J}_{\boldsymbol{\Lambda}}(\boldsymbol{x})
\end{align}
denotes the pre-image of $\mathcal{J}_{\boldsymbol{\Lambda}}$ under $\boldsymbol{\Phi}$, which is as well convex, proper and lower semicontinuous.\\
Then, from~\eqref{eq:z_prox2}, the estimator $\widehat{\boldsymbol{z}}(\boldsymbol{y};\boldsymbol{\Lambda} )$ is non expansive, i.e. 1-Lipschitz, because the proximal operators share that same property.
Moreover, from~\eqref{eq:x_Phix}, $\widehat{\boldsymbol{z}}(\boldsymbol{y};\boldsymbol{\Lambda} ) = \boldsymbol{\Phi} \widehat{\boldsymbol{x}}(\boldsymbol{y};\boldsymbol{\Lambda} )$, and since $\boldsymbol{\Phi}$ is full-rank according to Proposition~\ref{claim:phi_full}
\begin{align}
\widehat{\boldsymbol{x}}(\boldsymbol{y};\boldsymbol{\Lambda} ) = \left( \boldsymbol{\Phi}^* \boldsymbol{\Phi} \right)^{-1} \boldsymbol{\Phi}^* \widehat{\boldsymbol{z}}(\boldsymbol{y};\boldsymbol{\Lambda} ).
\end{align}
$\boldsymbol{\Phi}$ being bounded, we conclude that the estimator $\widehat{\boldsymbol{x}}(\boldsymbol{y};\boldsymbol{\Lambda} )$ is uniformly $L_1$-Lipschitz, with $L_1 = \lVert \boldsymbol{\Phi}\rVert^{-1}$ justifying Assumption~\ref{hyp:lip_L1},~\textit{(i)}.\\
Being uniformly $L_1$-Lipschitz, $\widehat{\boldsymbol{x}}(\boldsymbol{y};\boldsymbol{\Lambda} )$ is continuous and weakly-differentiable (see Theorem~5 of Section~4.2.3 in~\cite{evans2015measure}).
As a consequence, both $\left\langle \textbf{A}^* \boldsymbol{\Pi}\widehat{\boldsymbol{x}}(\boldsymbol{y};\boldsymbol{\Lambda}), \boldsymbol{\zeta} \right\rangle$ and $\nablaL_{\boldsymbol{\Lambda}} \widehat{\boldsymbol{x}}(\boldsymbol{y};\boldsymbol{\Lambda} )$ are integrable against the Gaussian density and Assumption~\ref{hyp:reg_int} holds.\\
Finally, 
setting $\boldsymbol{y} = \boldsymbol{0}_P$, for any $\boldsymbol{\Lambda} \in \mathbb{R}^L$, $\widehat{\boldsymbol{x}}(\boldsymbol{0}_P; \boldsymbol{\Lambda}) = \boldsymbol{0}_N$ reaches the minimum. The solution being unique from Proposition~\ref{claim:phi_full}, $\boldsymbol{0}_N$ is the unique solution and Assumption~\ref{hyp:lip_L1},~\textit{(ii)} is verified. \\
Further, it is reasonable to expect that the uniform Lipschitzianity with respect to hyperparameters results of Remark~\ref{rq:check_L2}, extend to the estimator $\widehat{\boldsymbol{h}}(\boldsymbol{\ell}; \boldsymbol{\Lambda})$, defined in~\eqref{eq:jointpb}.
Yet, to the best of our knowledge, no direct proof that Lipschitzianity Assumption~\ref{hyp:lip_L2} holds for general Penalized Least Squares exists.
This issue is a scientific question in itself and will be addressed in future work.
\end{proof}

\subsection{Practical computation of $\widehat{R}_{\nu, \boldsymbol{\varepsilon}}$ and $\nablaL_{\boldsymbol{\Lambda}}\widehat{R}_{\nu, \boldsymbol{\varepsilon}}$}
\label{sec:practical}

This section addresses all technical issues encountered in running Algorithm~\ref{alg:SURE_SUGAR}, in the context of texture segmentation described above.
 
\subsubsection{Covariance structure of the observations}
\label{subsec:leaders_cov}

The additive noise $\zeta_{j,\underline{n}}$ appearing in Equation~\eqref{eq:logmodel} being Gaussian, Gaussianity Assumption~\ref{hyp:gauss_noise} holds.
The covariance matrix $\boldsymbol{\mathcal{S}}$ of noise $\boldsymbol{\zeta}$ reads
\begin{align}
\label{eq:spatial_cor}
\mathcal{S}_{j,\underline{n}}^{j', \underline{n}'}  \triangleq \mathbb{E} \, \zeta_{j,\underline{n}} \zeta_{j',\underline{n}'}  = \mathcal{C}_{j}^{j'} \Xi_{j}^{j'}(\underline{n} - \underline{n}'),
\end{align}
where
\begin{align}
\label{eq:inter_cor}
\mathcal{C}_{j}^{j'} \triangleq \mathbb{E} \, \zeta_{j,\underline{n}} \zeta_{j',\underline{n}}, \quad \mathcal{C}_{j}^{j'}\text{ independent of }\underline{n} 
\end{align}
quantifies the inter-scale covariance, and $\boldsymbol{\Xi}_{j}^{j'}$ encapsulate the stationary spatial correlations, with correlation length proportional to $\max(2^j, 2^{j'})$.

\subsubsection{Matrix product $\boldsymbol{\mathcal{S}}\boldsymbol{\varepsilon}$}

Following Remark~\ref{rq:sparsity_S}, in general, the direct product $\boldsymbol{\mathcal{S}}\boldsymbol{\varepsilon}$ required for the practical evaluation of Finite Difference Monte Carlo SURE~\eqref{eq:SURE_FDMC_def} is intractable because of the large size of matrix $\boldsymbol{\mathcal{S}}$. 
Yet, in the case of log-\textit{leaders}, the spatial correlations presenting the Toeplitz structure~\eqref{eq:spatial_cor}, the product $\boldsymbol{\mathcal{S}} \boldsymbol{\varepsilon}$ can be computed efficiently, making~\eqref{eq:SURE_FDMC_def} usable in practice.\\
Indeed, given $\boldsymbol{\varepsilon} \in \mathbb{R}^{JN_1N_2} = \left( \boldsymbol{\varepsilon}_j\right)_{j = j_1}^{j_2}J$, with $\boldsymbol{\varepsilon}_j \in \mathbb{R}^{N_1N_2}$, 
\begin{align}
\left( \boldsymbol{\mathcal{S}} \boldsymbol{\varepsilon} \right)_{j, \underline{n}} &= \sum_{j'=j_1}^{j_2} \sum_{\underline{n}' \in \Omega} \mathcal{S}_{j,\underline{n}}^{j', \underline{n}'}  \varepsilon_{j',\underline{n}'} 
= \sum_{j'=j_1}^{j_2} \mathcal{C}_{j}^{j'} \sum_{\underline{n}' \in \Omega} \Xi_{j}^{j'}(\underline{n} - \underline{n}') \varepsilon_{j', \underline{n}'} 
 = \sum_{j'=j_1}^{j_2}  \mathcal{C}_{j}^{j'}  \boldsymbol{\Xi}_{j,j'} \ast \boldsymbol{\varepsilon}_j.
\end{align}
which is the sum of $J$ convolution products, denoted $\ast$, of high dimensional vector $\boldsymbol{\varepsilon} \in \mathbb{R}^{N_1N_2}$ with ``low dimensional" finite support window $\mathcal{C}_{j}^{j'} \boldsymbol{\Xi}_{j,j'} $. Hence evaluating $\boldsymbol{\mathcal{S}} \boldsymbol{\varepsilon} $ appears to be far less costly than a general product of matrix of size $P \times P$ by a vector of size $P$.

\subsubsection{Operator $\textbf{A}$}

In the same vein, the matrices $\boldsymbol{\Phi}^*$~\eqref{eq:def_Phi*}, $\left(\boldsymbol{\Phi}^* \boldsymbol{\Phi}\right)^{-1}$~\eqref{eq:phi*phi_inv} and $\boldsymbol{\Pi}$~\eqref{eq:part_pi} turn out to be very sparse, since they act independently on each pixel. 
Thus, the same sparse (pixel-wise) structure follows for 
\begin{align}
\label{eq:sparse_A}
\textbf{A} = \boldsymbol{\Pi} \left(\boldsymbol{\Phi}^* \boldsymbol{\Phi}\right)^{-1}\boldsymbol{\Phi}^* = \frac{1}{F_0F_2 -F_1^2} \begin{pmatrix}
(F_0+j_1F_1) \boldsymbol{I}_{N/2} & \cdots & (F_0+j_2F_1) \boldsymbol{I}_{N/2} \\
\boldsymbol{Z}_{N/2} & \cdots & \boldsymbol{Z}_{N/2}
\end{pmatrix}.
\end{align}
Hence the products $\textbf{A}^*\boldsymbol{\Pi} \widehat{\boldsymbol{x}}(\boldsymbol{y}; \boldsymbol{\Lambda})$ and $\textbf{A}^*\boldsymbol{\Pi} \nablaL_{\boldsymbol{\Lambda}} \widehat{\boldsymbol{x}}(\boldsymbol{y}; \boldsymbol{\Lambda})$, appearing in the Finite Difference Monte Carlo risk~\eqref{eq:SURE_FDMC_def} and gradient of the risk~\eqref{eq:SUGAR_FDMC} estimators, are very cheap to compute, involving $\mathcal{O}(N)$ operations.\\

\subsubsection{Evaluation of $\mathrm{Tr}(\textbf{A} \boldsymbol{\mathcal{S}} \textbf{A}^*)$}

The evaluation of risk estimate $\widehat{R}_{\nu, \boldsymbol{\varepsilon}}(\boldsymbol{\ell}; \boldsymbol{\Lambda}\lvert \boldsymbol{\mathcal{S}})$, at Step~\eqref{eq:SURE_eval} of generalized SURE and SUGAR Algorithm~\ref{alg:SURE_SUGAR} requires the computation of the trace of an $N\times N$ matrix, with $N$ possibly of order $10^6$, e.g. in image processing.\\
In the present application, combining the structure of covariance matrix $\boldsymbol{\mathcal{S}}$~\eqref{eq:spatial_cor} and the sparse expression of $\textbf{A}$~\eqref{eq:sparse_A}, provides a compact expression of the third term of generalized SURE~\eqref{eq:SURE_FDMC_def} detailed in Proposition~\ref{claim:const_term}, which can be evaluated with very little computational effort.

\begin{claim}[Third term of Stein Unbiased Risk Estimate]
\label{claim:const_term}
Consider texture's \textit{leader} coefficients~\eqref{eq:logmodel}, whose covariance matrix $\boldsymbol{\mathcal{S}}$ evidences the sparse structure described in~\eqref{eq:spatial_cor}. 
Define the linear operator $\textbf{A}$ from Formula~\eqref{eq:def_A}, using operator $\boldsymbol{\Phi}$~\eqref{eq:def_Phi} and projector $\boldsymbol{\Pi}$~\eqref{eq:part_pi}.
Then, the third term of Stein estimator of the risk~\eqref{eq:SURE_FDMC_def} reads 
\begin{align}
\label{eq:const_term}
\mathrm{Tr}(\textbf{A} \boldsymbol{\mathcal{S}} \textbf{A}^*  )  = \frac{N/2}{\left( F_0 F_2 - F_1^2 \right)^2}  \left(\sum_{j,j' } \left( F_1^2 \mathcal{C}_{j}^{j'} - 2 F_0F_1  j' \mathcal{C}_{j}^{j'} +  F_0^2  j j' \mathcal{C}_{j}^{j'} \right)  \right),
\end{align}
where the quantities $\left\lbrace F_{\alpha}, \, \alpha = 0, 1, 2\right\rbrace$ are defined in~\eqref{eq:phi*phi_inv} and $\mathcal{C}_j^{j'}$ denotes the covariance between scales $2^j$ and $2^{j'}$, as defined in~\eqref{eq:inter_cor}.
\end{claim}
\begin{proof}
Proof is postponed to Appendix~\ref{app:const_term}.
\end{proof}

\section{Hyperparameter tuning performance assessment}
\label{sec:numerics}

The aim of this section is to assess quantitatively, by means of numerical simulations, the performance in the estimation of the optimal hyperparamaters.
To that end, Section~\ref{subsec:num_set} will detail the numerical simulation set-up and Section~\ref{subsec:algo_set} will concentrate on several algorithmic issues. 
Section~\ref{subsec:log_cov_struct} will show on the prominent role of covariance matrix $\boldsymbol{\mathcal{S}}$, evaluating the impact of partial vs. full covariance matrix in Section~\ref{subsec:inf_cor} and comparing true vs. estimated covariance matrix in Section~\ref{subsec:estim_cov}. 
Section~\ref{subsec:BFGS_exp} will further assess quantitatively how well optimal hyperparameters are estimated in the absence of available ground truth, with respect to different quality metrics.

\subsection{Numerical simulation set-up}
\label{subsec:num_set}

\subsubsection{Textures}
\label{subsec:synth_text}

For sake of simplicity, we consider the two-region case $M=2$, with elliptic mask displayed in Figure~\ref{fig:mask_vh_ellipse}. 
Synthetic textures of resolution $N_1 \times N_2 = 256 \times 256$, characterized by two attributes configurations:
\begin{itemize}
\item Configuration ``\textbf{D}", ``difficult", one realization being displayed in Figure~\ref{fig:configA} 
\begin{center}
\begin{tabular}{ll}
$\left( \overline{H}_1, \overline{\Sigma}_1^2 \right) = (0.5,0.6)$ & (background),\\
$\left( \overline{H}_2, \overline{\Sigma}_2^2 \right) = (0.75,0.7)$ & (central ellipse).\\
\end{tabular}
\end{center}
\vspace{3mm}
\item Configuration ``\textbf{E}", ``easy", one realization being displayed in Figure~\ref{fig:configB} 
\begin{center}
\begin{tabular}{ll}
$\left( \overline{H}_1, \overline{\Sigma}_1^2 \right) = (0.5,0.6)$ & (background),\\
$\left( \overline{H}_2, \overline{\Sigma}_2^2 \right) = (0.9,1.1)$ & (central ellipse).\\
\end{tabular}
\end{center}
\end{itemize}
are generated from a Matlab routine designed by ourselves (see~\cite{pascal2019nonsmooth}).

\subsubsection{Multiscale analysis}

A 2D undecimated wavelet transform of the textured image is computed at scale $2^j$, with mother wavelet obtained as a tensor product of 1D least asymmetric Daubechies wavelets, with 3 vanishing moments, see~\cite{Mallat:2008:WTS:1525499} for more details.

\subsubsection{Performance evaluation}

Following~\cite{pascal2018joint,pascal2019nonsmooth}, for a given textured image $X$, and the derived log-\textit{leaders} $\left( \boldsymbol{\ell}_j\right)_{j = j_1}^{j_2}$,
two performance indices are used:
\begin{itemize}
\item The \textit{one-sample quadratic risk} on local regularity, computed from \textit{one sample} of log-\textit{leaders} $\boldsymbol{\ell}$ computed on the single image $X$
\begin{align}
\label{eq:def_R}
\mathcal{R}(\boldsymbol{\ell}; \boldsymbol{\Lambda}) \triangleq  \left\lVert \widehat{\boldsymbol{h}}(\boldsymbol{\ell}; \boldsymbol{\Lambda}) - \bar{\boldsymbol{h}} \right\rVert^2_2,
\end{align}
with estimator $\widehat{\boldsymbol{h}}(\boldsymbol{\ell};\boldsymbol{\Lambda})$ defined in~\eqref{eq:jointpb} and ground truth $\bar{\boldsymbol{h}}$ defined in~\eqref{eq:leaders_model}. 
\item The \textit{segmentation error}, defined as the percentage of incorrectly classified pixels
\begin{align}
\label{eq:def_P}
\mathcal{P}(\boldsymbol{\ell}; \boldsymbol{\Lambda}) \triangleq \left\lvert \overline{\Omega}_1 \cap \widehat{\Omega}_2(\boldsymbol{\ell}; \boldsymbol{\Lambda})\right\rvert + \left\lvert \widehat{\Omega}_1(\boldsymbol{\ell}; \boldsymbol{\Lambda}) \cap \overline{\Omega}_2\right\rvert,
\end{align}
where $\cup_m \widehat{\Omega}_m(\boldsymbol{\ell}\boldsymbol{\Lambda})$ is the estimated partition~\eqref{eq:estim_partition}, obtain from TV-based texture segmentation, as described in Section~\ref{subsec:joint}.
\end{itemize}

\begin{remark}
By definition of the quadratic risk~\eqref{eq:part_risk} and one-sample quadratic risk~\eqref{eq:def_R}, 
$\mathbb{E}_{\boldsymbol{\zeta}}\mathcal{R}(\boldsymbol{\ell};\boldsymbol{\Lambda}) =R[\widehat{\boldsymbol{h}}](\boldsymbol{\Lambda})$. 
In practice however, only one realization of $\boldsymbol{\ell}$ is available, hence the quadratic risk $R[\widehat{\boldsymbol{h}}](\boldsymbol{\Lambda})$ is not accessible. 
Thus, in the following experiments, the \textit{one-sample quadratic risk} $\mathcal{R}(\boldsymbol{\ell};\boldsymbol{\Lambda})$, defined in~\eqref{eq:def_R}, is used as a reference to which Stein risk estimator $\widehat{R}_{\nu, \boldsymbol{\varepsilon}}(\boldsymbol{\ell}; \boldsymbol{\Lambda} \lvert \boldsymbol{\mathcal{S}})$ will be compared.
\end{remark}

\subsection{Algorithmic set-up}
\label{subsec:algo_set}
\subsubsection{Primal dual with iterative differentiation}

Problem~\eqref{eq:jointpb} is solved using the accelerated primal-dual algorithm~\ref{alg:PD}, with primal variable $\boldsymbol{x} \triangleq \left( \boldsymbol{h}, \boldsymbol{v}\right)$, taking advantage of strong-convexity of the data fidelity term. 
The maximal number of iterations is set to $K_{\max} = 5 \, 10^5$, and a threshold on the normalized duality gap is set to $10^{-4}$ 
(see~\cite{pascal2019nonsmooth}).

\subsubsection{Scaling range}

The estimation of piecewise constant local attributes requires to focus on fine scales.
Thus, ideally, the least square term~\eqref{eq:h_v_LR} would involve the two finest scales of the multiscale representation, and range from $j_1 = 1$ to $j_2 = 2$. 
Yet, the efficiency of acceleration strategy of Algorithm~\ref{alg:PD} increases with the strong-convexity modulus $\gamma$~\eqref{eq:gamma_J}, displayed in Table~\ref{tab:gamma_J}, which is observed to increase with $j_2$, as $j_1 = 1$ is fixed. 
Thus, a trade-off between locality and convergence speed leads to select $j_2 = 3$.

\begin{table}[h!]
\centering
\begin{tabular}{lccccc}
\toprule
 & $j_2=2$ & $j_2=3$ & $j_2=4$ & $j_2=5$ & $j_2=6$ \\
 \midrule
$\gamma$ & $0.29$ & $\boldsymbol{0.72}$ & $1.20$ & $1.69$ & $2.20$ \\
\bottomrule
\end{tabular}
\caption{\label{tab:gamma_J} Strong-convexity modulus $\gamma$ of data-fidelity term of~\eqref{eq:obj_fun}, computed from Formula~\eqref{eq:gamma_J}, for fixed $j_1=1$ and varied $j_2$. The bold entry correspond to the range of scales used in the experiments of Sections~\ref{sec:numerics}.}
\end{table}

\subsubsection{Finite Difference Monte Carlo parameters}

The Monte Carlo vector $\boldsymbol{\varepsilon} \in \mathbb{R}^{P}$, $P = JN_1N_2$, is drawn randomly, according to a i.i.d. normalized Gaussian $\mathcal{N}(\boldsymbol{0}_P, \boldsymbol{I}_P)$. 
We adapt the heuristic of~\cite{deledalle2014stein} or the Finite Difference step $\nu$ to the case of correlated noise as 
\begin{align}
\nu = \frac{2}{P^{\alpha}}\max \left(\sqrt{C_j^j}, \, j \in \lbrace 1, \hdots, J\rbrace \right), \quad \alpha = 0.3,
\end{align}
where $C_j^j$ is the variance of the log-\textit{leaders} $\boldsymbol{\ell}_j$ at scale $2^j$.\\
The derivatives with respect to hyperparameters of the estimates, $\nablaL_{\boldsymbol{\Lambda}} \widehat{\boldsymbol{h}}$, $\nablaL_{\boldsymbol{\Lambda}} \widehat{\boldsymbol{v}}$, are obtained by iterative differentiation of primal dual algorithm, customized to texture segmentation in Appendix~\ref{alg:PD}.

\subsubsection{BFGS quasi-Newton initialization and parameters}
\label{subsec:BFGS_num}
To perform the risk minimization sketched in Algorithm~\ref{alg:BFGS}, we used the GRadient-based Algorithm for Non-Smooth Optimization, implemented in GRANSO toolbox\footnote{\texttt{http://www.timmitchell.com/software/GRANSO/}}, from the BFGS quasi-Newton algorithm proposed in~\cite{curtis2017bfgs}.
It consists of a low memory BFGS algorithm with box constraints, enabling to enforce positive $\lambda_h$ and $\lambda_v$. 
The maximal number of iterations of BFGS Algorithm~\ref{alg:BFGS} is set to $K_{\max} = 250$, while the stopping criterion on the gradient norm is set to $10^{-6}$.\\
As mentioned in Section~\ref{subsec:BFGS}, the initialization of quasi-Newton algorithms might drastically impact their convergence.
Hence, we propose a model-based strategy for initializing $\boldsymbol{\Lambda}$ and $\boldsymbol{H}$.
The initialization of $\lambda_h$ and $\lambda_v$ is performed by balancing the data fidelity term and the penalization appearing of functional~\eqref{eq:obj_fun}.
The data fidelity term grows like the variance of the noise
\begin{align}
\mathbb{E} \sum_{j = j_1}^{j_2} \lVert  j \bar{\boldsymbol{h}} + \bar{\boldsymbol{v}} - \boldsymbol{\ell}_j \rVert_2^2 = \mathrm{tr}(\boldsymbol{\mathcal{S}}),
\end{align}
and the penalization term can be evaluated using $\left(\widehat{\boldsymbol{h}}_{\mathrm{LR}},\widehat{\boldsymbol{v}}_{\mathrm{LR}}\right)$ introduced in~\eqref{eq:h_v_LR}.
Thus, the initial hyperparameters $\boldsymbol{\Lambda}$ for BFGS Algorithm~\ref{alg:BFGS} are set to
\begin{align}
\label{eq:init_Lambda}
\boldsymbol{\Lambda}^{[0]} = \left( \lambda_h^{[0]},\lambda_v^{[0]} \right), \quad \text{where} \quad \lambda_h^{[0]} = \frac{\mathrm{tr}(\boldsymbol{\mathcal{S}})}{2\, \mathrm{TV}(\widehat{\boldsymbol{h}}_{\mathrm{LR}}(\boldsymbol{\ell}))}, \text{ and } \lambda_v^{[0]} = \frac{\mathrm{tr}(\boldsymbol{\mathcal{S}})}{2\, \mathrm{TV}(\widehat{\boldsymbol{v}}_{\mathrm{LR}}(\boldsymbol{\ell}))}.
\end{align}
The \textit{inverse} Hessian matrix $\boldsymbol{H}^{[0]} \in \mathbb{R}^{2 \times 2}$, is initialized to enforce $\boldsymbol{\Lambda}^{[1]} = (1 \pm \kappa) \boldsymbol{\Lambda}^{[0]}$.
It is chosen diagonal with coefficients
\begin{align}
\label{eq:init_H}
\boldsymbol{H}^{[0]} = \mathrm{diag} \left( 
 \left\lvert \frac{\kappa \lambda_h^{[0]}}{\partial_{\lambda_h} \widehat{R}_{\nu, \boldsymbol{\varepsilon}}(\boldsymbol{\ell}; \boldsymbol{\Lambda}^{[0]}\lvert \boldsymbol{\mathcal{S}})} \right\rvert ,  \left\lvert \frac{\kappa \lambda_v^{[0]}}{\partial_{\lambda_v} \widehat{R}_{\nu, \boldsymbol{\varepsilon}}(\boldsymbol{\ell}; \boldsymbol{\Lambda}^{[0]}\lvert \boldsymbol{\mathcal{S}})} \right\rvert  \right).
\end{align}
In practice, we used $\kappa = 0.5$ for all experiments.
It is observed that this choice of $\boldsymbol{H}^{[0]}$ avoids the first iteration falling away from natural hyperpamaters scaling~\eqref{eq:init_Lambda}, which would induce huge computational cost to reach to optimal hyperparameters.

\subsection{Covariance of \textit{leaders}}
\label{subsec:log_cov_struct}

\subsubsection{Covariance estimation procedure}

No closed-form formula exists to compute exactly the covariance matrix $\boldsymbol{\mathcal{S}}$ from the texture's attributes. 
Hence, from \textit{one} sample $\boldsymbol{\ell}$, computed from a single texture $X$, the \textit{estimated} covariance matrix, denoted $\widehat{\boldsymbol{\mathcal{S}}}$, is computed using classic sample covariance estimator:
\begin{align}
\label{eq:spat_av}
\widehat{\mathcal{S}}_{j,\underline{n}}^{j',\underline{n}'} \triangleq \frac{1}{\lvert \Omega \rvert} \sum_{\underline{n} \in \Omega } \ell_{j,\underline{n}} \ell_{j',\underline{n} + \delta \underline{n}} - \left( \frac{1}{\lvert \Omega \rvert} \sum_{\underline{n} \in \Omega }\ell_{j,\underline{n}} \right) \left( \frac{1}{\lvert \Omega \rvert} \sum_{\underline{n} \in \Omega }\ell_{j',\underline{n} } \right),
\end{align}
for spatial lag $\delta \underline{n} \triangleq \underline{n}' - \underline{n}$, leading to inter-scale covariance
\begin{align}
\label{eq:spat_av_c}
\widehat{\mathcal{C}}_j^{j'} = \widehat{\mathcal{S}}_{j,\underline{n}}^{j',\underline{n}}
\end{align}
and spatial correlations
\begin{align}
\label{eq:spat_av_xi}
\widehat{\Xi}_{j}^{j'}(\delta \underline{n}) = \frac{\widehat{\mathcal{S}}_{j,\underline{n}}^{j',\underline{n}'}}{\widehat{\mathcal{C}}_j^{j'}}.
\end{align}
Then, for Textures~``\textbf{D}"~and~``\textbf{E}", a \textit{true} covariance matrix $\boldsymbol{\mathcal{S}}$ is obtained numerically by averaging the above estimated covariance matrix $\widehat{\boldsymbol{\mathcal{S}}}^{(q)}$ over $Q = 5000$ texture samples as:
\begin{align}
\label{eq:smpl_av}
\boldsymbol{\mathcal{S}} \triangleq \left\langle \widehat{\boldsymbol{\mathcal{S}}}^{(q)}\right\rangle_{q=1}^Q,
\end{align}
the samples being generated with the mathematical model of~\cite{pascal2019nonsmooth}.

\subsubsection{Impact of partial versus full covariance on estimated risk}
\label{subsec:inf_cor}

We now assess the impact of using two partial versions of the \textit{full} \textit{true} covariance matrix $\boldsymbol{\mathcal{S}}$, described in~\eqref{eq:smpl_av}: 
\begin{enumerate}
\item \textit{Variance} matrix $\boldsymbol{\mathcal{S}}_{\mathrm{var}}$ neglecting both inter-scale and spatial correlations, reduces to the variances $\mathcal{C}_j^j$ of the $\boldsymbol{\ell}_j$'s, and hence is diagonal
\begin{align}
{S_{\mathrm{var}}}_{j, \underline{n}}^{j', \underline{n}'} = \mathcal{C}_j^j \delta_{j,j'} \delta_{\underline{n}, \underline{n}'}.
\end{align}
\item \textit{Inter-scale} covariance matrix $\boldsymbol{\mathcal{S}}_{\mathrm{int}}$, neglecting spatial correlations, reduces to cross-correlations $\mathcal{C}_j^{j'}$ between the $\boldsymbol{\ell}_j$'s and the $\boldsymbol{\ell}_{j'}$'s at same location
\begin{align}
\label{eq:inter_scale_cov}
{S_{\mathrm{int}}}_{j, \underline{n}}^{j', \underline{n}'} = \mathcal{C}_j^{j'} \delta_{\underline{n}, \underline{n}'}.
\end{align}
\end{enumerate}

For texture~``\textbf{D}", both $\widehat{R}_{\nu,\boldsymbol{\varepsilon}}(\boldsymbol{\ell};\boldsymbol{\Lambda} \lvert \boldsymbol{\mathcal{S}}_{\mathrm{var}})$ (Fig.~\ref{subfig:fdmc_var_1}) and $\widehat{R}_{\nu,\boldsymbol{\varepsilon}}(\boldsymbol{\ell};\boldsymbol{\Lambda} \lvert \boldsymbol{\mathcal{S}}_{\mathrm{int}})$ (Fig.~\ref{subfig:fdmc_inter_1}) fail to reproduce $\mathcal{R}(\boldsymbol{\ell};\boldsymbol{\Lambda})$ (Fig.~\ref{subfig:risk_1}).
Hence, the selected hyperparameters $\widehat{\boldsymbol{\Lambda}}_{\nu, \boldsymbol{\varepsilon}}^\dagger (\boldsymbol{\ell}\lvert \boldsymbol{\mathcal{S}}_{\mathrm{var}})$~(`$\boldsymbol{\square}$') and $\widehat{\boldsymbol{\Lambda}}_{\nu, \boldsymbol{\varepsilon}}^\dagger (\boldsymbol{\ell}\lvert \boldsymbol{\mathcal{S}}_{\mathrm{int}})$~(`$\boldsymbol{\diamond}$') do not coincide with the optimal $\boldsymbol{\Lambda}_{\mathcal{R}}$~(`$\boldsymbol{+}$').
The corresponding segmentations, $T \widehat{\boldsymbol{h}}(\boldsymbol{\ell};\widehat{\boldsymbol{\Lambda}}_{\nu, \boldsymbol{\varepsilon}}^\dagger (\boldsymbol{\ell}\lvert \boldsymbol{\mathcal{S}}_{\mathrm{var}}))$ (Fig.~\ref{subfig:segh_var_1}) and $T \widehat{\boldsymbol{h}}(\boldsymbol{\ell};\widehat{\boldsymbol{\Lambda}}_{\nu, \boldsymbol{\varepsilon}}^\dagger (\boldsymbol{\ell}\lvert \boldsymbol{\mathcal{S}}_{\mathrm{int}}))$ (Fig.~\ref{subfig:segh_inter_1}), differ significantly from the targeted $T \widehat{\boldsymbol{h}}(\boldsymbol{\ell};\boldsymbol{\Lambda}_{\mathcal{R}})$~(Fig.~\ref{subfig:segh_risk_1}).\\
On the opposite, $\widehat{R}_{\nu,\boldsymbol{\varepsilon}}(\boldsymbol{\ell};\boldsymbol{\Lambda} \lvert \boldsymbol{\mathcal{S}})$ (Fig.~\ref{subfig:fdmc_true_1}) perfectly matches $\mathcal{R}(\boldsymbol{\ell};\boldsymbol{\Lambda})$ (Fig.~\ref{subfig:risk_1}). Thanks to the exact computation of the constant term $\mathrm{Tr}(\textbf{A} \boldsymbol{\mathcal{S}} \textbf{A}^* )$ in Proposition~\ref{claim:const_term}, the order of magnitude $\mathcal{R}(\boldsymbol{\ell};\boldsymbol{\Lambda})$ is well reproduced by $\widehat{R}_{\nu,\boldsymbol{\varepsilon}}(\boldsymbol{\ell};\boldsymbol{\Lambda} \lvert \boldsymbol{\mathcal{S}})$, as observed on the colorbars in Figure~\ref{fig:S_DE_13}.
Further, $\widehat{\boldsymbol{\Lambda}}_{\nu, \boldsymbol{\varepsilon}}^\dagger (\boldsymbol{\ell}\lvert \boldsymbol{\mathcal{S}}_{\mathrm{int}})$~(`$\boldsymbol{\bigtriangleup}$') coincides with $\boldsymbol{\Lambda}_{\mathcal{R}}$~(`$\boldsymbol{+}$'), leading to segmentation $T \widehat{\boldsymbol{h}}(\boldsymbol{\ell};\widehat{\boldsymbol{\Lambda}}_{\nu, \boldsymbol{\varepsilon}}^\dagger (\boldsymbol{\ell}\lvert \boldsymbol{\mathcal{S}}))$ (Fig.~\ref{subfig:segh_true_1}) similar to $T \widehat{\boldsymbol{h}}(\boldsymbol{\ell};\boldsymbol{\Lambda}_{\mathcal{R}})$~(Fig.~\ref{subfig:segh_risk_1}).\\
Similar observations can be made for Texture~``\textbf{E}" at columns~3,~4 of Figure~\ref{fig:S_DE_13}.\\

Altogether, these two examples illustrate that the \textit{full} covariance is necessary so that $\widehat{R}_{\nu,\boldsymbol{\varepsilon}}(\boldsymbol{\ell};\boldsymbol{\Lambda} \lvert \boldsymbol{\mathcal{S}})$ provides an accurate estimate of $\mathcal{R}(\boldsymbol{\ell};\boldsymbol{\Lambda})$.
Moreover, $\boldsymbol{\Lambda}_{\mathcal{R}}$ appears to be well approximated by the optimal hyperparameters $\boldsymbol{\Lambda}_{\nu, \boldsymbol{\varepsilon}}^\dagger (\boldsymbol{\ell}\lvert \boldsymbol{\mathcal{S}})$, obtained using \textit{full} covariance.

\begin{figure}[h!]
\centering
\begin{subfigure}{0.49\linewidth}
\centering
Texture ``\textbf{D}"
\end{subfigure}
\begin{subfigure}{0.49\linewidth}
\centering
Texture ``\textbf{E}"
\end{subfigure}

\vspace{5mm}

\begin{subfigure}{0.24\linewidth}
\centering
\includegraphics[height = 2.5cm]{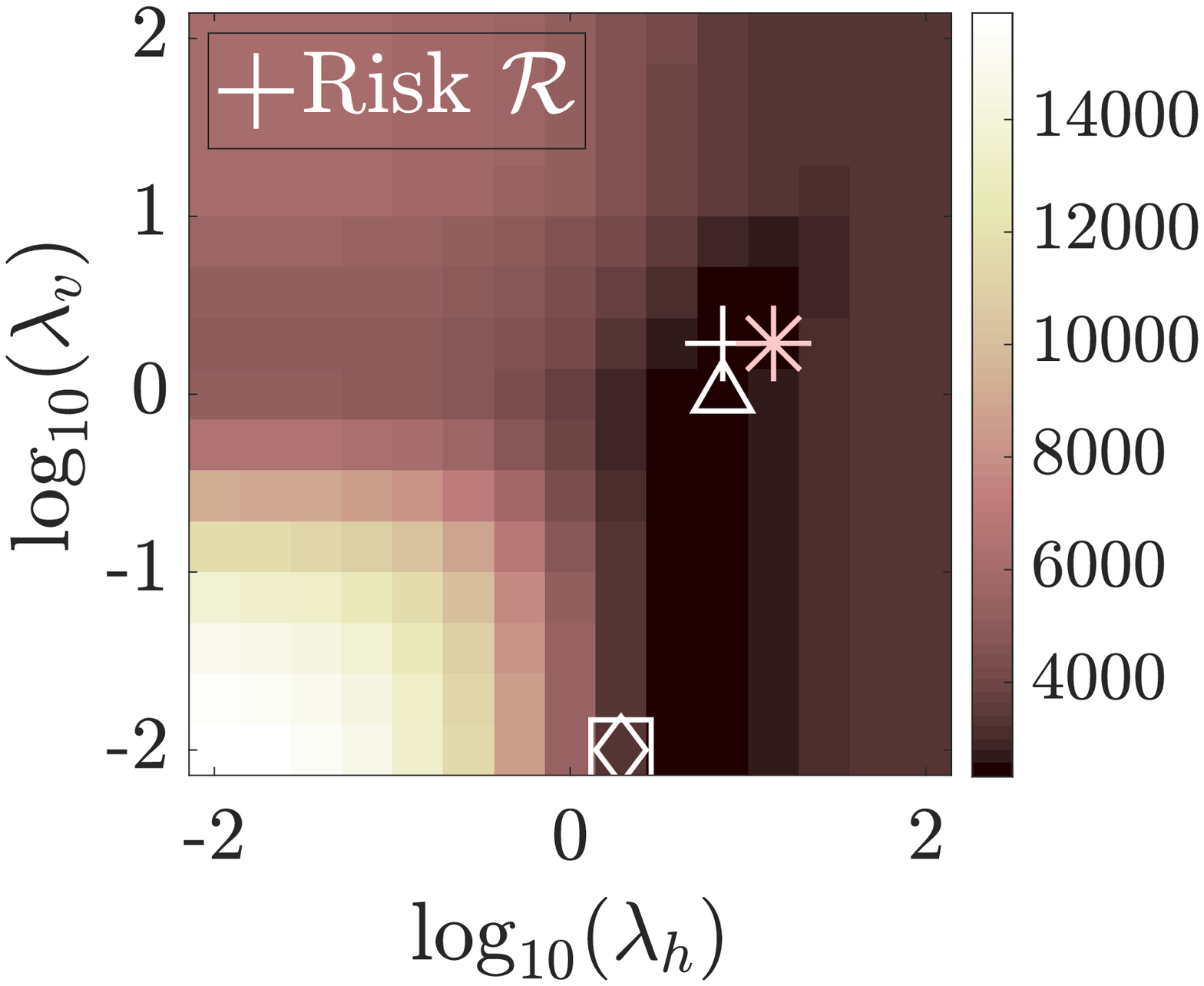}
\subcaption{\label{subfig:risk_1}$\mathcal{R}(\boldsymbol{\ell};\boldsymbol{\Lambda})$}
\end{subfigure}
\begin{subfigure}{0.24\linewidth}
\centering
\begin{minipage}[t][2.5cm]{3cm}
\centering
\includegraphics[width = 2cm]{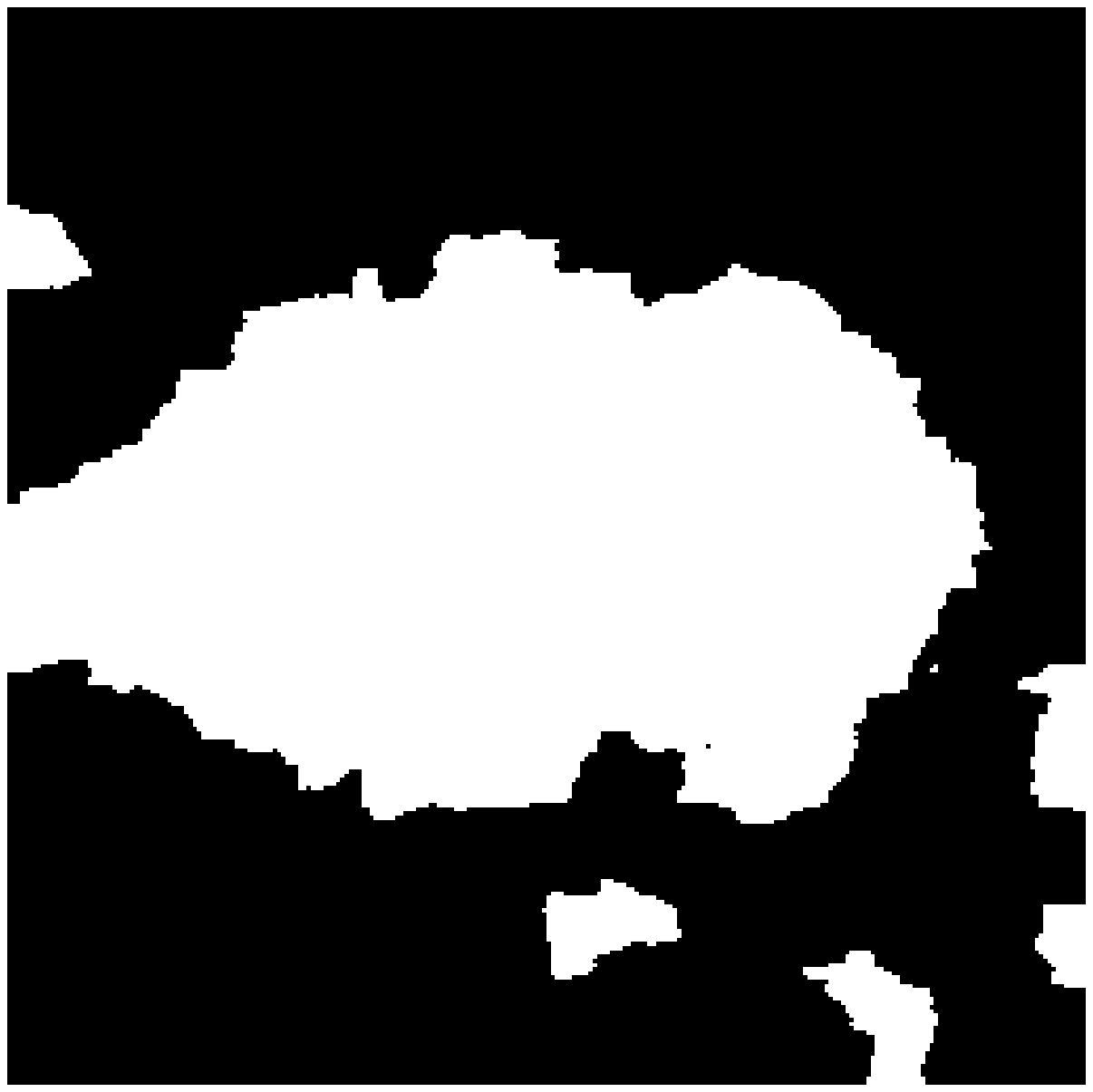}
\end{minipage}
\subcaption{\label{subfig:segh_risk_1}Min. $\mathcal{R}$ `+'}
\end{subfigure}
\begin{subfigure}{0.24\linewidth}
\centering
\includegraphics[height = 2.5cm]{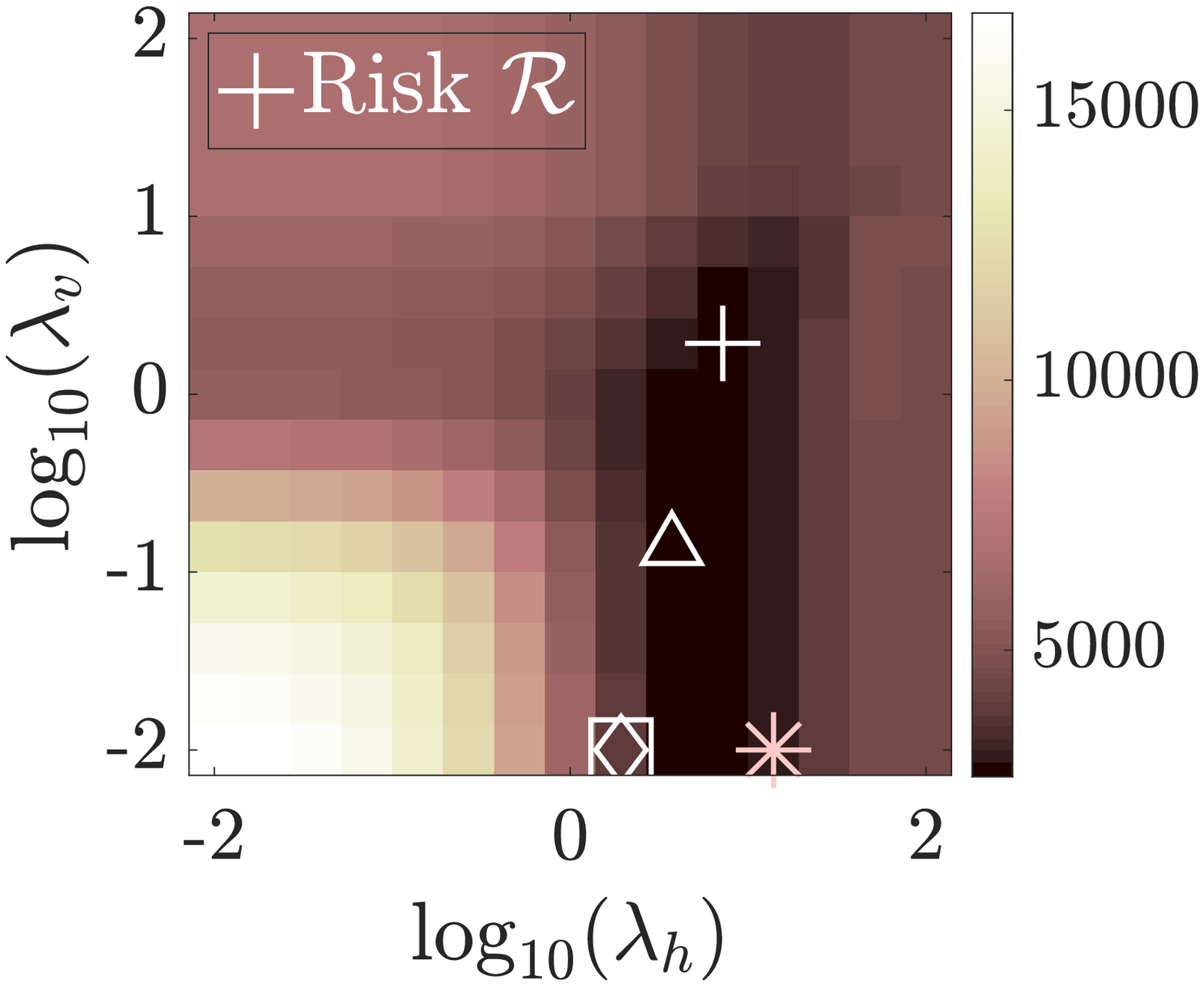}
\subcaption{\label{subfig:risk_4}$\mathcal{R}(\boldsymbol{\ell};\boldsymbol{\Lambda})$}
\end{subfigure}
\begin{subfigure}{0.24\linewidth}
\centering
\begin{minipage}[t][2.5cm]{3cm}
\centering
\includegraphics[width = 2cm]{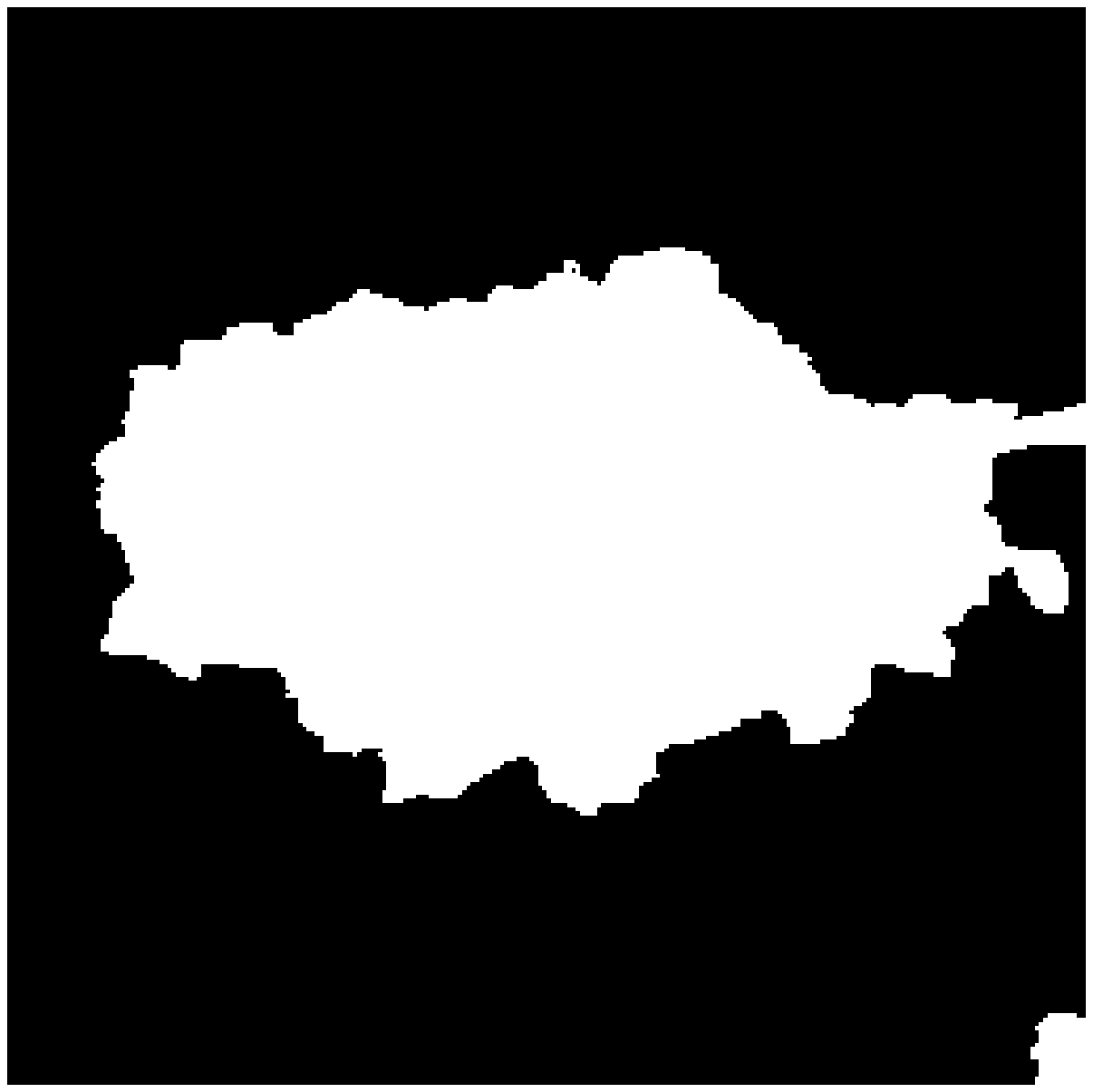}
\end{minipage}
\subcaption{\label{subfig:segh_risk_4}Min. $\mathcal{R}$ `+'}
\end{subfigure}\\

\begin{subfigure}{0.24\linewidth}
\centering
\includegraphics[height = 2.5cm]{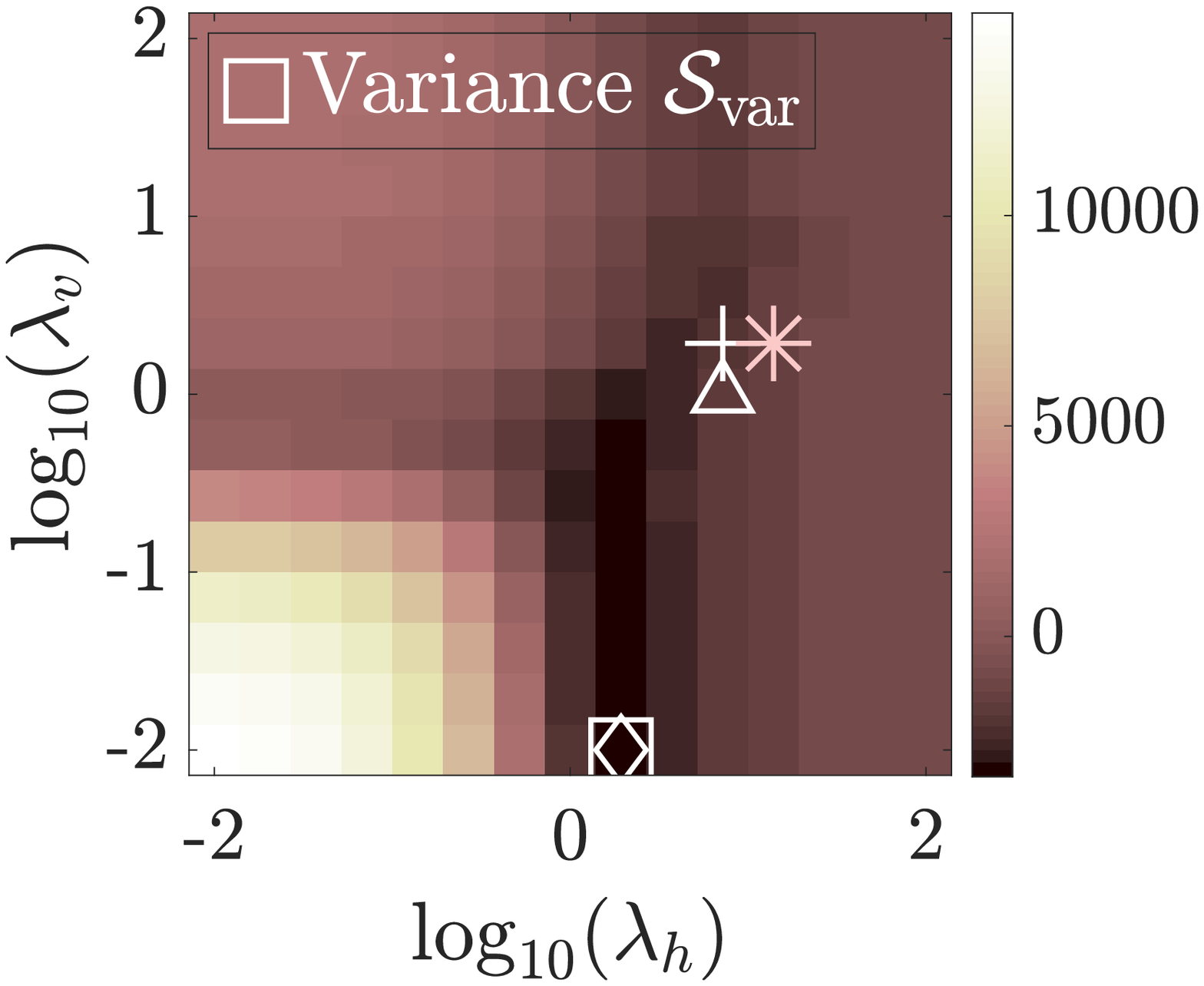}
\subcaption{\label{subfig:fdmc_var_1}$\widehat{R}_{\nu,\boldsymbol{\varepsilon}}(\boldsymbol{\ell};\boldsymbol{\Lambda} \lvert \boldsymbol{\mathcal{S}}_{\mathrm{var}})$}
\end{subfigure}
\begin{subfigure}{0.24\linewidth}
\centering
\begin{minipage}[t][2.5cm]{3cm}
\centering
\includegraphics[width = 2cm]{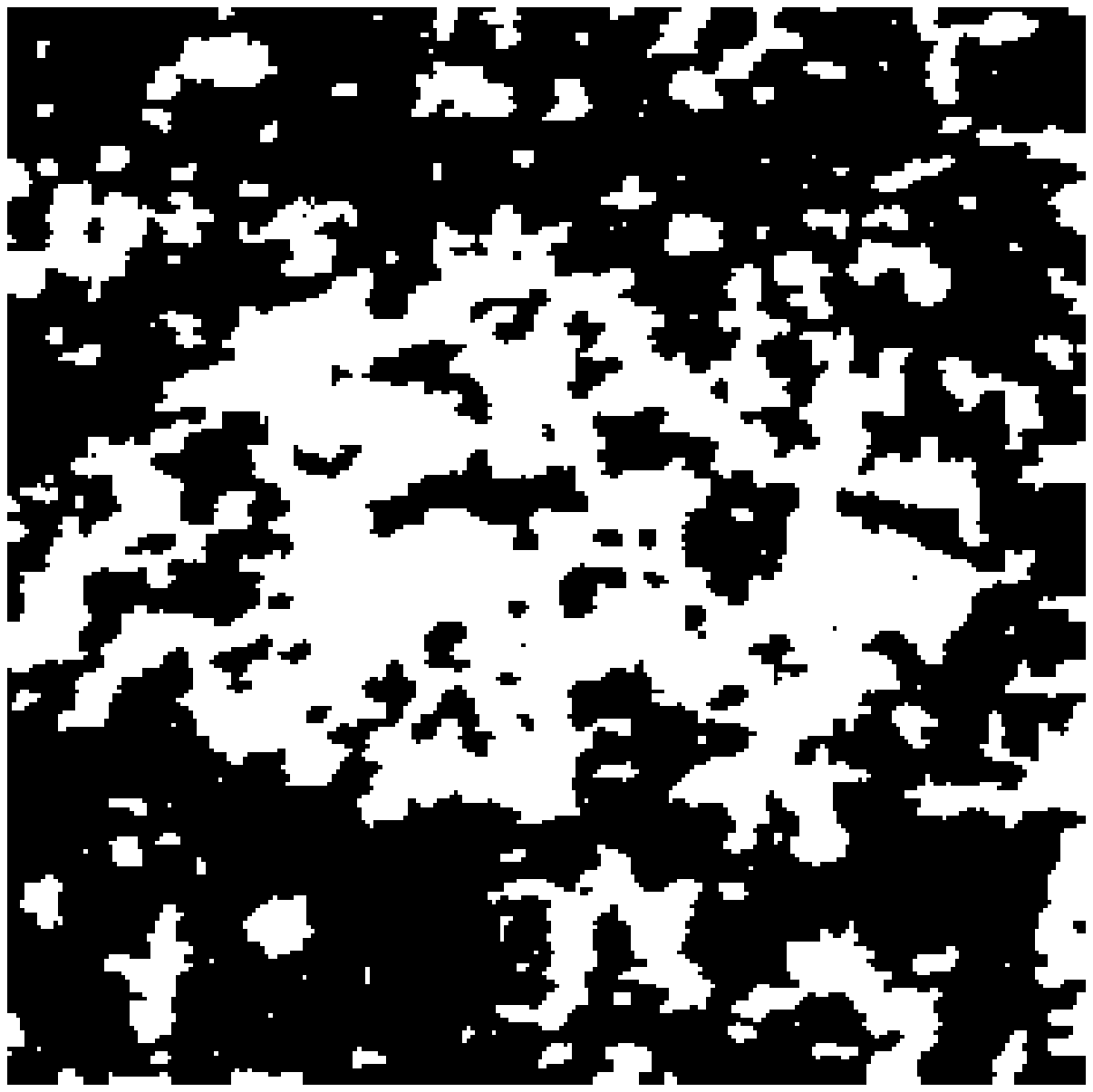}
\end{minipage}
\subcaption{\label{subfig:segh_var_1}Min. $\widehat{R}_{.}(\cdot \lvert \boldsymbol{\mathcal{S}}_{\mathrm{var}})$ `$\boldsymbol{\square}$'}
\end{subfigure}
\begin{subfigure}{0.24\linewidth}
\centering
\includegraphics[height = 2.5cm]{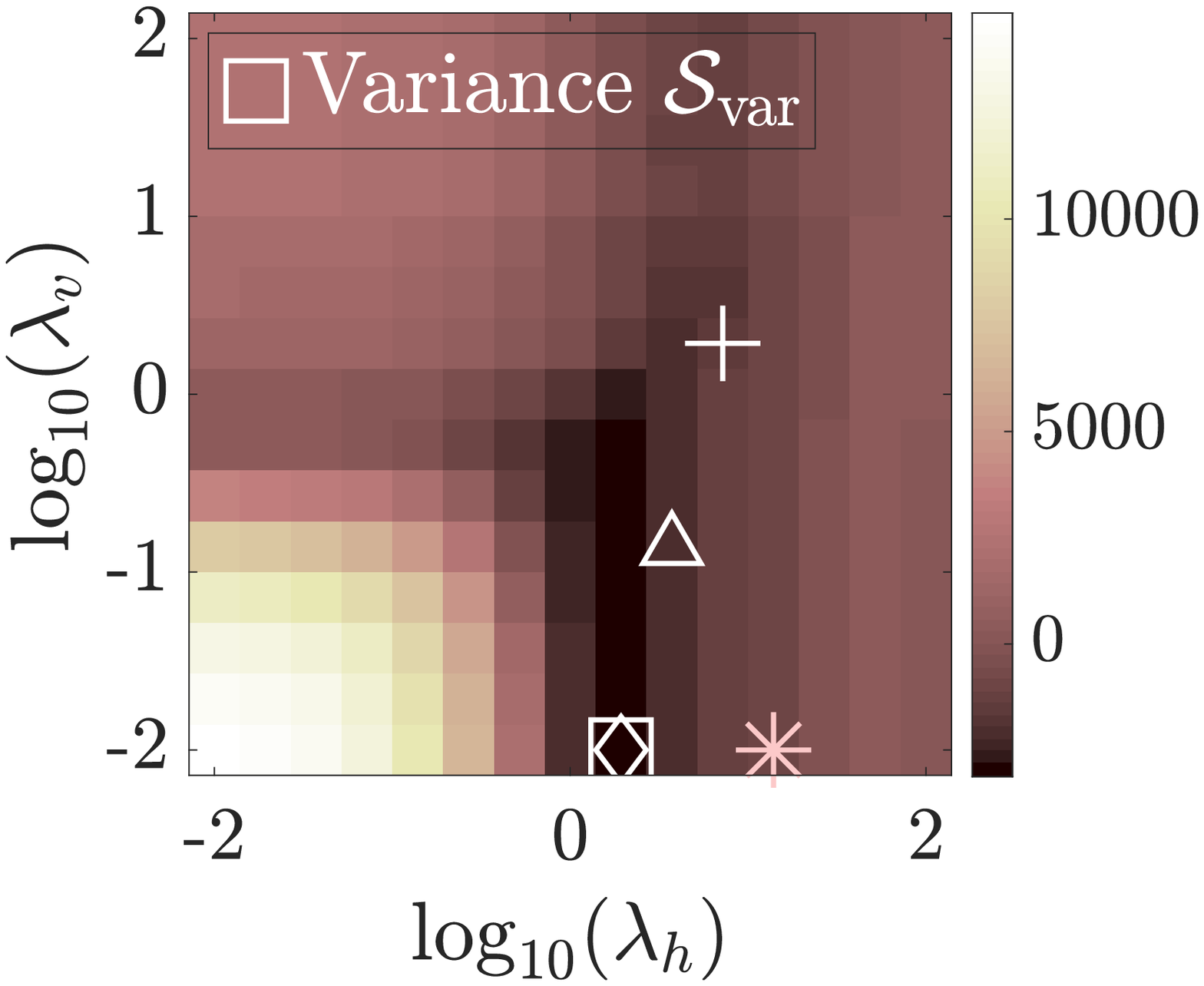}
\subcaption{\label{subfig:fdmc_var_4}$\widehat{R}_{\nu,\boldsymbol{\varepsilon}}(\boldsymbol{\ell};\boldsymbol{\Lambda} \lvert \boldsymbol{\mathcal{S}}_{\mathrm{var}})$}
\end{subfigure}
\begin{subfigure}{0.24\linewidth}
\centering
\begin{minipage}[t][2.5cm]{3cm}
\centering
\includegraphics[width = 2cm]{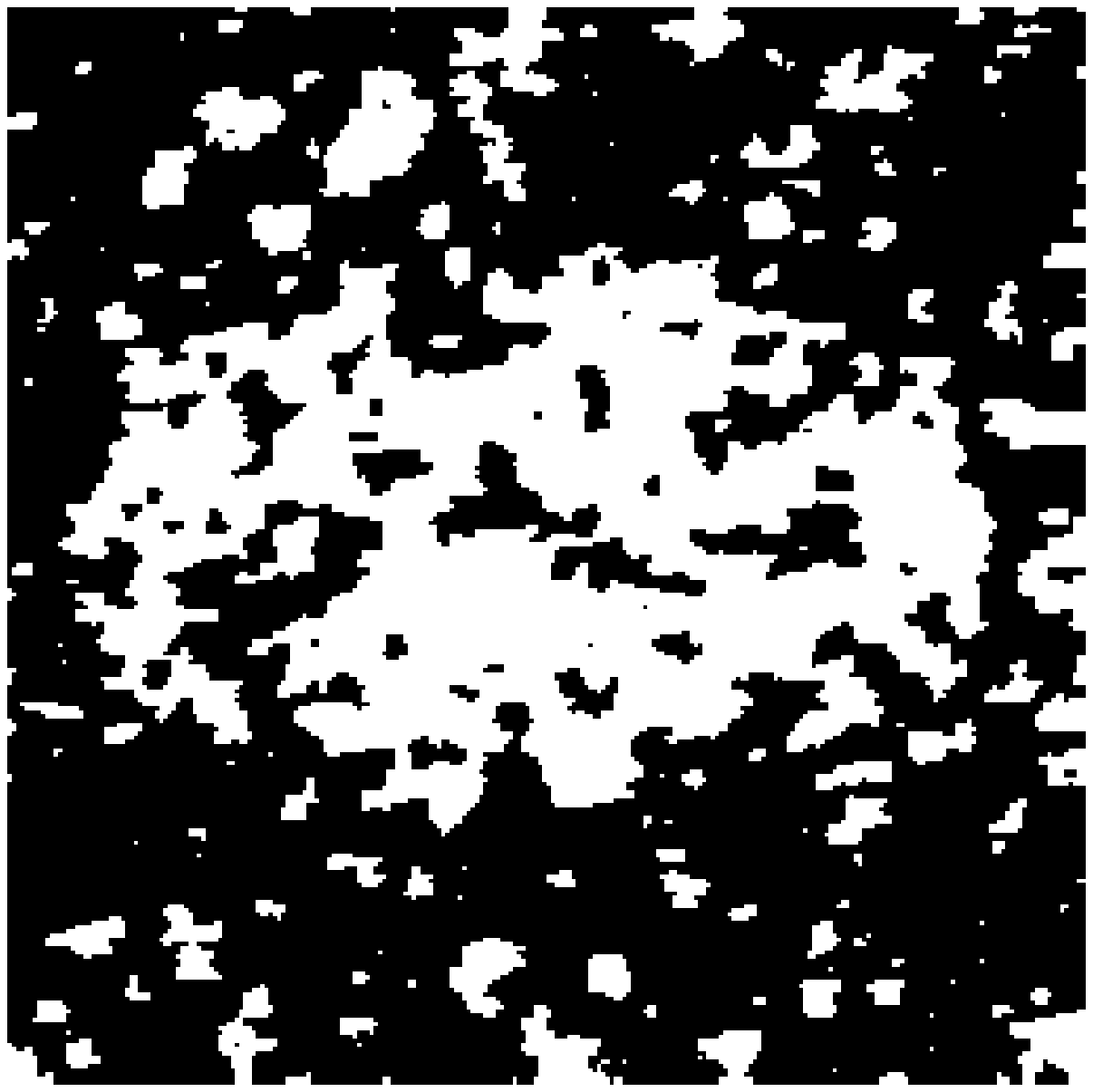}
\end{minipage}
\subcaption{\label{subfig:segh_var_4}Min. $\widehat{R}_{.}(\cdot \lvert \boldsymbol{\mathcal{S}}_{\mathrm{var}})$ `$\boldsymbol{\square}$'}
\end{subfigure}\\

\begin{subfigure}{0.24\linewidth}
\centering
\includegraphics[height = 2.5cm]{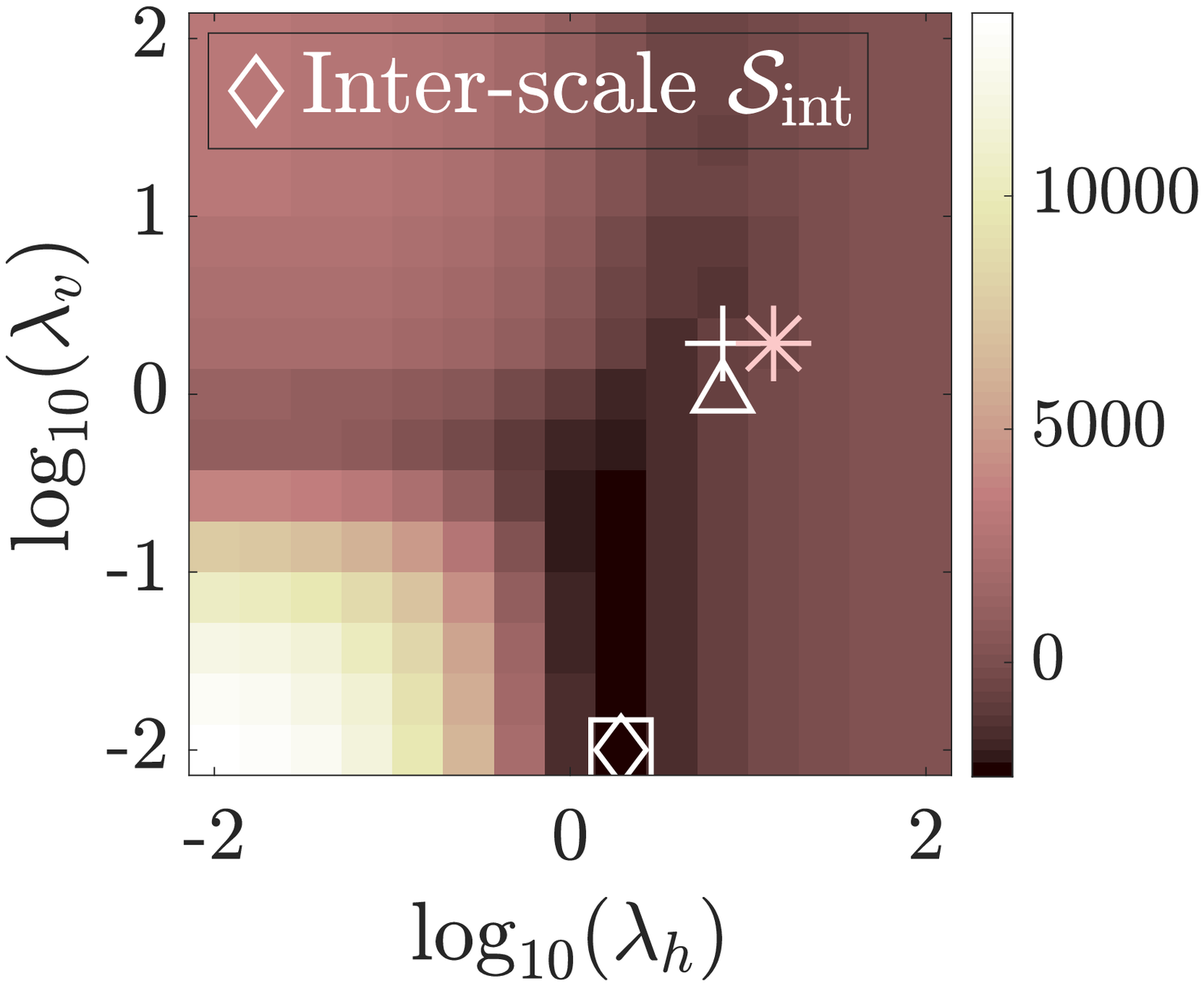}
\subcaption{\label{subfig:fdmc_inter_1} $\widehat{R}_{\nu, \boldsymbol{\varepsilon}}(\boldsymbol{\ell}; \boldsymbol{\Lambda} \lvert \boldsymbol{\mathcal{S}}_{\mathrm{int}})$}
\end{subfigure}
\begin{subfigure}{0.24\linewidth}
\centering
\begin{minipage}[t][2.5cm]{3cm}
\centering
\includegraphics[width = 2cm]{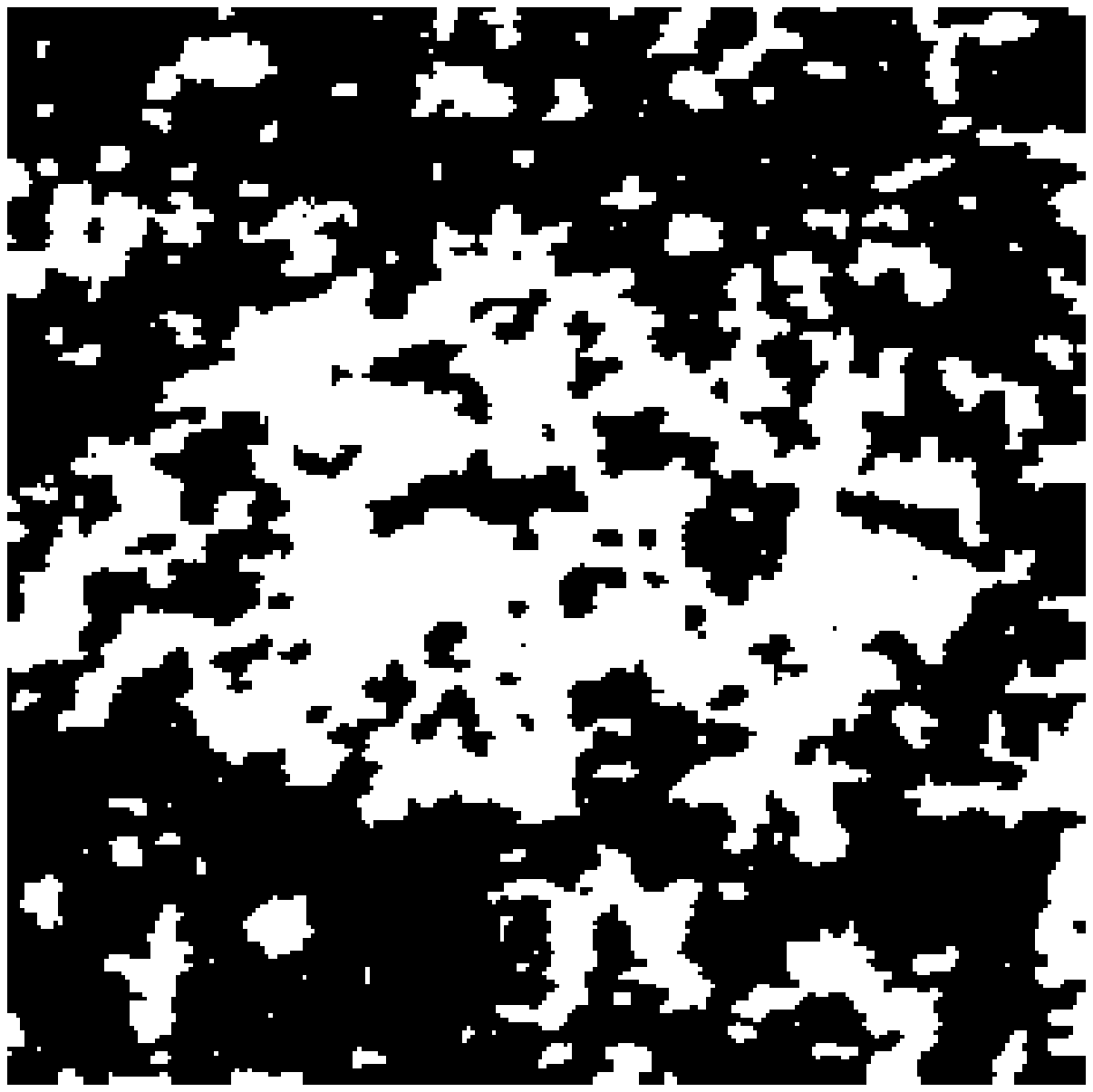}
\end{minipage}
\subcaption{\label{subfig:segh_inter_1}Min. $\widehat{R}_{.}(\cdot \lvert \boldsymbol{\mathcal{S}}_{\mathrm{int}})$ `$\boldsymbol{\diamond}$'}
\end{subfigure}
\begin{subfigure}{0.24\linewidth}
\centering
\includegraphics[height = 2.5cm]{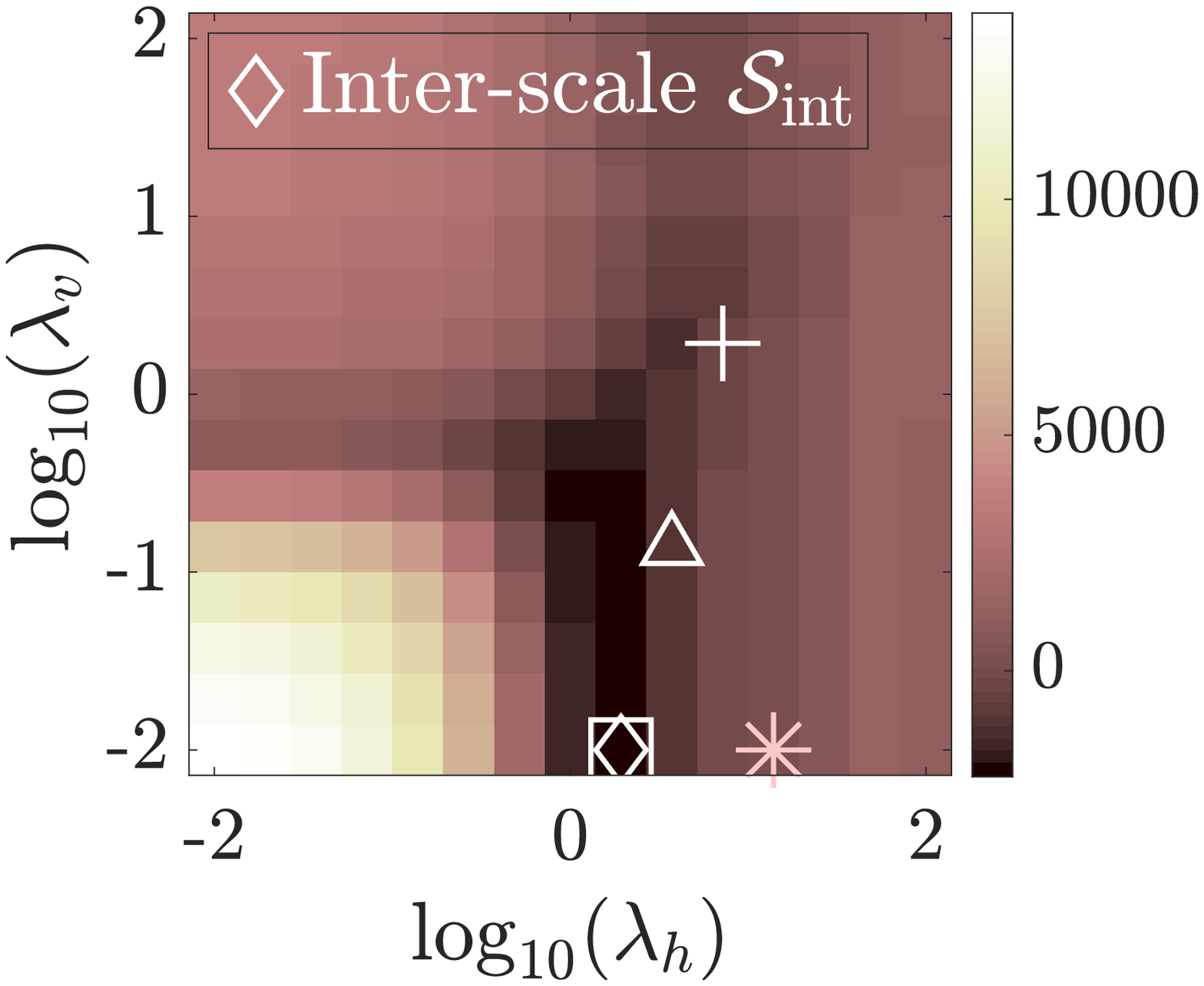}
\subcaption{\label{subfig:fdmc_inter_4}$\widehat{R}_{\nu, \boldsymbol{\varepsilon}}(\boldsymbol{\ell}; \boldsymbol{\Lambda}  \lvert \boldsymbol{\mathcal{S}}_{\mathrm{int}})$}
\end{subfigure}
\begin{subfigure}{0.24\linewidth}
\centering
\begin{minipage}[t][2.5cm]{3cm}
\centering
\includegraphics[width = 2cm]{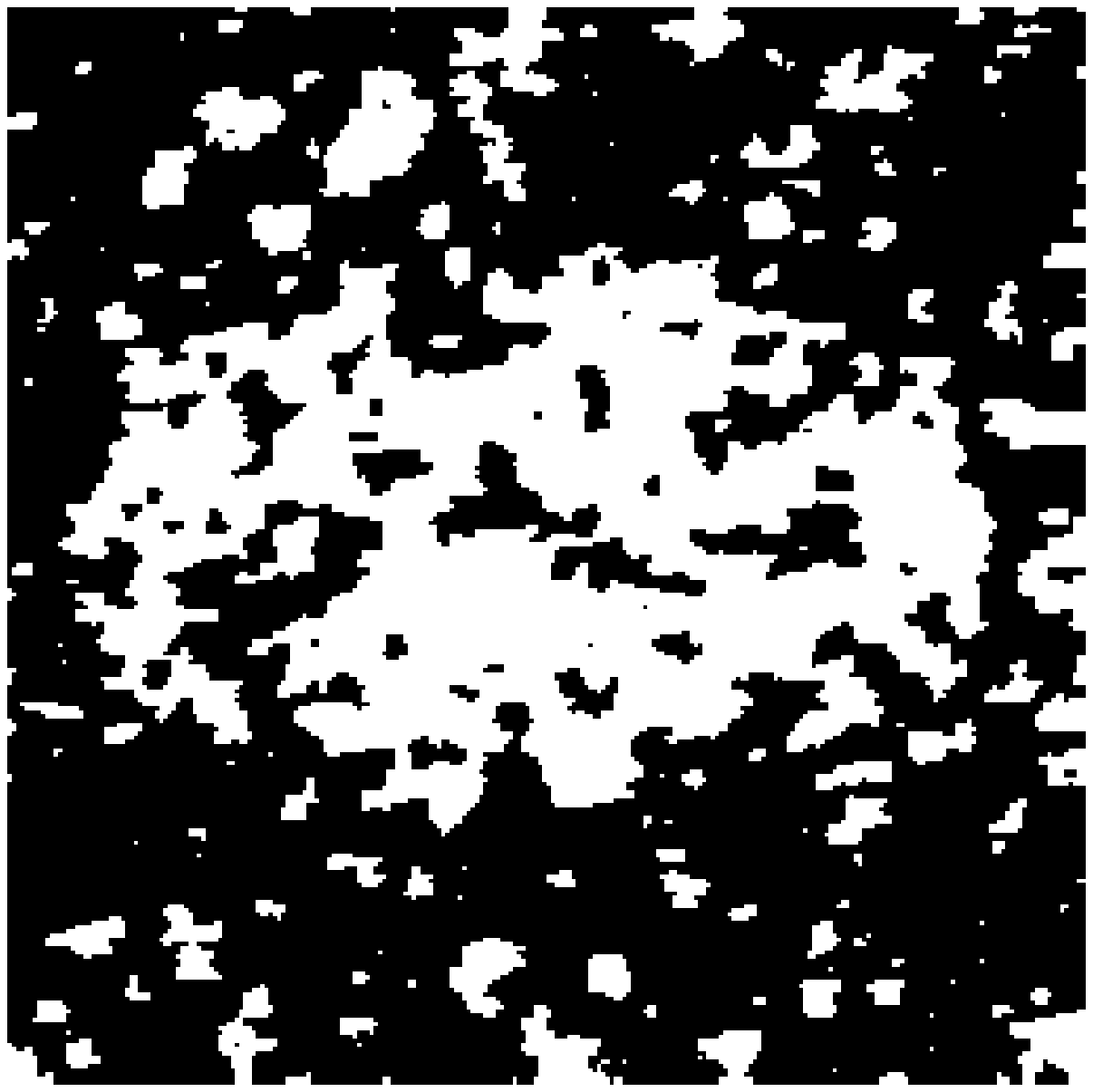}
\end{minipage}
\subcaption{\label{subfig:segh_inter_4}Min. $\widehat{R}_{.}(\cdot \lvert \boldsymbol{\mathcal{S}}_{\mathrm{int}})$ `$\boldsymbol{\diamond}$'}
\end{subfigure}\\

\begin{subfigure}{0.24\linewidth}
\centering
\includegraphics[height = 2.5cm]{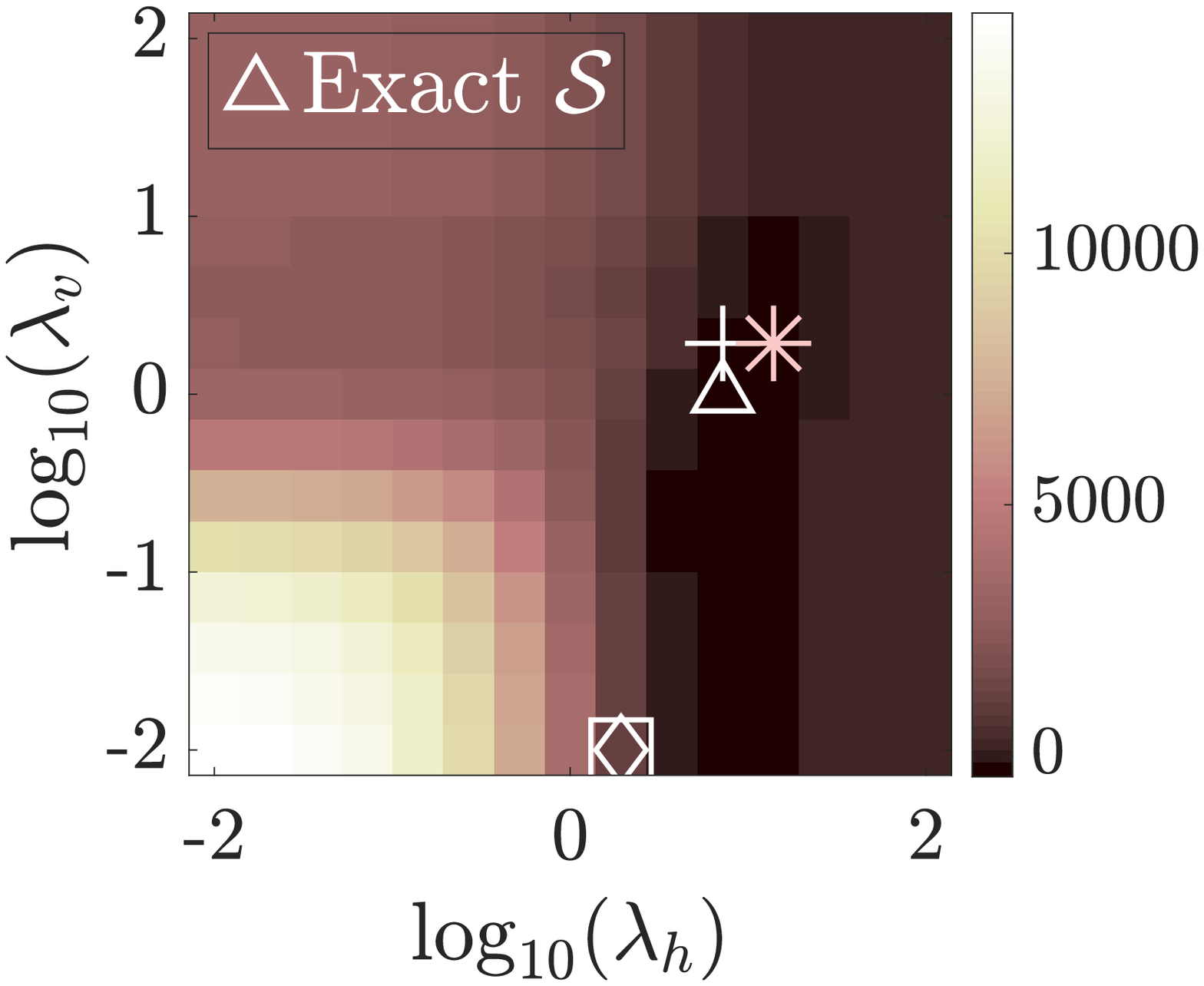}
\subcaption{\label{subfig:fdmc_true_1}$\widehat{R}_{\nu,\boldsymbol{\varepsilon}}(\boldsymbol{\ell};\boldsymbol{\Lambda} \lvert \boldsymbol{\mathcal{S}})$ }
\end{subfigure}
\begin{subfigure}{0.24\linewidth}
\centering
\begin{minipage}[t][2.5cm]{3cm}
\centering
\includegraphics[width = 2cm]{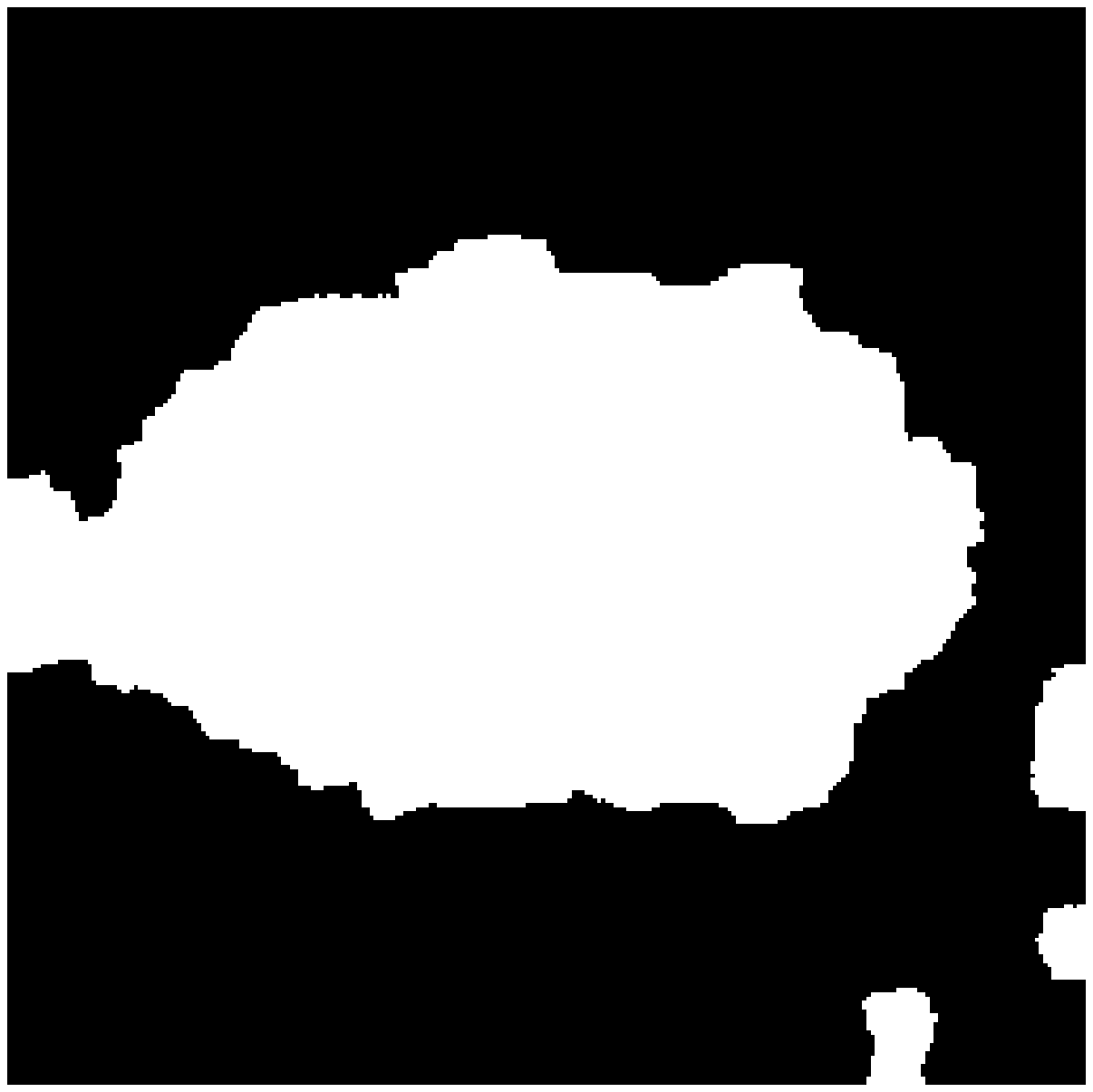}
\end{minipage}
\subcaption{\label{subfig:segh_true_1}Min. $\widehat{R}_{.}(\cdot \lvert \boldsymbol{\mathcal{S}})$ `$\boldsymbol{\bigtriangleup}$'}
\end{subfigure}
\begin{subfigure}{0.24\linewidth}
\centering
\includegraphics[height = 2.5cm]{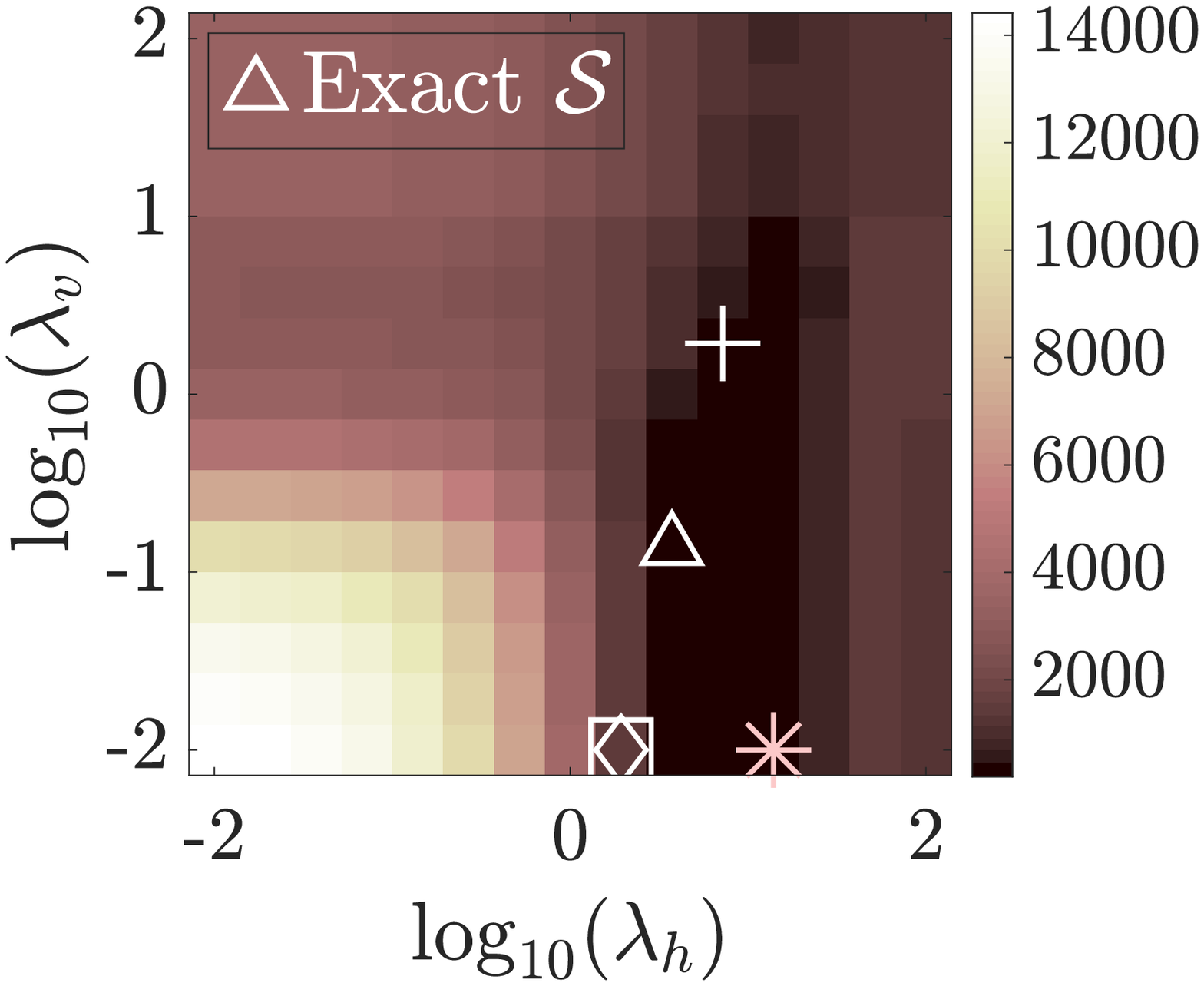}
\subcaption{\label{subfig:fdmc_true_4}$\widehat{R}_{\nu,\boldsymbol{\varepsilon}}(\boldsymbol{\ell};\boldsymbol{\Lambda} \lvert \boldsymbol{\mathcal{S}})$ }
\end{subfigure}
\begin{subfigure}{0.24\linewidth}
\centering
\begin{minipage}[t][2.5cm]{3cm}
\centering
\includegraphics[width = 2cm]{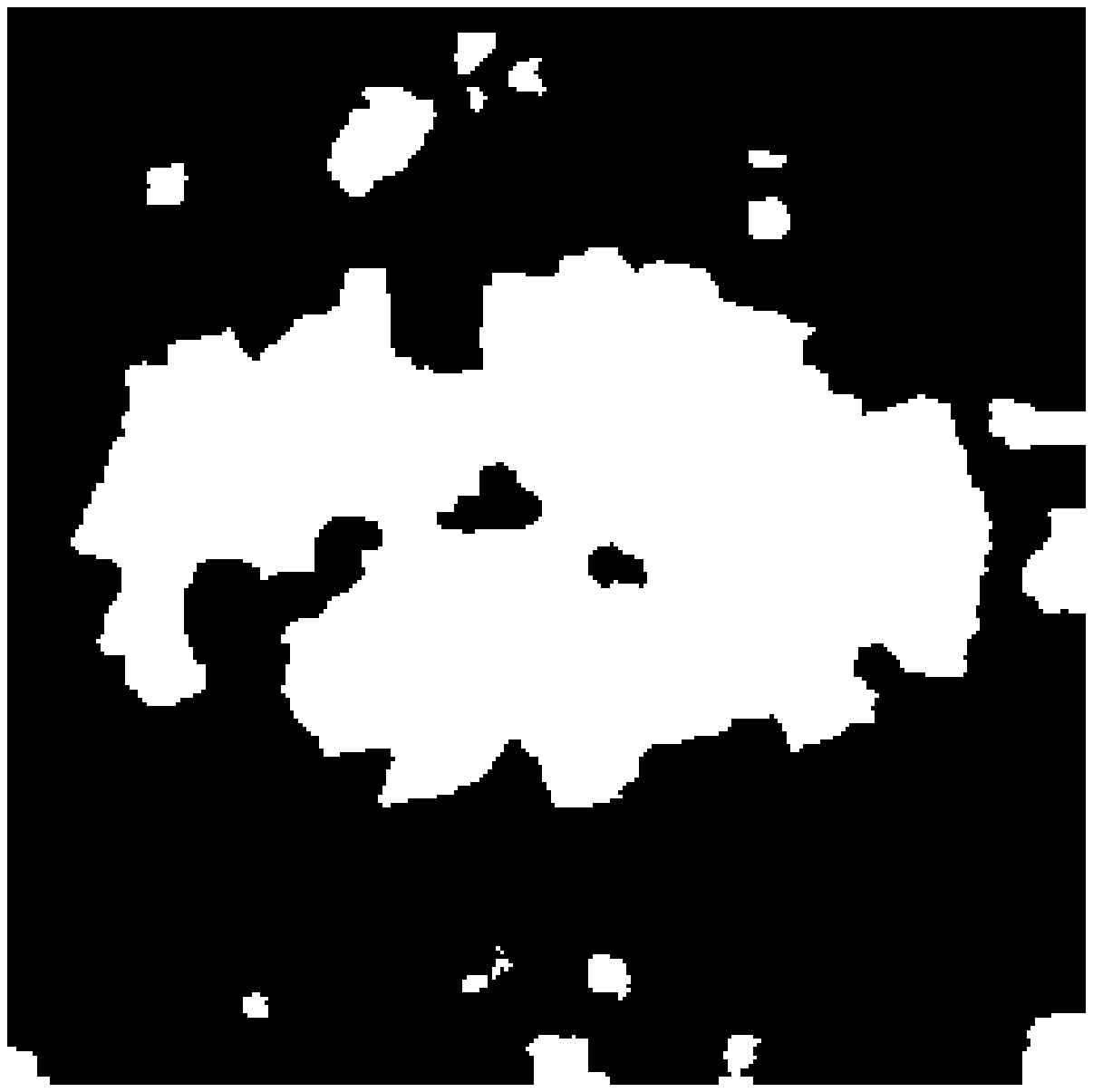}
\end{minipage}
\subcaption{\label{subfig:segh_true_4}Min. $\widehat{R}_{.}(\cdot \lvert \boldsymbol{\mathcal{S}})$ `$\boldsymbol{\bigtriangleup}$'}
\end{subfigure}\\

\begin{subfigure}{0.24\linewidth}
\centering
\includegraphics[height = 2.5cm]{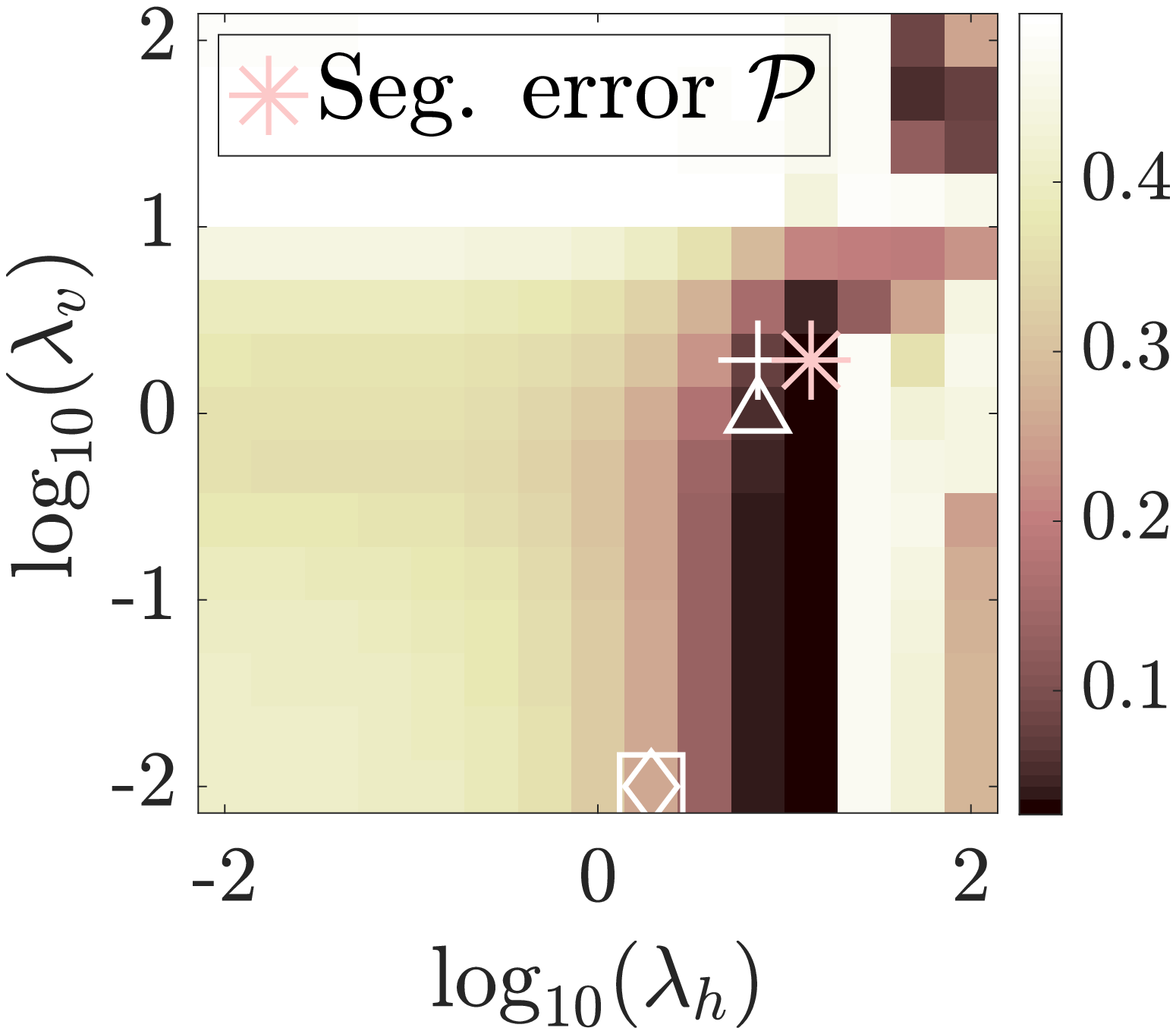}
\subcaption{\label{subfig:seg_1}$\mathcal{P}(\boldsymbol{\ell};\boldsymbol{\Lambda})$}
\end{subfigure}
\begin{subfigure}{0.24\linewidth}
\centering
\begin{minipage}[t][2.5cm]{3cm}
\centering
\includegraphics[width = 2cm]{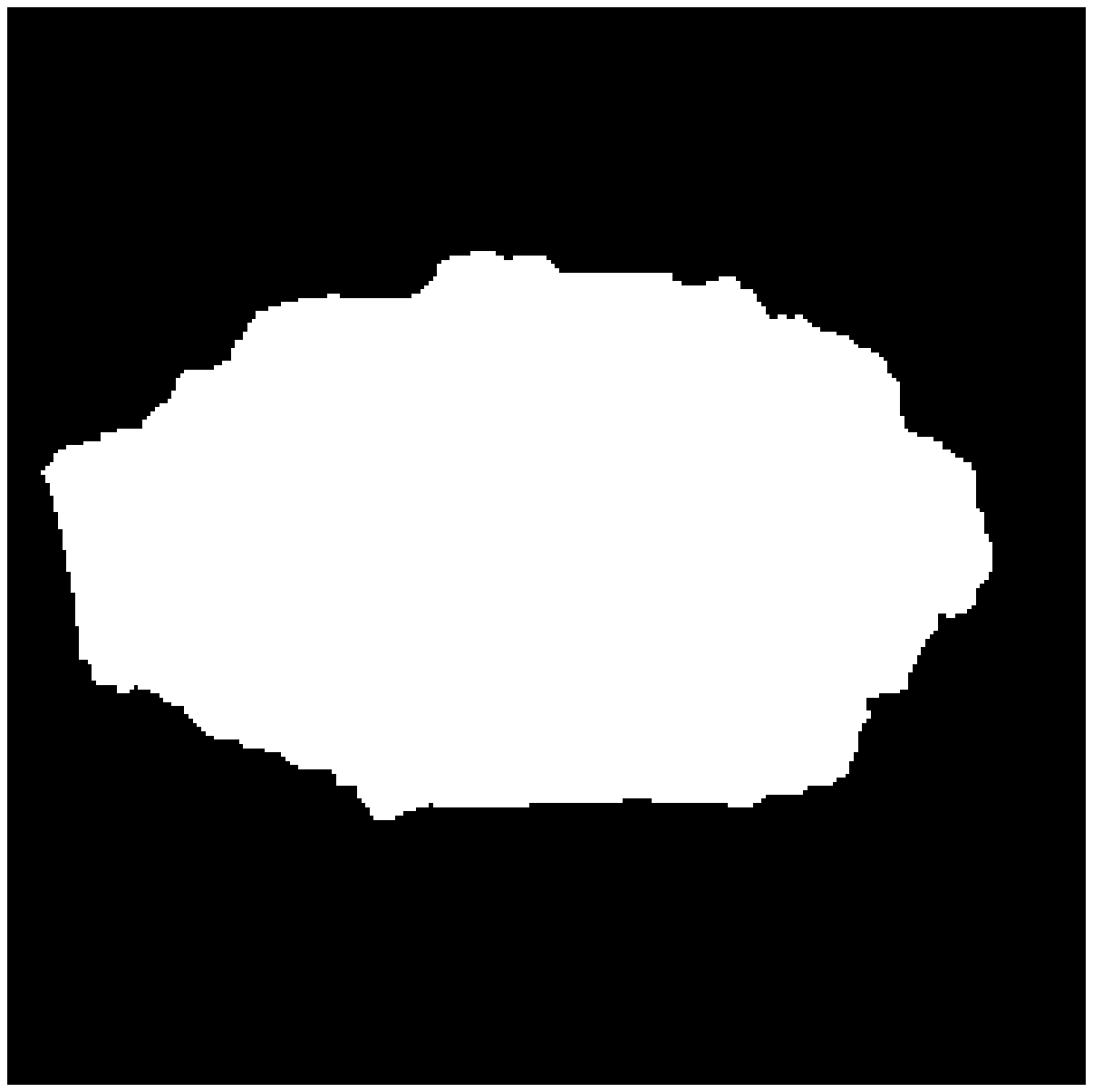}
\end{minipage}
\subcaption{Min. $\mathcal{P}$ `${\color{rose} \Large \boldsymbol{\ast}}$'}
\end{subfigure}
\begin{subfigure}{0.24\linewidth}
\centering
\includegraphics[height = 2.5cm]{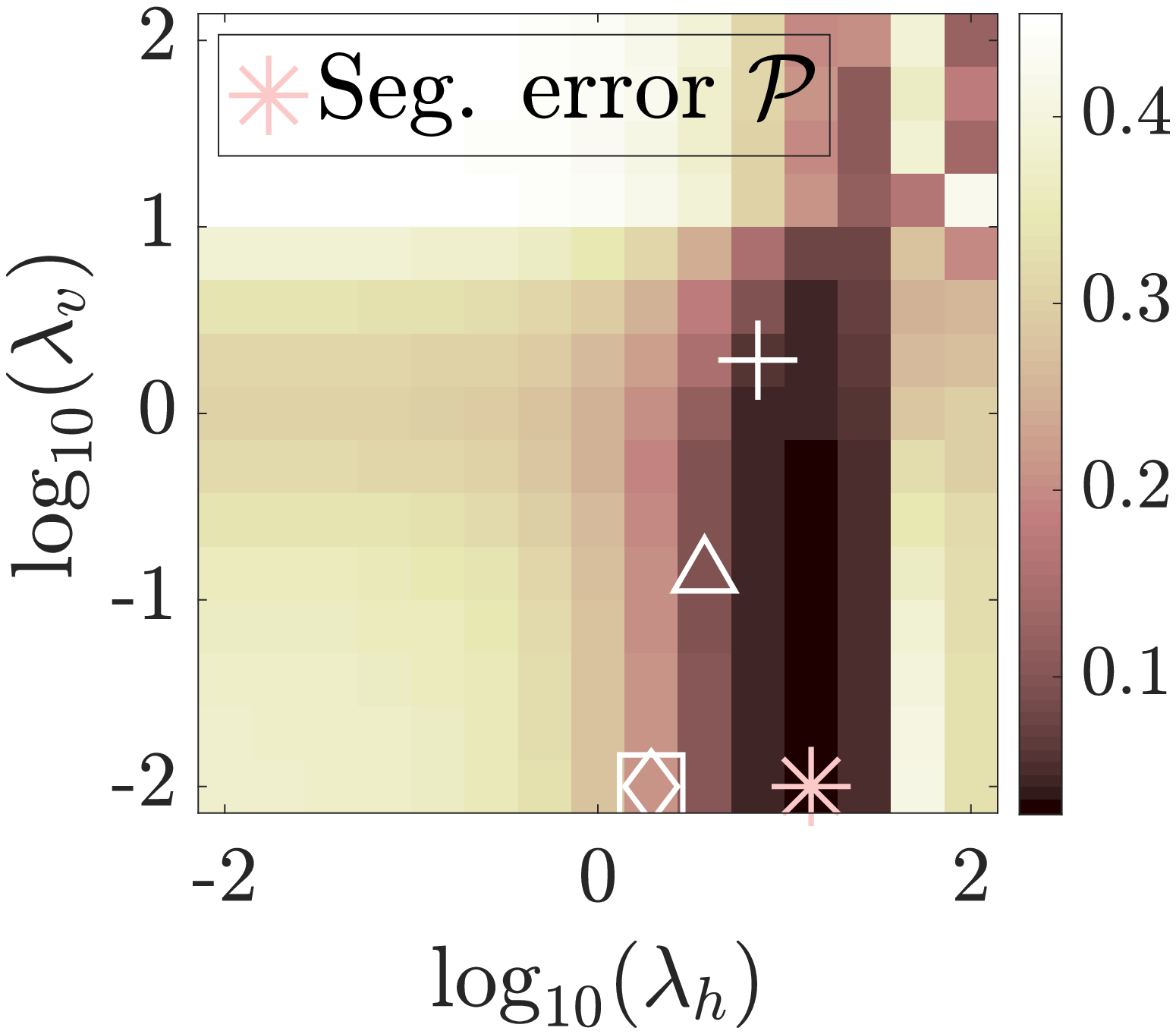}
\subcaption{\label{subfig:seg_4}$\mathcal{P}(\boldsymbol{\ell};\boldsymbol{\Lambda})$}
\end{subfigure}
\begin{subfigure}{0.24\linewidth}
\centering
\begin{minipage}[t][2.5cm]{3cm}
\centering
\includegraphics[width = 2cm]{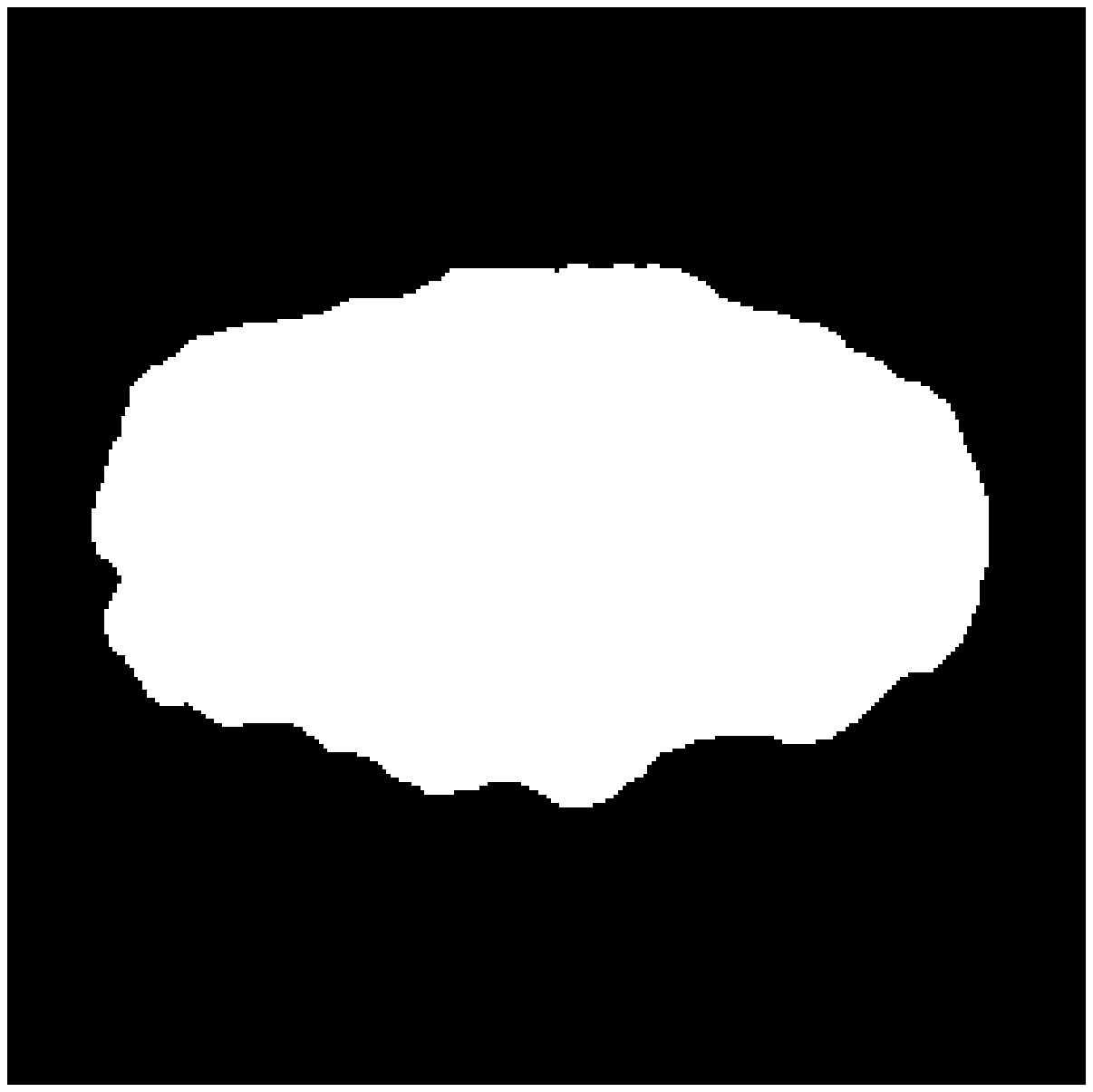}
\end{minipage}
\subcaption{Min. $\mathcal{P}$ `${\color{rose} \Large \boldsymbol{\ast}}$'}
\end{subfigure}

\vspace{-3mm}
\caption{\label{fig:S_DE_13} Error maps for TV-based texture segmentation on a grid of $\boldsymbol{\Lambda} = (\lambda_h,\lambda_v)$, and segmentation obtained with associated optimal hyperparameters for piecewise Textures~``\textbf{D}" (column 1,~2) and ``\textbf{E}" (column 3,~4). 
Estimated risks $\widehat{R}_{\nu \boldsymbol{\varepsilon}}(\boldsymbol{\ell};\boldsymbol{\Lambda} \lvert \boldsymbol{\mathcal{S}})$ computed either with variance matrix $\boldsymbol{\mathcal{S}}_{\mathrm{var}}$~(second row),  inter-scale covariance matrix $\boldsymbol{\mathcal{S}}_{\mathrm{int}}$~(third row), or full  covariance matrix $\boldsymbol{\mathcal{S}}$~(fourth row) are compared.
}
\end{figure}

\subsubsection{Impact of estimating the covariance matrix}
\label{subsec:estim_cov}

In practice, on has access to only the \textit{estimated} covariance matrix $\widehat{\boldsymbol{S}}$. 
This Section compares generalized SURE computed from \textit{estimated} covariance $\widehat{\boldsymbol{\mathcal{S}}}$ to SURE computed assuming the knowledge of \textit{true} covariance $\boldsymbol{\mathcal{S}}$.\\

For Texture~``\textbf{D}", $\widehat{R}_{\nu,\boldsymbol{\varepsilon}}(\boldsymbol{\ell};\boldsymbol{\Lambda} \lvert \widehat{\boldsymbol{\mathcal{S}}})$~(Figure~\ref{subfig:FDMC_LWT_1_fdmc_est_8_JJ_13}) is identical to $\widehat{R}_{\nu,\boldsymbol{\varepsilon}}(\boldsymbol{\ell};\boldsymbol{\Lambda} \lvert \boldsymbol{\mathcal{S}})$~(Figure~\ref{subfig:FDMC_LWT_1_fdmc_true_8_JJ_13}).
Further, optimal hyperparameters $\widehat{\boldsymbol{\Lambda}}^{\dagger}_{\nu, \boldsymbol{\varepsilon}}(\boldsymbol{\ell} \lvert \widehat{\boldsymbol{\mathcal{S}}})$~(`$\color{bclair} \boldsymbol{\bigtriangleup}$') perfectly matches $\widehat{\boldsymbol{\Lambda}}^{\dagger}_{\nu, \boldsymbol{\varepsilon}}(\boldsymbol{\ell} \lvert \boldsymbol{\mathcal{S}})$~(`$\boldsymbol{\bigtriangleup}$') and lead to similar segmentations, $T\widehat{\boldsymbol{h}}(\boldsymbol{\ell};\widehat{\boldsymbol{\Lambda}}^{\dagger}_{\nu, \boldsymbol{\varepsilon}}(\boldsymbol{\ell} \lvert \widehat{\boldsymbol{\mathcal{S}}}))$ (Figure~\ref{subfig:FDMC_LWT_1_fdmc_est_8_JJ_13_segh}) and $T\widehat{\boldsymbol{h}}(\boldsymbol{\ell};\widehat{\boldsymbol{\Lambda}}^{\dagger}_{\nu, \boldsymbol{\varepsilon}}(\boldsymbol{\ell} \lvert \boldsymbol{\mathcal{S}}))$ (Figure~\ref{subfig:FDMC_LWT_1_fdmc_true_8_JJ_13_segh}).
These observations are precisely quantified in Table~\ref{tab:grid_bfgs} in term of values of $\mathcal{R}(\boldsymbol{\ell};\boldsymbol{\Lambda})$ and percentage of misclassified pixels.
The same observations can be made for Texture~``\textbf{E}".\\

Altogether, Figure~\ref{fig:True_Est_DE_13} and the quantitative results reported in Table~\ref{tab:grid_bfgs} show that $\widehat{R}_{\nu,\boldsymbol{\varepsilon}}(\boldsymbol{\ell};\boldsymbol{\Lambda} \lvert \widehat{\boldsymbol{\mathcal{S}}})$ provides an accurate estimate of $\mathcal{R}(\boldsymbol{\ell};\boldsymbol{\Lambda})$, and that $\widehat{\boldsymbol{\Lambda}}^{\dagger}_{\nu, \boldsymbol{\varepsilon}}(\boldsymbol{\ell} \lvert \widehat{\boldsymbol{\mathcal{S}}})$ is a good estimate of $\boldsymbol{\Lambda}_{\mathcal{R}}$.

\begin{table}[h!]
\centering
\begin{tabular}{ccccc}
\toprule
 & \multicolumn{2}{c}{Texture~``\textbf{D}"}  & \multicolumn{2}{c}{Texture~``\textbf{E}"} \\ 
 \midrule
Hyperparameter $\boldsymbol{\Lambda}$ & $\mathcal{R}(\boldsymbol{\ell};\boldsymbol{\Lambda})$ & $ \mathcal{P}(\boldsymbol{\ell};\boldsymbol{\Lambda})$ & $\mathcal{R}(\boldsymbol{\ell};\boldsymbol{\Lambda})$ & $ \mathcal{P}(\boldsymbol{\ell};\boldsymbol{\Lambda})$ \\
 \midrule
$\boldsymbol{\Lambda}_{\mathcal{R}}$ `+' & $2.32 \, 10^3$ & $7.79\%$ & $2.66 \, 10^3$ & $5.34\%$\\
$\widehat{\boldsymbol{\Lambda}}_{\nu, \boldsymbol{\varepsilon}}^\dagger(\boldsymbol{\ell} \lvert \boldsymbol{\mathcal{S}})$ `$\boldsymbol{\bigtriangleup}$' & $2.35 \, 10^3$ & $5.51\%$& $2.83 \, 10^3$ & $9.58\%$\\
$\widehat{\boldsymbol{\Lambda}}_{\nu, \boldsymbol{\varepsilon}}^\dagger (\boldsymbol{\ell} \lvert \widehat{\boldsymbol{\mathcal{S}}})$ `${\color{bclair}\boldsymbol{\bigtriangleup}}$'  & $2.35 \, 10^3$ & $5.51\%$& $2.96 \, 10^3$ & $4.61\%$\\
$\widehat{\boldsymbol{\Lambda}}_{\nu, \boldsymbol{\varepsilon}}^\mathrm{BFGS} (\boldsymbol{\ell} \lvert \boldsymbol{\mathcal{S}})$ `$\boldsymbol{\bigtriangledown}$'  & $2.36 \, 10^3$ & $4.66\%$& $2.83 \, 10^3$ & $3.71\%$\\
$\widehat{\boldsymbol{\Lambda}}_{\nu, \boldsymbol{\varepsilon}}^\mathrm{BFGS}(\boldsymbol{\ell} \lvert \widehat{\boldsymbol{\mathcal{S}}})$ `${\color{bclair}\boldsymbol{\bigtriangledown}}$' & $2.36 \, 10^3$ & $6.22\%$& $2.83 \, 10^3$ & $3.27\%$\\
\bottomrule
\end{tabular}
\caption{\label{tab:grid_bfgs}Grid search v.s. BFGS Algorithm~\ref{alg:BFGS} performance in term of quadratic error $\mathcal{R}(\boldsymbol{\ell};\boldsymbol{\Lambda})$ and segmentation error $ \mathcal{P}(\boldsymbol{\ell};\boldsymbol{\Lambda})$ for the two different Textures~``\textbf{D}"~and~``\textbf{E}".}
\end{table}

\begin{figure}[h!]
\centering
\begin{subfigure}{0.49\linewidth}
\centering
Texture ``\textbf{D}"
\end{subfigure}
\begin{subfigure}{0.49\linewidth}
\centering
Texture ``\textbf{E}"
\end{subfigure}

\vspace{5mm}

\begin{subfigure}{0.24\linewidth}
\centering
\includegraphics[height = 2.5cm]{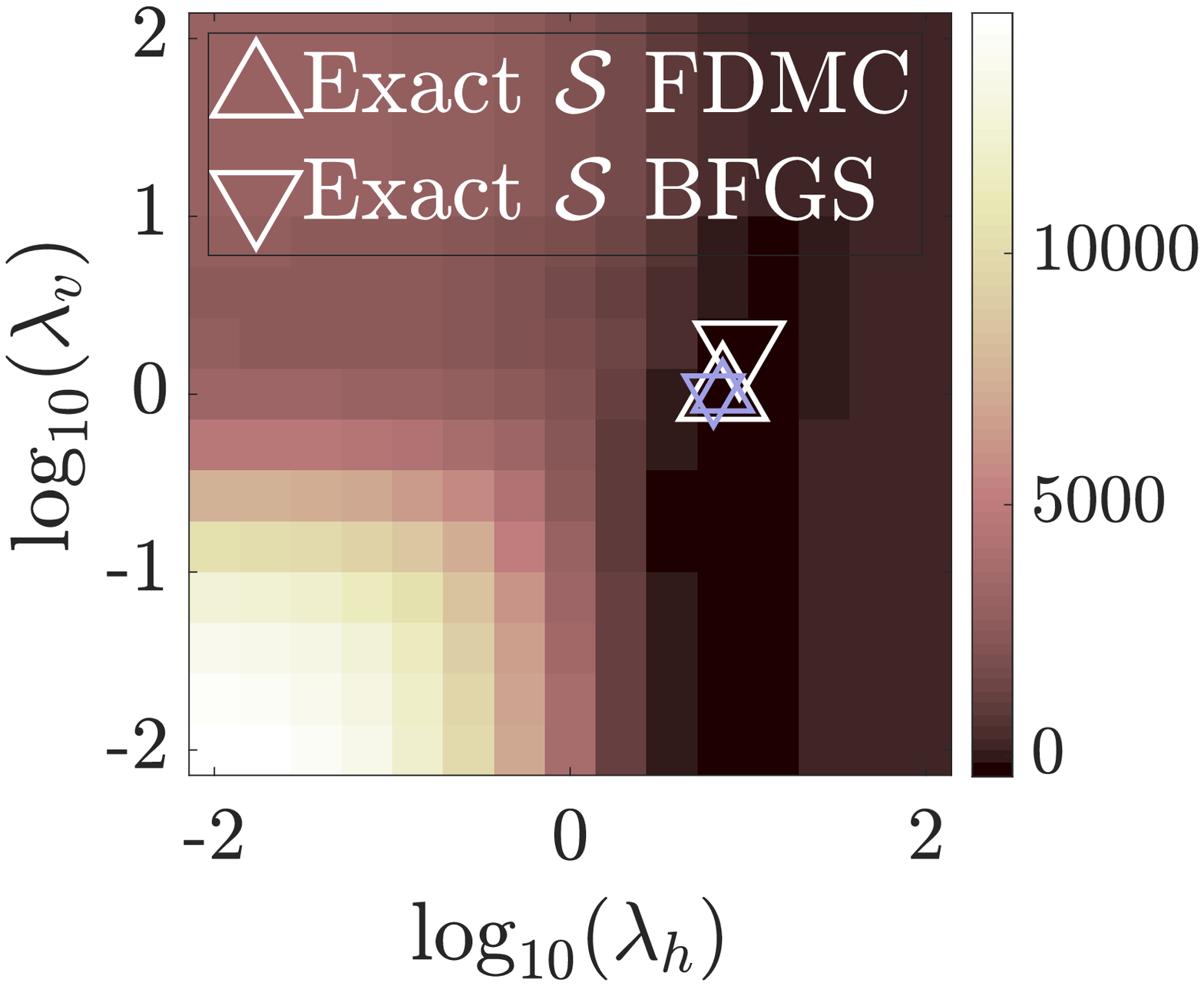}
\subcaption{\label{subfig:FDMC_LWT_1_fdmc_true_8_JJ_13}$\widehat{R}_{\nu,\boldsymbol{\varepsilon}}(\boldsymbol{\ell};\boldsymbol{\Lambda} \lvert \boldsymbol{\mathcal{S}})$ }
\end{subfigure}
\begin{subfigure}{0.24\linewidth}
\centering
\includegraphics[height = 2.5cm]{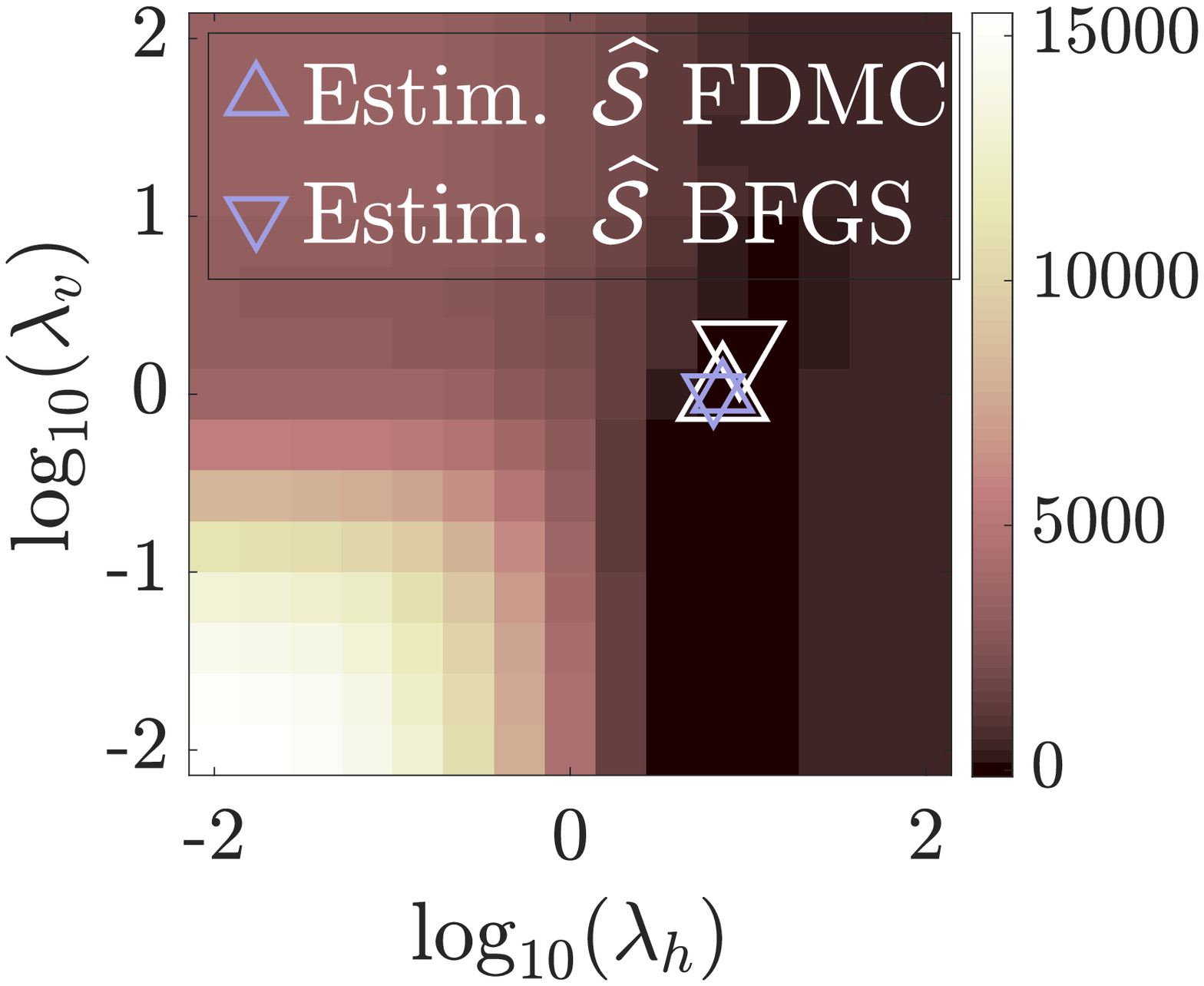}
\subcaption{\label{subfig:FDMC_LWT_1_fdmc_est_8_JJ_13}$\widehat{R}_{\nu,\boldsymbol{\varepsilon}}(\boldsymbol{\ell};\boldsymbol{\Lambda} \lvert \widehat{\boldsymbol{\mathcal{S}}})$ }
\end{subfigure}
\begin{subfigure}{0.24\linewidth}
\centering
\includegraphics[height = 2.5cm]{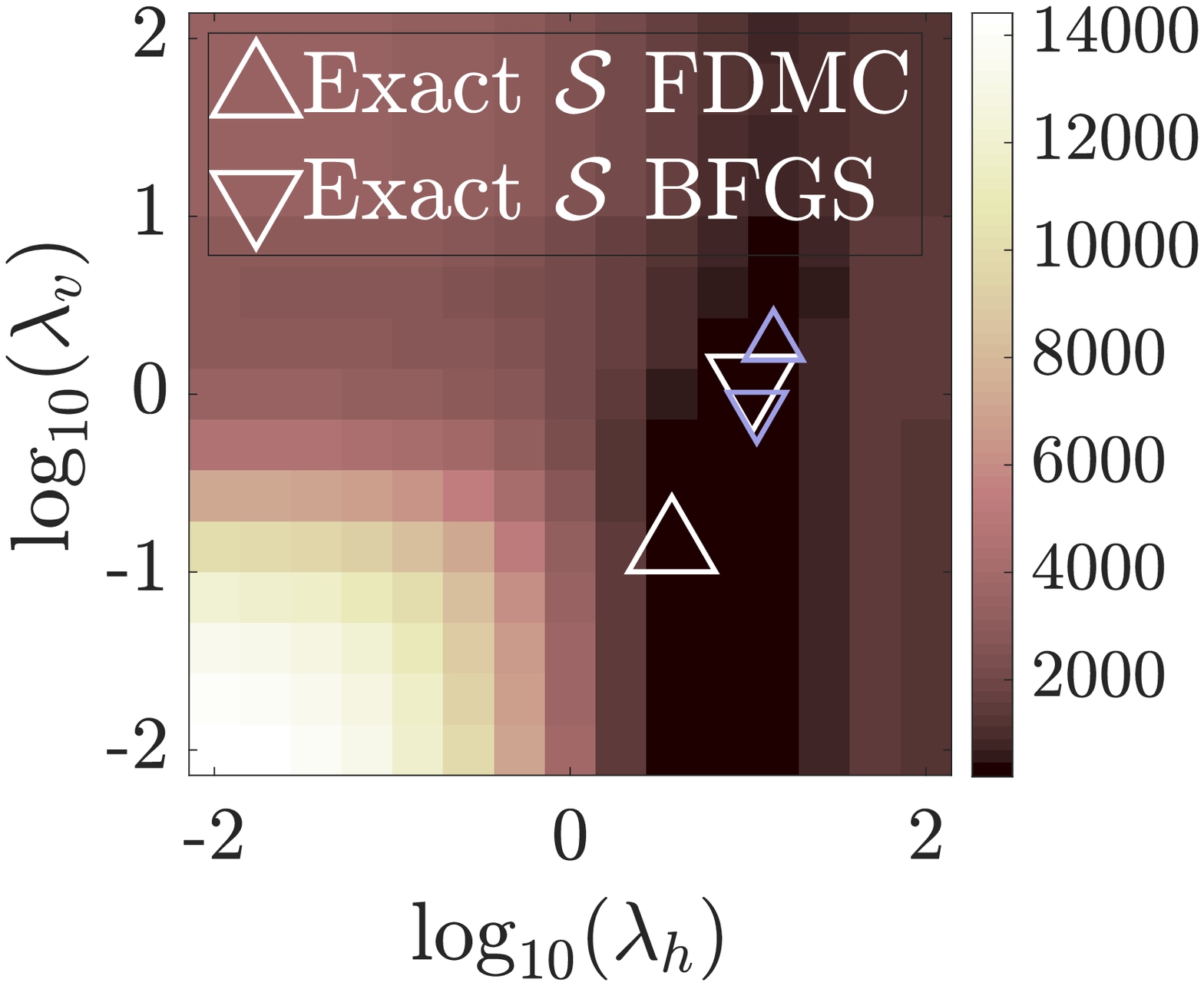}
\subcaption{\label{subfig:FDMC_LWT_4_fdmc_true_8_JJ_13}$\widehat{R}_{\nu,\boldsymbol{\varepsilon}}(\boldsymbol{\ell};\boldsymbol{\Lambda} \lvert \boldsymbol{\mathcal{S}})$ }
\end{subfigure}
\begin{subfigure}{0.24\linewidth}
\centering
\includegraphics[height = 2.5cm]{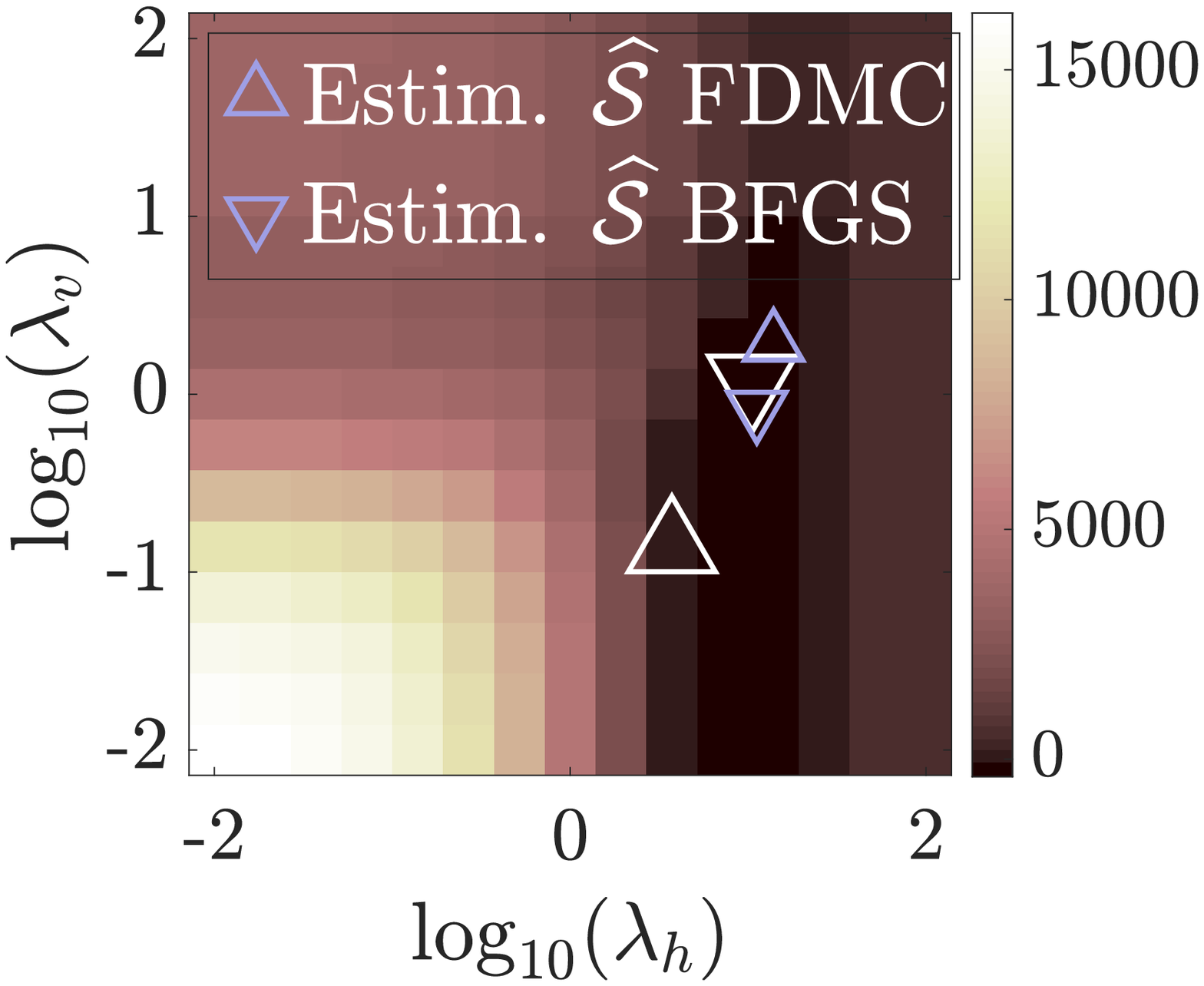}
\subcaption{\label{subfig:FDMC_LWT_4_fdmc_est_8_JJ_13}$\widehat{R}_{\nu,\boldsymbol{\varepsilon}}(\boldsymbol{\ell};\boldsymbol{\Lambda} \lvert \widehat{\boldsymbol{\mathcal{S}}})$ }
\end{subfigure}

\begin{subfigure}{0.24\linewidth}
\centering
\includegraphics[width = 2cm]{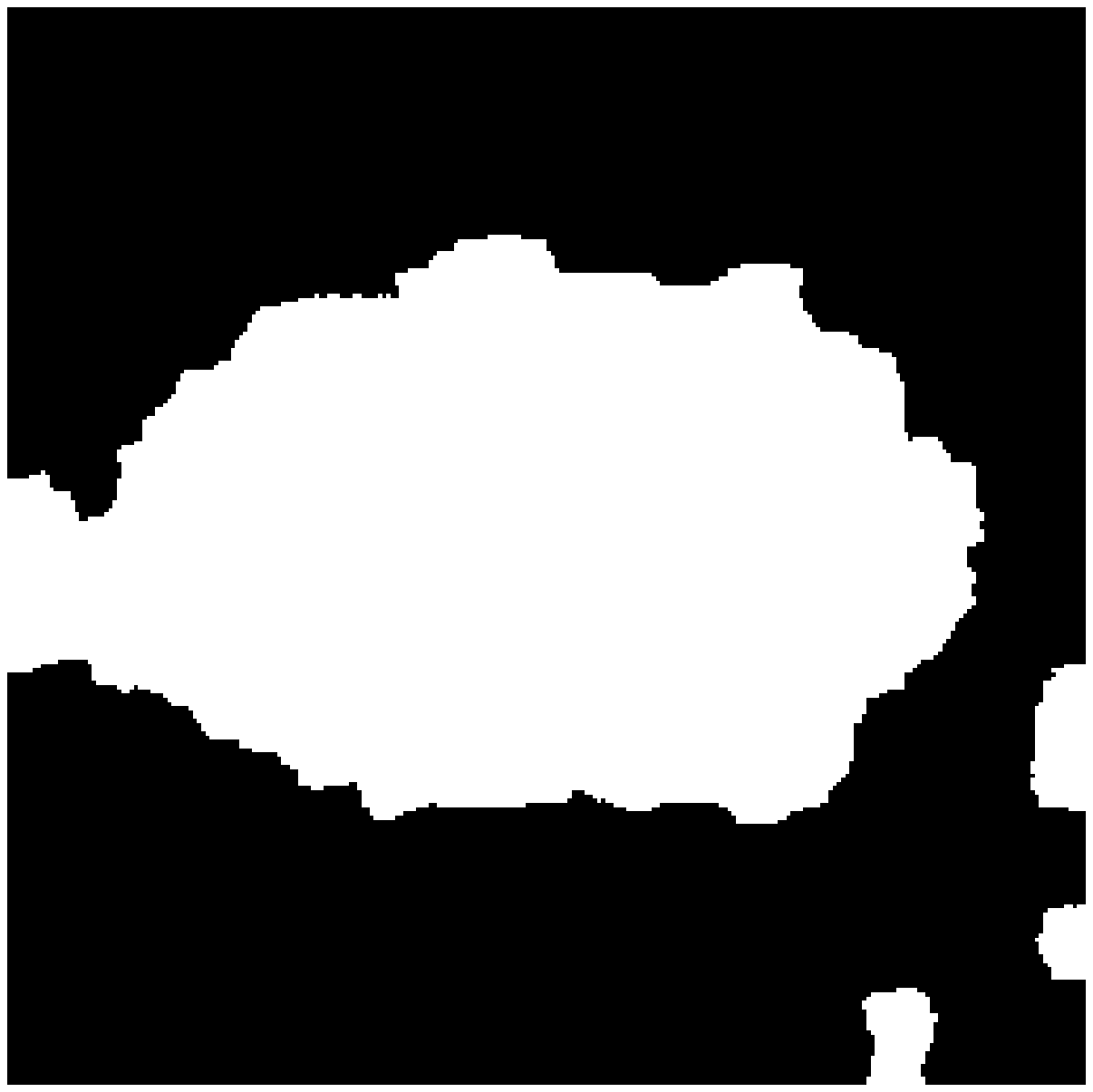}\\
\vspace{3mm}
\subcaption{\label{subfig:FDMC_LWT_1_fdmc_true_8_JJ_13_segh}Min. $\widehat{R}_{.}(\cdot \lvert \boldsymbol{\mathcal{S}})$ `$\boldsymbol{\bigtriangleup}$'}
\end{subfigure}
\begin{subfigure}{0.24\linewidth}
\centering
\includegraphics[width = 2cm]{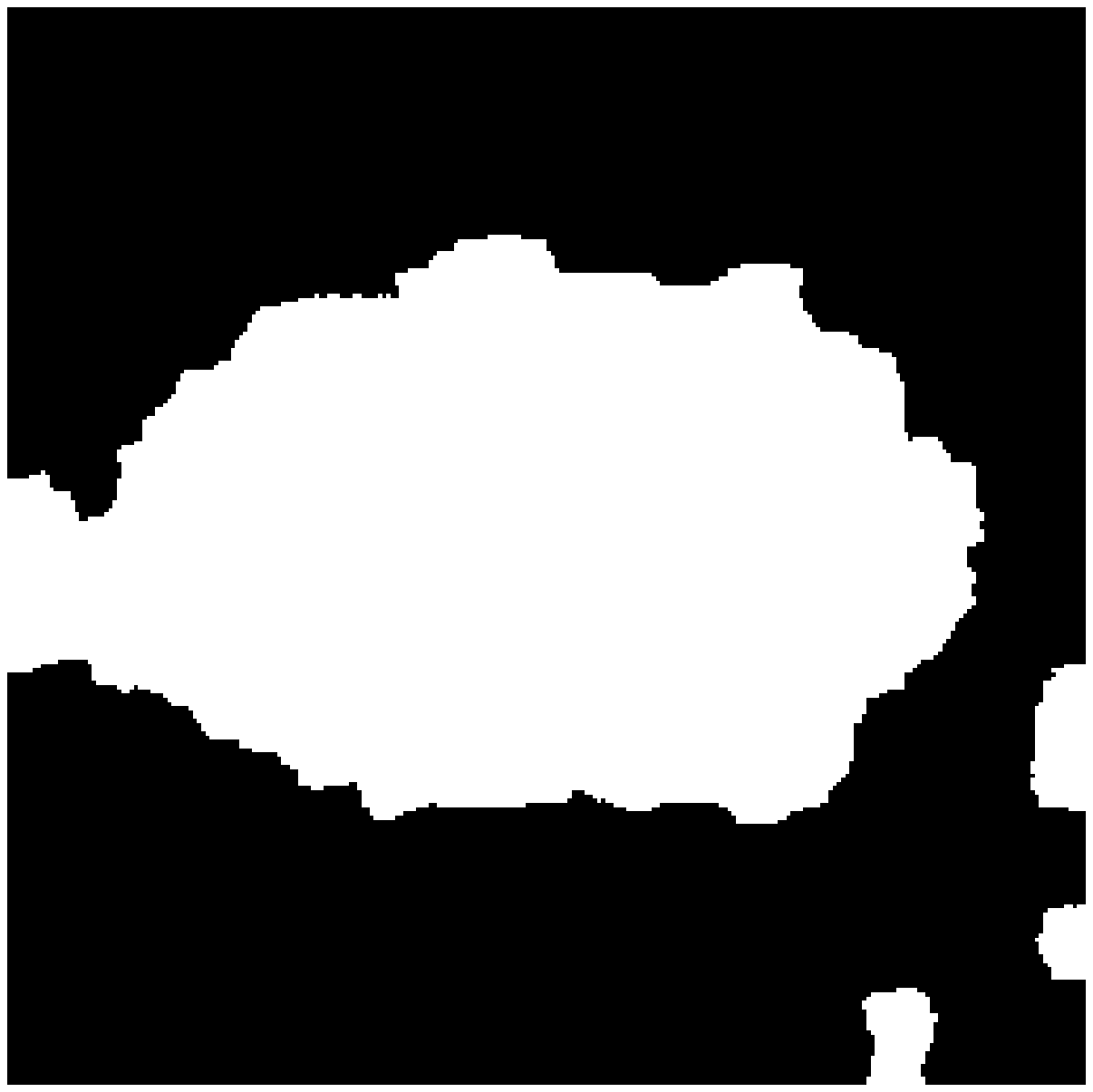}\\
\vspace{3mm}
\subcaption{\label{subfig:FDMC_LWT_1_fdmc_est_8_JJ_13_segh}Min. $\widehat{R}_{.}(\cdot \lvert \widehat{\boldsymbol{\mathcal{S}}})$ `$\color{bclair} \boldsymbol{\bigtriangleup}$'}
\end{subfigure}
\begin{subfigure}{0.24\linewidth}
\centering
\includegraphics[width = 2cm]{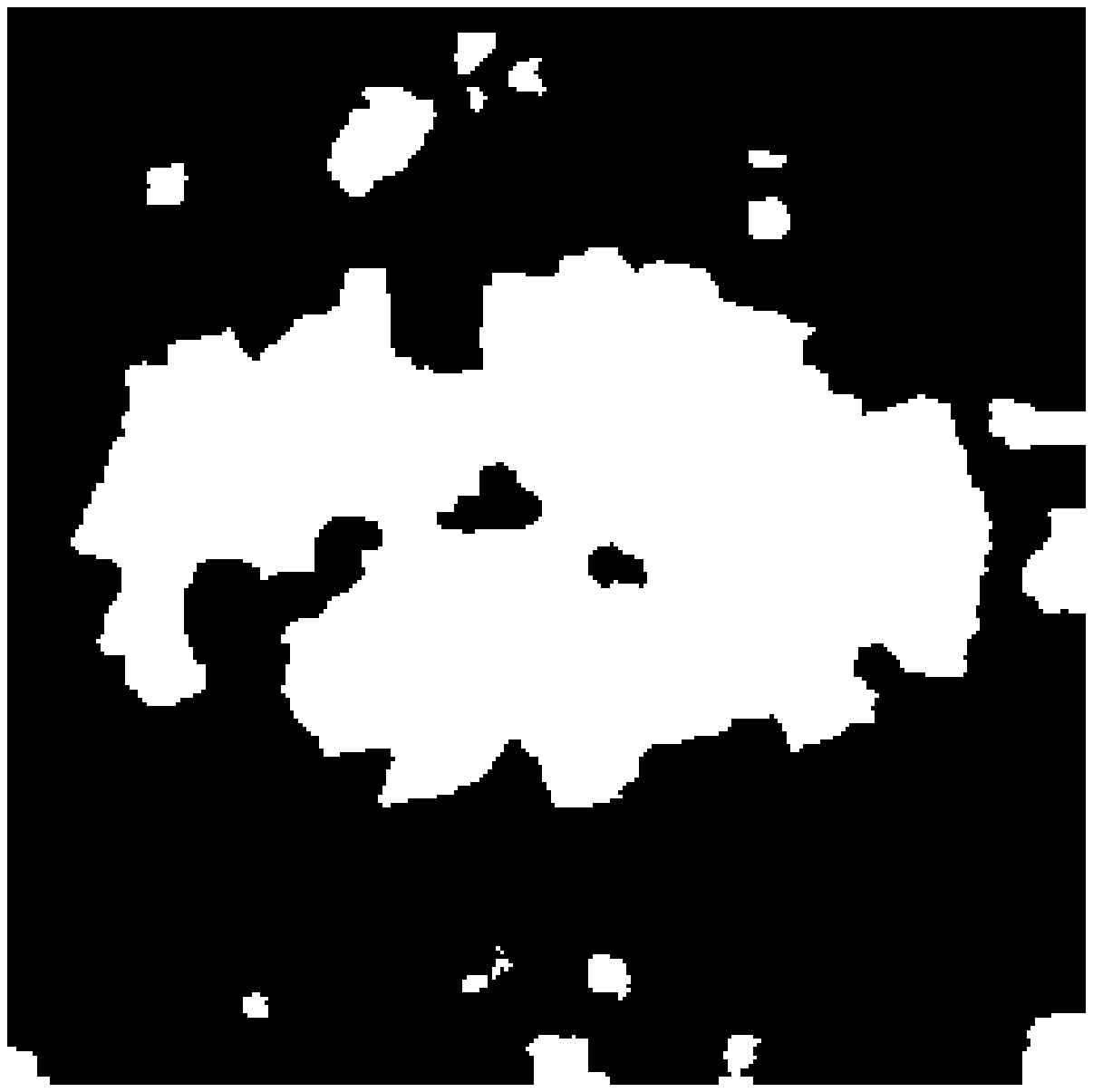}\\
\vspace{3mm}
\subcaption{\label{subfig:FDMC_LWT_4_fdmc_true_8_JJ_13_segh}Min. $\widehat{R}_{.}(\cdot \lvert \boldsymbol{\mathcal{S}})$ `$\boldsymbol{\bigtriangleup}$'}
\end{subfigure}
\begin{subfigure}{0.24\linewidth}
\centering
\includegraphics[width = 2cm]{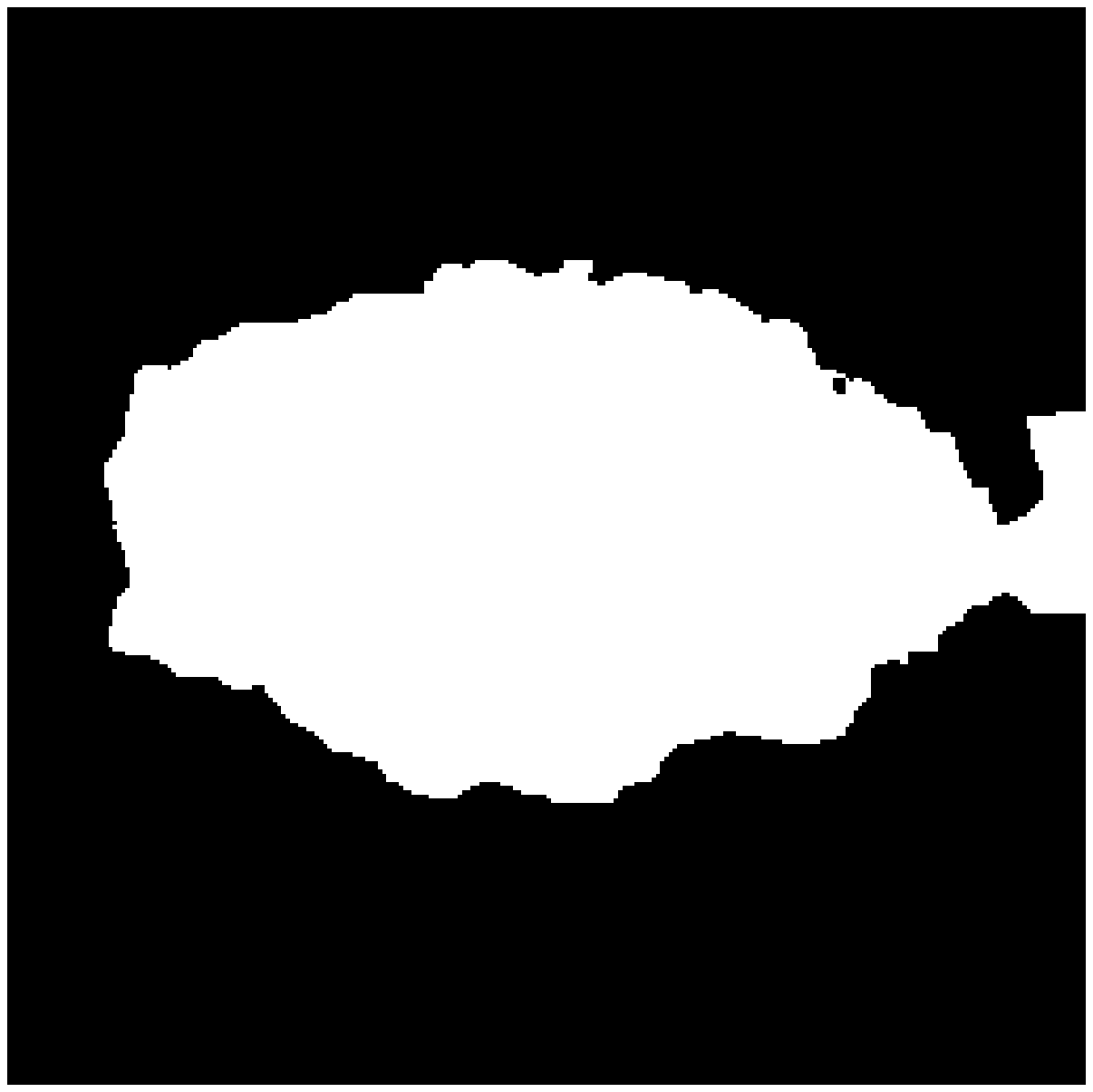}\\
\vspace{3mm}
\subcaption{\label{subfig:FDMC_LWT_4_fdmc_est_8_JJ_13_segh}Min. $\widehat{R}_{.}(\cdot \lvert \widehat{\boldsymbol{\mathcal{S}}})$ `$\color{bclair} \boldsymbol{\bigtriangleup}$'}
\end{subfigure}

\begin{subfigure}{0.24\linewidth}
\centering
\includegraphics[width = 2cm]{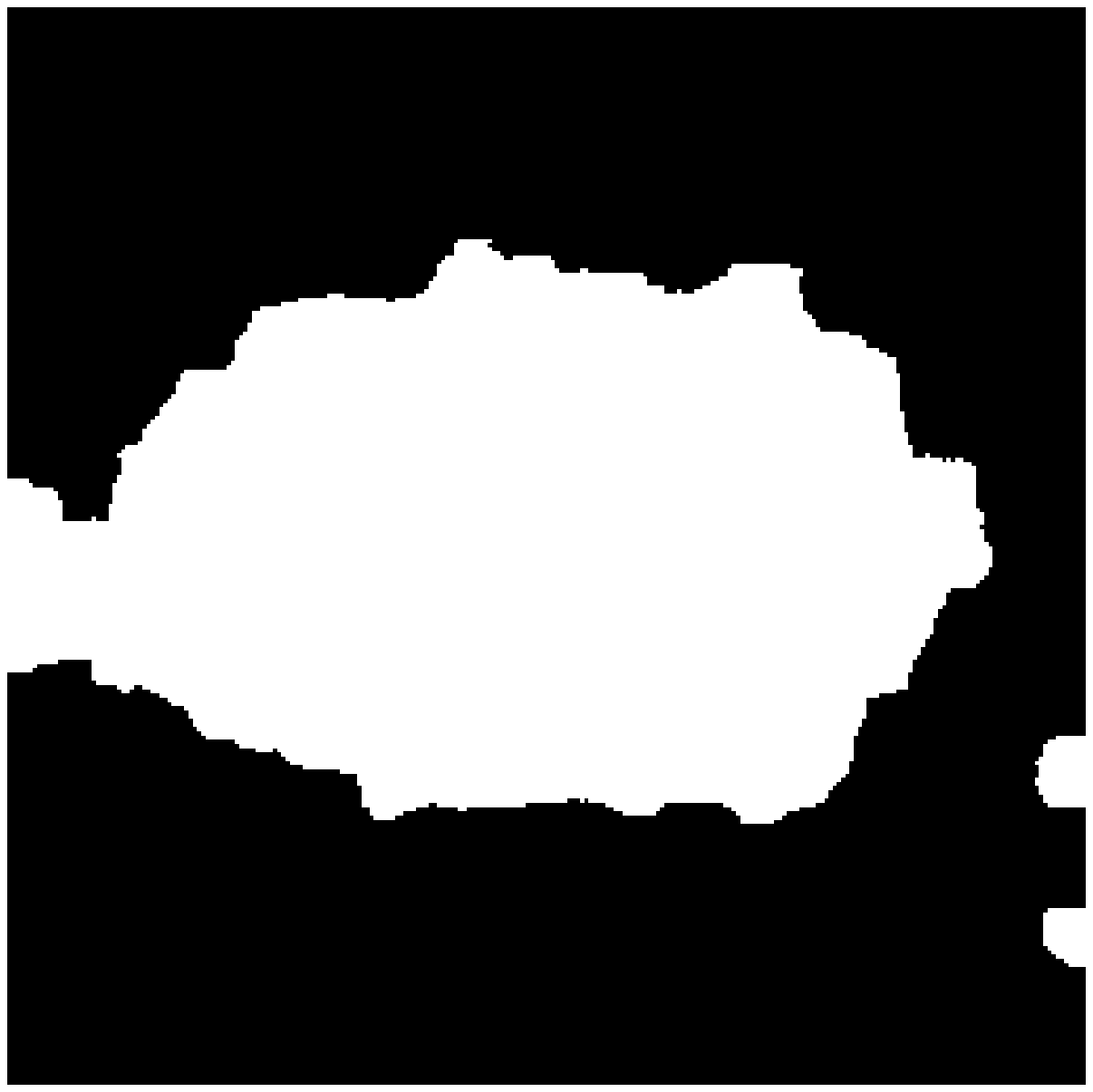}\\
\vspace{3mm}
\subcaption{\label{subfig:FDMC_LWT_1_bfgs_true_8_JJ_13_segh}Auto. selec. `$\boldsymbol{\bigtriangledown}$'}
\end{subfigure}
\begin{subfigure}{0.24\linewidth}
\centering
\includegraphics[width = 2cm]{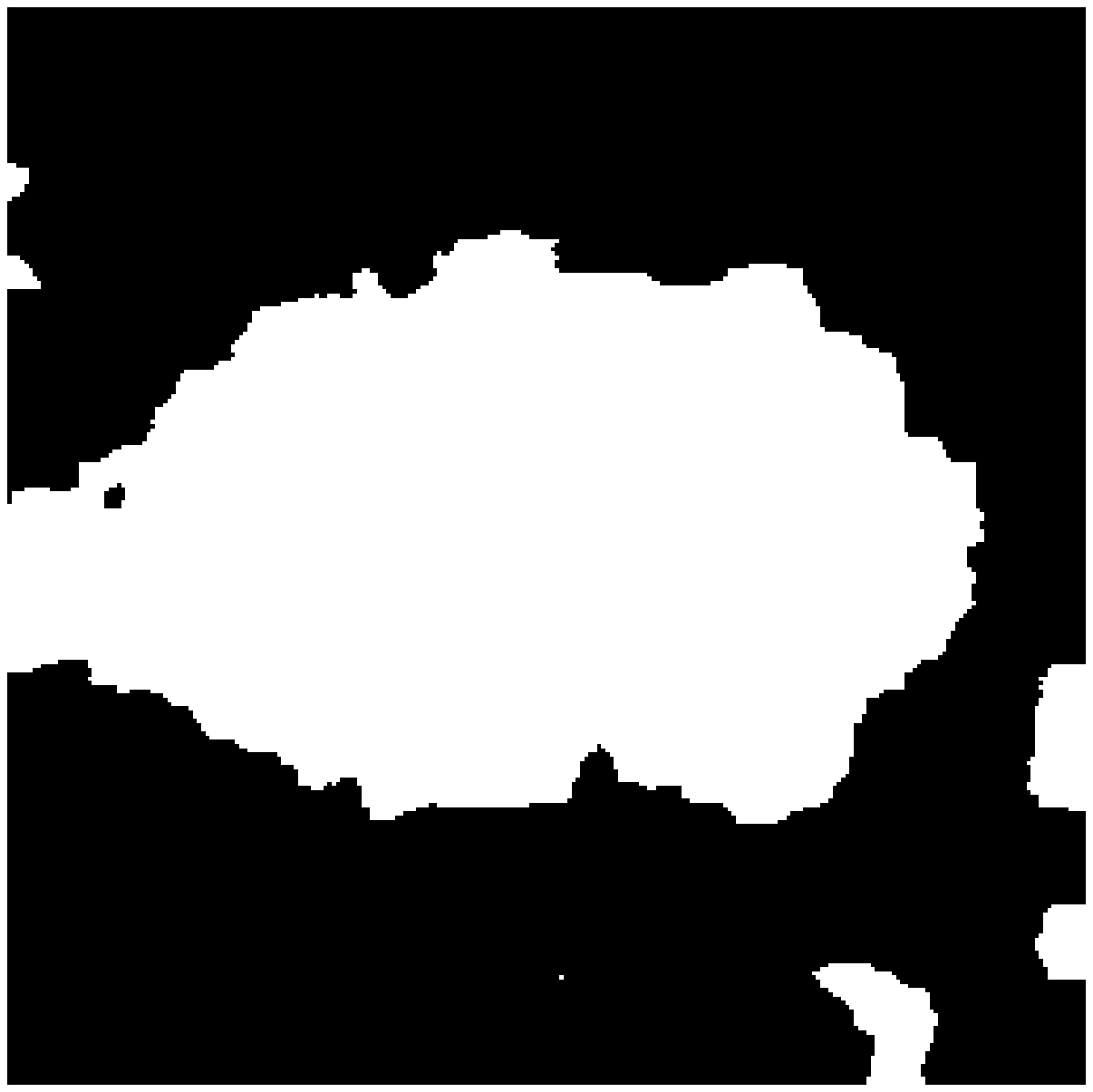}\\
\vspace{3mm}
\subcaption{\label{subfig:FDMC_LWT_1_bfgs_8_JJ_13_segh}Auto. selec. `$\color{bclair} \boldsymbol{\bigtriangledown}$'}
\end{subfigure}
\begin{subfigure}{0.24\linewidth}
\centering
\includegraphics[width = 2cm]{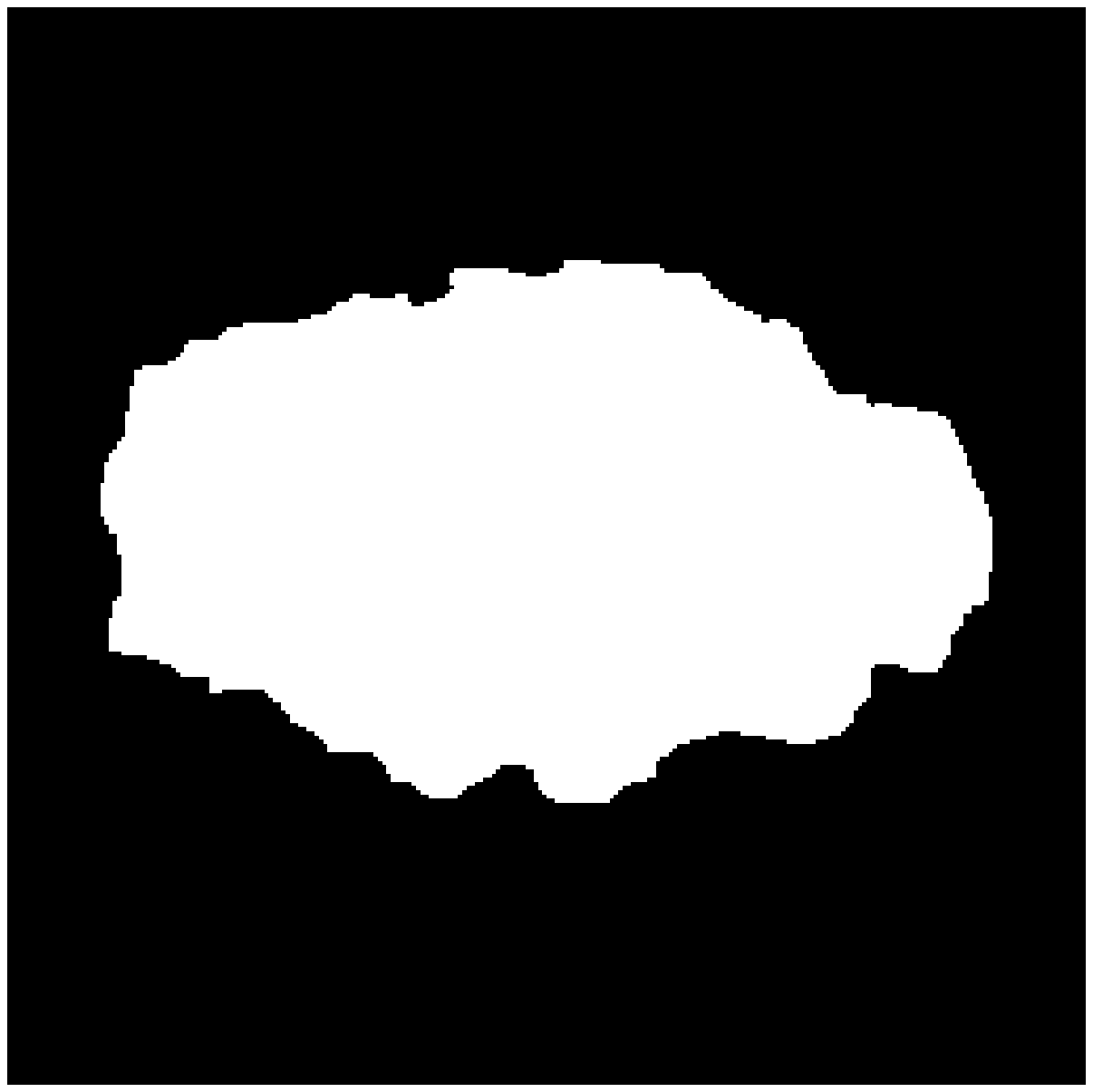}\\
\vspace{3mm}
\subcaption{\label{subfig:FDMC_LWT_4_bfgs_true_8_JJ_13_segh}Auto. selec. `$\boldsymbol{\bigtriangledown}$'}
\end{subfigure}
\begin{subfigure}{0.24\linewidth}
\centering
\includegraphics[width = 2cm]{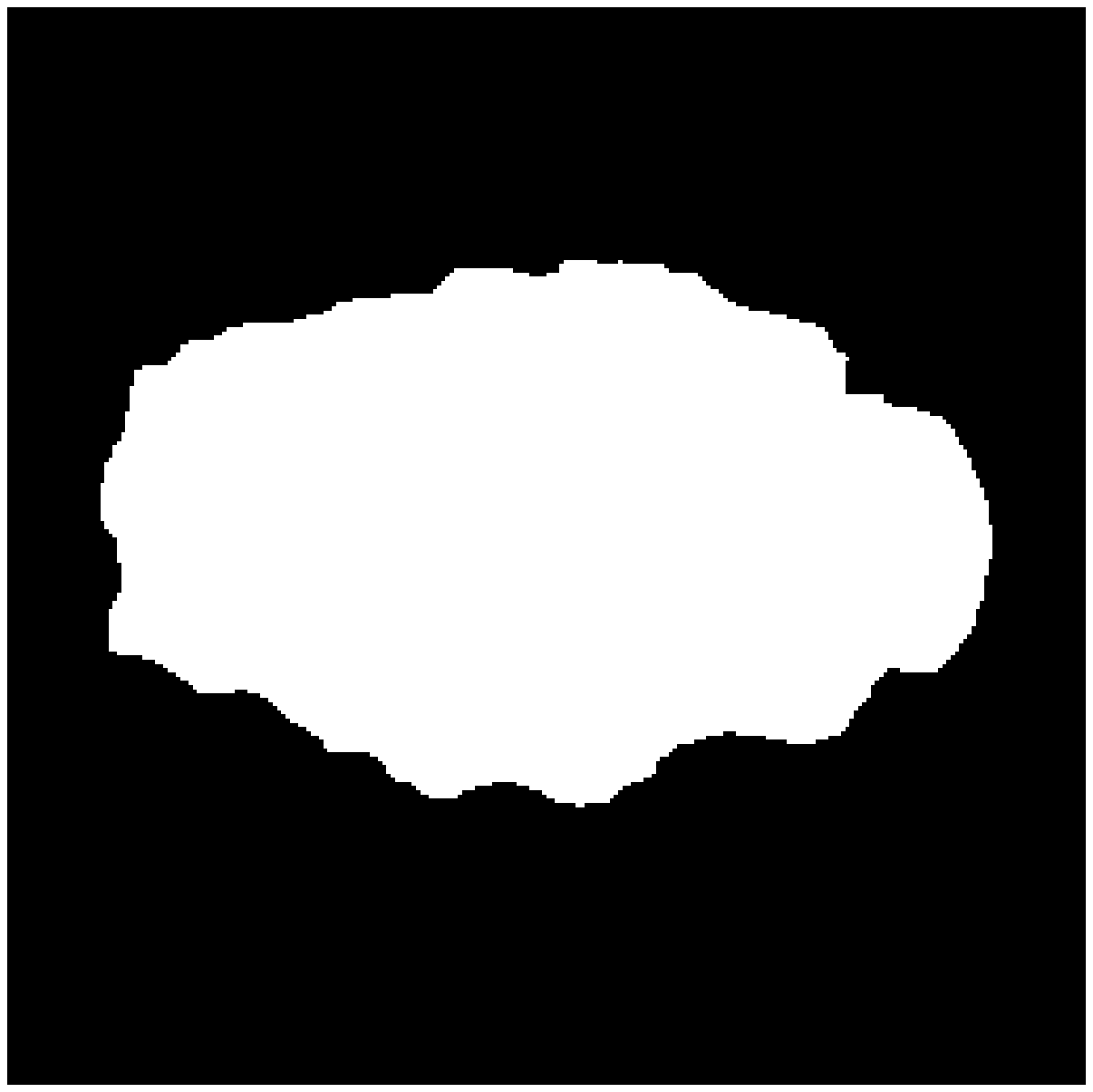}\\
\vspace{3mm}
\subcaption{\label{subfig:FDMC_LWT_4_bfgs_est_8_JJ_13_segh}Auto. selec. `$\color{bclair} \boldsymbol{\bigtriangledown}$'}
\end{subfigure}

\caption{\label{fig:True_Est_DE_13} Generalized SURE computed either from \textit{true} covariance matrix $\boldsymbol{\mathcal{S}}$~\eqref{eq:smpl_av}, or from \textit{estimated} covariance matrix~\eqref{eq:spat_av} for Textures ``\textbf{D}" and ``\textbf{E}" (first row).
Segmentations obtained minimizing the above generalized SURE (second row).
Segmentations obtained with automated selection of hyperparameters from Algorithm~\ref{alg:BFGS}, using generalized SUGAR with either \textit{true} covariance matrix or \textit{estimated} covariance matrix (third row).
}
\end{figure}

\subsection{Automated selection of hyperparameters}
\label{subsec:BFGS_exp}

Section~\ref{subsec:log_cov_struct} has shown the relevance of Algorithm~\ref{alg:BFGS} by comparing its performance against those obtained from a grid search on hyperparameters $\Lambda$. 
Section~\ref{subsec:BFGS_exp} will now test the practical effectiveness of the proposed procedure by assessing the convergence of the quasi-Newton algorithm and corresponding performance in hyperparameter selection and segmentation, avoiding the recourse to any ground truth and hence to the greedy and unfeasible grid search.

\subsubsection{Effective convergence of quasi-Newton Algorithm}
\label{subsec:grid_bfgs}

The convergence of quasi-Newton Algorithm~\ref{alg:BFGS} is assessed empirically comparing automatically selected hyperparameters $\widehat{\boldsymbol{\Lambda}}_{\nu , \boldsymbol{\varepsilon}}^{\mathrm{BFGS}}$ with optimal hyperparameters found from exhaustive grid search $\widehat{\boldsymbol{\Lambda}}_{\nu , \boldsymbol{\varepsilon}}^{\dagger}$.\\

Figures~\ref{subfig:FDMC_LWT_1_fdmc_true_8_JJ_13}~and~\ref{subfig:FDMC_LWT_1_fdmc_est_8_JJ_13} illustrate that $\widehat{\boldsymbol{\Lambda}}_{\nu , \boldsymbol{\varepsilon}}^{\mathrm{BFGS}}(\boldsymbol{\ell}\lvert \boldsymbol{\mathcal{S}})$ (`$\boldsymbol{\bigtriangledown}$') and $\widehat{\boldsymbol{\Lambda}}_{\nu , \boldsymbol{\varepsilon}}^{\mathrm{BFGS}}(\boldsymbol{\ell}\lvert \widehat{\boldsymbol{\mathcal{S}}})$ (`$\color{bclair}\boldsymbol{\bigtriangledown}$') respectively match $\widehat{\boldsymbol{\Lambda}}_{\nu , \boldsymbol{\varepsilon}}^{\dagger}(\boldsymbol{\ell}\lvert \boldsymbol{\mathcal{S}})$ (`$\boldsymbol{\bigtriangleup}$') and $\widehat{\boldsymbol{\Lambda}}_{\nu , \boldsymbol{\varepsilon}}^{\dagger}(\boldsymbol{\ell}\lvert \widehat{\boldsymbol{\mathcal{S}}})$ (`$\color{bclair}\boldsymbol{\bigtriangleup}$') in the case of Texture~``\textbf{D}".
Similar conclusions can be drawn from Figures~\ref{subfig:FDMC_LWT_4_fdmc_true_8_JJ_13}~and~\ref{subfig:FDMC_LWT_4_fdmc_est_8_JJ_13} for Texture~``\textbf{E}".\\

Figure~\ref{fig:True_Est_DE_13} and the quantitative results provided in Table~\ref{tab:grid_bfgs} show the convergence of Algorithm~\ref{alg:BFGS} using $\boldsymbol{\mathcal{S}}$ (resp. $\widehat{\boldsymbol{\mathcal{S}}}$) toward the minimum of $\widehat{R}_{\nu,\boldsymbol{\varepsilon}}(\boldsymbol{\ell};\boldsymbol{\Lambda} \lvert \boldsymbol{\mathcal{S}})$ (resp. $\widehat{R}_{\nu,\boldsymbol{\varepsilon}}(\boldsymbol{\ell};\boldsymbol{\Lambda} \lvert \widehat{\boldsymbol{\mathcal{S}}})$).\\

In term of computational cost, Algorithm~\ref{alg:BFGS} requires an average of $40$ calls of Algorithm~\ref{alg:PD}, compared to $225$ calls needed to perform grid search at Section~\ref{subsec:inf_cor}.

\subsubsection{Automated selection of $\boldsymbol{\Lambda}$ and segmentation performance}

Ten realizations of Textures~``\textbf{D}"~and~``\textbf{E}" are generated following the procedure described in Section~\ref{subsec:synth_text}.
For each of them, Algorithm~\ref{alg:BFGS} is run twice, first using $\boldsymbol{\mathcal{S}}$ and second using $\widehat{\boldsymbol{\mathcal{S}}}$.\\

Since here no grid search is performed, the minimum value of quadratic risk is unknown.
The performance will hence be measured in terms of \textit{normalized one-sample quadratic risk} $\widetilde{\mathcal{R}}$ defined as
\begin{align}
\widetilde{\mathcal{R}}(\boldsymbol{\ell} \lvert \boldsymbol{\mathcal{S}}) = \frac{\mathcal{R}(\boldsymbol{\ell}; \widehat{\boldsymbol{\Lambda}}^{\mathrm{BFGS}}_{\nu , \boldsymbol{\varepsilon}}(\boldsymbol{\ell} \lvert \boldsymbol{\mathcal{S}}))}{\lVert \widehat{\boldsymbol{h}}_{\mathrm{LR}}(\boldsymbol{\ell}) - \bar{\boldsymbol{h}} \rVert_2^2} = \frac{\lVert \widehat{\boldsymbol{h}}^{\mathrm{BFGS}}_{\nu, \boldsymbol{\varepsilon}}(\boldsymbol{\ell}\lvert \boldsymbol{\mathcal{S}}) - \bar{\boldsymbol{h}} \rVert^2_2}{\lVert \widehat{\boldsymbol{h}}_{\mathrm{LR}}(\boldsymbol{\ell}) - \bar{\boldsymbol{h}} \rVert_2^2},
\end{align}
measuring the improvement of the estimation achieved using TV-based texture segmentation~\eqref{eq:jointpb} with hyperparameters automatically selected by Algorithm~\ref{alg:BFGS}, compared to the classical least square estimate $\widehat{\boldsymbol{h}}_{\mathrm{LR}}$.

Averaged performance over ten realizations, presented in~Table~\ref{tab:bfgs_perf}, show that the quadratic risk $\mathcal{R}$ obtained is decrease by a factor of $16$ for Texture~``\textbf{D}" and of $14$ for Texture~``\textbf{E}".
The corresponding \textit{segmentation error} is as low as $6\%$ for Texture~``\textbf{D}", and $3\%$ for Texture~``\textbf{E}".
Further, the use of \textit{estimated} covariance matrix does not degrade achieved performance compare to using \textit{true} covariance matrix. \\

Hence, Algorithm~\ref{alg:BFGS}, using the \textit{estimated} covariance $\widehat{\boldsymbol{\mathcal{S}}}$, computed from~\eqref{eq:spat_av}, provides an efficient, parameter-free, automated and data-driven texture segmentation procedure.

\setlength{\tabcolsep}{1mm}
\begin{table}[h!]
\centering
\begin{tabular}{lcccc}
\toprule
 & \multicolumn{2}{c}{Texture~``\textbf{D}"}  & \multicolumn{2}{c}{Texture~``\textbf{E}"} \\
 \midrule
 Covariance matrix &  $\boldsymbol{\mathcal{S}}$ &  $\widehat{\boldsymbol{\mathcal{S}}}$ &  $\boldsymbol{\mathcal{S}}$ &  $\widehat{\boldsymbol{\mathcal{S}}}$ \\
 \midrule
$\widetilde{\mathcal{R}}(\boldsymbol{\ell} \lvert \cdot)$ & $0.060\pm 0.003 $ & $0.057\pm 0.002 $ & $0.071\pm 0.003  $ & $0.073\pm 0.004$\\
$\mathcal{P}(\boldsymbol{\ell}; \widehat{\boldsymbol{\Lambda}}_{\nu , \boldsymbol{\varepsilon}}^\mathrm{BFGS}(\boldsymbol{\ell} \lvert \cdot))$ (\%) & $5.4\pm 0.7$& $6.8\pm 1.5$ & $3.3\pm 0.7$ & $2.8\pm 0.3$\\
\bottomrule
\end{tabular}
\caption{\label{tab:bfgs_perf} Averaged performance of TV-based texture segmentation with automated selection of hyperparameters. 
}
\end{table}

\section{Conclusion}

This work was focused on devising a procedure for the automated selection of the hyperparameters of parametric estimators, such as e.g., parametric linear filtering or penalized least squares. 
The main result obtained here consists of a theoretically grounded and practical operational fully-automated data driven procedure, that requires neither ground truth nor expert-based knowledge and work satisfactorily even when applied to a single observation of data.\\

To that end, Stein Unbiased Risk Estimator (SURE) was rewritten to account for additive correlated Gaussian noise, with any covariance structure.
The main contribution compared to state-of-the-art procedure relies on including the covariance matrix of the noise only in SURE, rather than in the data fidelity term.
The benefit is twofold: handling with a strongly convex function when Penalized Least Square is considered, and avoiding costly, if not intractable, inversion of the covariance matrix.
Differentiating this Generalized SURE with respect to hyperparameters, an estimator for the risk gradient was designed, permitting to propose a Generalized Finite Difference Monte Carlo Stein Unbiased GrAdient Risk (SUGAR) estimate.
The asymptotic unbiasedness of Generalized SUGAR was assessed theoretically, based on regularity assumptions on the parametric estimator.\\
Further, the case of sequential parametric estimators is discussed in depth in the case of primal-dual minimization scheme for Penalized Least Squares and a differentiated scheme is derived.\\

Embedding Generalized SURE and SUGAR into a quasi-Newton algorithm enabled to perform an automated risk minimization.
An explicit algorithm permitting to implement the minimization was proposed.\\

To assess the performance of this automated hyperparameter selection procedure devised in a general setting, it has been customized to the specific problem of texture segmentation, based on multiscale descriptors (wavelet leaders) and nonsmooth Total-Variation based penalization. 
This problem is uneasy because observations are in nature multiscale, with inhomogeneous variance across scales and correlations both across scales and in space at each scale. 
Further, variances and correlations are unknown and need to be estimated directly from data. \\

Numerical simulations, conducted on ten realizations of synthetic piecewise fractal textures, permitted to show that the proposed strategy yield satisfactory performance in selecting automatically the penalization hyperparameter, leading to excellent texture segmentation, with no ad-hoc (or expert-based) tuning and without prior knowledge for ground truth, and using one-sample estimate of the covariance matrix. \\

The corresponding {\sc Matlab} routines, developed by ourselves and implementing these tools, ready for applications to real-world texture segmentation, where hyperparameter tuning constitutes an on-going hot topic, will be made publicly available to the research community in a documented toolbox at the time of publication.

\clearpage

\appendix

\section{Proof of Theorem~\ref{thm:SURE}}
\label{app:SURE}

\begin{proof}
For ease of computation we first define the \textit{predictor} in~Definition~\ref{def:predictor} and the ground truth \textit{prediction} in~Definition~\ref{def:prediction}.

\begin{definition}[\textit{Predictor}]
\label{def:predictor}
From the estimator of underlying features $\widehat{\boldsymbol{x}}(\boldsymbol{y}; \boldsymbol{\Lambda})$ one can equivalently consider a \textit{prediction} estimator
\begin{align}
\label{eq:y_est}
\widehat{\boldsymbol{y}}(\boldsymbol{y} ; \boldsymbol{\Lambda}) \triangleq  \boldsymbol{\Phi} \widehat{\boldsymbol{x}}(\boldsymbol{y} ; \boldsymbol{\Lambda}).
\end{align}
\end{definition}
Indeed, from Assumption~\ref{hyp:full_rank}, $\boldsymbol{\Phi}^* \boldsymbol{\Phi}$ is invertible, and the relation~\eqref{eq:y_est} can be inverted computing
\begin{align}
\label{eq:x_est}
\widehat{\boldsymbol{x}}(\boldsymbol{y} ; \boldsymbol{\Lambda}) = \left(\boldsymbol{\Phi}^* \boldsymbol{\Phi}\right)^{-1}  \boldsymbol{\Phi}^* \widehat{\boldsymbol{y}}(\boldsymbol{y} ; \boldsymbol{\Lambda}).
\end{align}

\begin{definition}[\textit{Prediction} ground truth]
\label{def:prediction}
The noise-free observation writes
\begin{align}
\label{eq:y_true}
 \bar{\boldsymbol{y}} \triangleq \mathbb{E}_{\boldsymbol{\zeta}} \boldsymbol{y} =\boldsymbol{\Phi} \bar{\boldsymbol{x}}.
\end{align}
\end{definition}

Thus, the quadratic risk defined in~\eqref{eq:risk_def} can be expressed using operator $\textbf{A}$ defined in~\eqref{eq:def_A} as
\begin{align}
R[\widehat{\boldsymbol{x}}]( \boldsymbol{\Lambda}) = \mathbb{E}_{\boldsymbol{\zeta}}  \left\lVert \boldsymbol{\Pi}\widehat{\boldsymbol{x}}(\boldsymbol{y} ; \boldsymbol{\Lambda}) - \boldsymbol{\Pi}\boldsymbol{x} \right\rVert_2^2
 &=\mathbb{E}_{\boldsymbol{\zeta}}  \lVert \boldsymbol{\Pi}\left(\boldsymbol{\Phi}^* \boldsymbol{\Phi}\right)^{-1}  \boldsymbol{\Phi}^*\left(\widehat{\boldsymbol{y}}(\boldsymbol{y} ; \boldsymbol{\Lambda}) -  \bar{\boldsymbol{y}} \right)\rVert_2^2,\\
 & \overset{\eqref{eq:def_A}}{=} \mathbb{E}_{\boldsymbol{\zeta}}  \lVert \textbf{A}\left(\widehat{\boldsymbol{y}}(\boldsymbol{y} ; \boldsymbol{\Lambda}) -  \bar{\boldsymbol{y}} \right)\rVert_2^2\nonumber
\end{align}
which will be easier to manipulate in the following when expressed in term of noise-free (or noisy) observations $\bar{\boldsymbol{y}}$ (or $\boldsymbol{y}$) and \textit{prediction} $\widehat{\boldsymbol{y}}$.\\

By construction, the matrix $ \textbf{A}$, defined in~\eqref{eq:def_A}, performs both: 
\begin{itemize}
\item The projection on the interest subspace $\mathcal{I}$ of $\mathcal{H}$ via the linear operator $\boldsymbol{\Pi}$.
\item The transition from predicted quantities $\widehat{\boldsymbol{y}}$ to estimated features $\widehat{\boldsymbol{x}}$, making use of relation~\eqref{eq:x_est}.\\
\end{itemize}

From now, for sake of simplicity, we make implicit the dependency of $\widehat{\boldsymbol{x}}$ in $(\boldsymbol{y} ; \boldsymbol{\Lambda})$. 
From the model~\eqref{eq:obs_gen_model} and the Assumption~\ref{hyp:gauss_noise} on the noise probability distribution, one directly derive two useful relations:
\begin{align}
\label{eq:use1}
&\mathbb{E}_{\boldsymbol{\zeta}}\left\lVert \textbf{A}\left( \boldsymbol{y} - \bar{\boldsymbol{y}} \right) \right\rVert_2^2 \overset{\eqref{eq:y_true}}{=} \mathbb{E}_{\boldsymbol{\zeta}}\left\lVert \textbf{A}\left( \boldsymbol{y} - \boldsymbol{\Phi} \bar{\boldsymbol{x}} \right) \right\rVert_2^2 = \mathbb{E}_{\boldsymbol{\zeta}}\left\lVert \textbf{A} \boldsymbol{\zeta} \right\rVert_2^2   &\overset{\text{Hyp.~\ref{hyp:gauss_noise}}}{=}  \mathrm{Tr}(\textbf{A} \boldsymbol{\mathcal{S}} \textbf{A}^*),\\
\label{eq:use2}
& \mathbb{E}_{\boldsymbol{\zeta}} \langle \textbf{A} \boldsymbol{y} , \textbf{A}\left( \boldsymbol{y} - \bar{\boldsymbol{y}} \right) \rangle \overset{\mathbb{E} (\boldsymbol{y} - \bar{\boldsymbol{y}}) = 0}{=} \mathbb{E}_{\boldsymbol{\zeta}} \langle \textbf{A} \left( \boldsymbol{y} - \bar{\boldsymbol{y}}\right) , \textbf{A}\left( \boldsymbol{y} - \bar{\boldsymbol{y}} \right) \rangle &\overset{\text{Hyp.~\ref{hyp:gauss_noise}}}{=}   \mathrm{Tr}(\textbf{A} \boldsymbol{\mathcal{S}} \textbf{A}^*).
\end{align} 
Thus the risk can be expanded as
\begin{align*}
R[\widehat{\boldsymbol{x}}]&( \boldsymbol{\Lambda}) \\
\triangleq \quad&\mathbb{E}_{\boldsymbol{\zeta}} \left\lVert  \textbf{A}\left( \widehat{\boldsymbol{y}} - \bar{\boldsymbol{y}} \right) \right\rVert_2^2 \\
= \quad&\mathbb{E}_{\boldsymbol{\zeta}} \left[\left\lVert  \textbf{A}\left( \widehat{\boldsymbol{y}} - \boldsymbol{y} \right) \right\rVert_2^2 +  \left\lVert \textbf{A}\left( \boldsymbol{y} - \bar{\boldsymbol{y}} \right) \right\rVert_2^2 + 2  \langle \textbf{A}\left( \widehat{\boldsymbol{y}} - \boldsymbol{y} \right), \textbf{A}\left(\boldsymbol{y} - \bar{\boldsymbol{y}} \right)  \rangle \right]\\
\overset{\eqref{eq:use1}}{=}  \, \, \,&\mathbb{E}_{\boldsymbol{\zeta}}\left[ \left\lVert  \textbf{A}\left( \widehat{\boldsymbol{y}} - \boldsymbol{y} \right) \right\rVert_2^2 + 2 \langle \textbf{A}\widehat{\boldsymbol{y}} , \textbf{A}\left( \boldsymbol{y} - \bar{\boldsymbol{y}}\right)  \rangle -  2 \langle \textbf{A} \boldsymbol{y} , \textbf{A}\left( \boldsymbol{y} - \bar{\boldsymbol{y}} \right) \rangle \right] + \mathrm{Tr}(\textbf{A} \boldsymbol{\mathcal{S}} \textbf{A}^*) \\
\overset{\eqref{eq:use2}}{=} \, \, \,&\mathbb{E}_{\boldsymbol{\zeta}}\left[ \left\lVert  \textbf{A}\left( \widehat{\boldsymbol{y}} - \boldsymbol{y} \right) \right\rVert_2^2 + 2 \langle \textbf{A}\widehat{\boldsymbol{y}} , \textbf{A}\left(\boldsymbol{y} - \bar{\boldsymbol{y}}\right) \rangle  \right]  - 2\mathrm{Tr}(\textbf{A} \boldsymbol{\mathcal{S}} \textbf{A}^*)   + \mathrm{Tr}(\textbf{A} \boldsymbol{\mathcal{S}} \textbf{A}^*) \\
= \quad&\mathbb{E}_{\boldsymbol{\zeta}}\left[ \left\lVert  \textbf{A}\left( \widehat{\boldsymbol{y}} - \boldsymbol{y} \right) \right\rVert_2^2 +  2 \langle \textbf{A}  \widehat{\boldsymbol{y}},  \textbf{A}\left( \boldsymbol{y} - \bar{\boldsymbol{y}} \right)  \rangle \right]-   \mathrm{Tr}(\textbf{A} \boldsymbol{\mathcal{S}} \textbf{A}^*)\\
= \quad&\mathbb{E}_{\boldsymbol{\zeta}} \left[ \left\lVert  \textbf{A}\left( \boldsymbol{\Phi}\widehat{\boldsymbol{x}} - \boldsymbol{y} \right) \right\rVert_2^2 +  2 \langle \textbf{A}^* \textbf{A}  \boldsymbol{\Phi}\widehat{\boldsymbol{x}},  \left( \boldsymbol{y} - \boldsymbol{\Phi}\bar{\boldsymbol{x}} \right)  \rangle  \right]-   \mathrm{Tr}(\textbf{A} \boldsymbol{\mathcal{S}} \textbf{A}^*)\\
= \quad&\mathbb{E}_{\boldsymbol{\zeta}} \left\lVert  \textbf{A}\left( \boldsymbol{\Phi}\widehat{\boldsymbol{x}} - \boldsymbol{y} \right) \right\rVert_2^2 +  2 \mathbb{E}_{\boldsymbol{\zeta}}\langle \textbf{A}^* \textbf{A}  \boldsymbol{\Phi}\widehat{\boldsymbol{x}},  \boldsymbol{\zeta}  \rangle -   \mathrm{Tr}(\textbf{A} \boldsymbol{\mathcal{S}} \textbf{A}^*)\\
= \quad&\mathbb{E}_{\boldsymbol{\zeta}} \left\lVert  \textbf{A}\left( \boldsymbol{\Phi}\widehat{\boldsymbol{x}} - \boldsymbol{y} \right) \right\rVert_2^2 +  2 \mathbb{E}_{\boldsymbol{\zeta}}\langle \textbf{A}^* \boldsymbol{\Pi} \widehat{\boldsymbol{x}},  \boldsymbol{\zeta}  \rangle -   \mathrm{Tr}(\textbf{A} \boldsymbol{\mathcal{S}} \textbf{A}^*),
\end{align*}
re-injecting the definition of $\textbf{A}$~\eqref{eq:def_A} in terms of $\boldsymbol{\Phi}$ and $\boldsymbol{\Pi}$.\\

The second term, $\mathbb{E}_{\boldsymbol{\zeta}}\langle \textbf{A}^* \boldsymbol{\Pi} \widehat{\boldsymbol{x}} (\boldsymbol{y} ; \boldsymbol{\Lambda}),  \boldsymbol{\zeta}  \rangle$, is called the \textit{degrees of freedom}~\cite{efron1986biased}. From Assumption~\ref{hyp:reg_int} it is well-defined and writes
\begin{align}
\label{eq:dof_density}
\mathbb{E}_{\boldsymbol{\zeta}}\langle \textbf{A}^*  \boldsymbol{\Pi} \widehat{\boldsymbol{x}} (\boldsymbol{y} ; \boldsymbol{\Lambda}),  \boldsymbol{\zeta}  \rangle &= \\
& \frac{1}{\sqrt{(2\pi)^P \lvert \mathrm{det}(\boldsymbol{\mathcal{S}})\rvert}}\int  \langle \textbf{A}^*  \boldsymbol{\Pi} \widehat{\boldsymbol{x}} (\boldsymbol{y} ; \boldsymbol{\Lambda}),  \boldsymbol{\zeta}  \rangle \, \exp \left(-\frac{\boldsymbol{\zeta}^*\boldsymbol{\mathcal{S}}^{-1}\boldsymbol{\zeta}}{2}\right)  \mathrm{d}\boldsymbol{\zeta},
\end{align}
hence requiring generalized Stein's lemma to be estimated\footnote{
Stein's lemma states that, for a real random variable $\zeta\sim\mathcal{N}(0, \sigma^2)$, if $f : \mathbb{R} \rightarrow \mathbb{R}$ is a function such that both $\mathbb{E}_{\zeta}  [\zeta f(\zeta) ]$ and $\mathbb{E}_{\zeta}  [f'(\zeta) ]$ exist, then $\mathbb{E}_{\zeta}  [\zeta f(\zeta) ] = \sigma^2\mathbb{E}_{\zeta}  [f'(\zeta) ]$. Its demonstration relies on appropriate integration by parts.
}.

Because of the off-diagonal terms in $\boldsymbol{\mathcal{S}}^{-1}$, the Integration by Parts (IP) required to transform~\eqref{eq:dof_density} cannot be directly justified, thus Stein's lemma generalization to $\mathcal{G}$-valued random variable $\boldsymbol{\zeta}$ is not straightforward. 
Hence we propose to first diagonalize $\boldsymbol{\mathcal{S}}^{-1}$ (which is a symmetric matrix) in a orthonormal basis, obtaining
\begin{align*}
\boldsymbol{\mathcal{S}}^{-1} = \boldsymbol{\mathcal{V}}^* \boldsymbol{\mathcal{D}} \boldsymbol{\mathcal{V}},
\end{align*}
with $\boldsymbol{\mathcal{V}}$ an orthonormal matrix (which columns are eigenvectors of $\boldsymbol{\mathcal{S}}^{-1}$) and $\boldsymbol{\mathcal{D}} = \mathrm{diag}(\beta_1, \hdots, \beta_P)$ containing (positive) eigenvalues of $\boldsymbol{\mathcal{S}}^{-1}$. Then, setting $\boldsymbol{\vartheta} = \boldsymbol{\mathcal{V}} \boldsymbol{\zeta}$
\begin{align*}
&\mathbb{E}_{\boldsymbol{\vartheta}}\langle \textbf{A}^*  \boldsymbol{\Pi} \widehat{\boldsymbol{x}} (\boldsymbol{y} ; \boldsymbol{\Lambda}),  \boldsymbol{\vartheta}  \rangle  =\\
&\frac{1}{\sqrt{(2\pi)^P \lvert \mathrm{det}(\boldsymbol{\mathcal{S}})\rvert}}\int  \langle \textbf{A}^*  \boldsymbol{\Pi} \widehat{\boldsymbol{x}} (\boldsymbol{y} ; \boldsymbol{\Lambda}),  \boldsymbol{\mathcal{V}}^{-1} \boldsymbol{\vartheta}   \rangle \, \exp \left(-\frac{ \boldsymbol{\vartheta} ^*\boldsymbol{\mathcal{D}} \boldsymbol{\vartheta}}{2}\right) \, \lvert \mathrm{det}(\boldsymbol{\mathcal{V}}^{-1})\rvert \mathrm{d}\boldsymbol{\vartheta}.
\end{align*}
with $\boldsymbol{\vartheta}^*\boldsymbol{\mathcal{D}}\boldsymbol{\vartheta} = \sum_{p=1}^P \beta_{p} \lvert \vartheta_{p} \rvert^2$.

Since $\boldsymbol{\mathcal{V}}$ is orthonormal: $\boldsymbol{\mathcal{V}}^{-1} = \boldsymbol{\mathcal{V}}^*$ and $\lvert \mathrm{det}(\boldsymbol{\mathcal{V}}^{-1}) \rvert = 1$, leading to
\begin{align*}
&\mathbb{E}_{\boldsymbol{\vartheta}}\langle \textbf{A}^*  \boldsymbol{\Pi} \widehat{\boldsymbol{x}} (\boldsymbol{y} ; \boldsymbol{\Lambda}),  \boldsymbol{\vartheta}  \rangle  \\
&=\frac{1}{\sqrt{(2\pi)^P \lvert \mathrm{det}(\boldsymbol{\mathcal{S}})\rvert}}\int  \langle \boldsymbol{\mathcal{V}} \textbf{A}^*  \boldsymbol{\Pi} \widehat{\boldsymbol{x}} (\boldsymbol{y} ; \boldsymbol{\Lambda}),   \boldsymbol{\vartheta}   \rangle \, \exp \left(-\frac{ \boldsymbol{\vartheta} ^*\boldsymbol{\mathcal{D}} \boldsymbol{\vartheta}}{2}\right) \,  \mathrm{d}\boldsymbol{\vartheta}\\
&=\frac{1}{\sqrt{(2\pi)^P \lvert \mathrm{det}(\boldsymbol{\mathcal{S}})\rvert}}\int  \sum_{p= 1}^P \left( \boldsymbol{\mathcal{V}} \textbf{A}^*  \boldsymbol{\Pi} \widehat{\boldsymbol{x}} (\boldsymbol{y} ; \boldsymbol{\Lambda})\right)_p \vartheta_p    \, \exp \left(-\frac{\sum_{p=1}^P \beta_p \lvert \vartheta_p \rvert^2}{2}\right) \,  \mathrm{d}\vartheta_1 \hdots \mathrm{d}\vartheta_P\\
&\overset{\text{(IP)}}{=} \mathbb{E}_{\boldsymbol{\vartheta}} \left[ \sum_{p=1}^P \frac{1}{\beta_p} \dfrac{\partial \left( \boldsymbol{\mathcal{V}} \textbf{A}^*  \boldsymbol{\Pi} \widehat{\boldsymbol{x}} (\boldsymbol{y} ; \boldsymbol{\Lambda})\right)_p }{\partial \vartheta_p} \right] \\
&= \mathbb{E}_{\boldsymbol{\vartheta}} \left[ \mathrm{Tr} \left( \boldsymbol{\mathcal{D}}^{-1} \dfrac{\partial \left( \boldsymbol{\mathcal{V}} \textbf{A}^*  \boldsymbol{\Pi} \widehat{\boldsymbol{x}} (\boldsymbol{y} ; \boldsymbol{\Lambda})\right) }{\partial \boldsymbol{\vartheta}}\right) \right],
\end{align*}
where $\displaystyle \dfrac{\partial \left( \boldsymbol{\mathcal{V}} \textbf{A}^*  \boldsymbol{\Pi} \widehat{\boldsymbol{x}} (\boldsymbol{y} ; \boldsymbol{\Lambda})\right) }{\partial \boldsymbol{\vartheta}}$ denotes the Jacobian matrix of  $\boldsymbol{\mathcal{V}} \textbf{A}^*  \boldsymbol{\Pi} \widehat{\boldsymbol{x}} (\boldsymbol{y} ; \boldsymbol{\Lambda})$ with respect to the variable $\boldsymbol{\vartheta} \triangleq \boldsymbol{\mathcal{V}}^{-1} \boldsymbol{\zeta}$. In order to go back to variable $\boldsymbol{\zeta}$, we make use of~\eqref{eq:obs_gen_model} relating $\boldsymbol{y}$ and $\boldsymbol{\zeta}$, and apply the reverse change of variable $\boldsymbol{\zeta} \triangleq \boldsymbol{\mathcal{V}} \boldsymbol{\vartheta}$ and obtain
\begin{align*}
 \dfrac{\partial \left( \boldsymbol{\mathcal{V}} \textbf{A}^*  \boldsymbol{\Pi} \widehat{\boldsymbol{x}} (\boldsymbol{y} ; \boldsymbol{\Lambda})\right) }{\partial \boldsymbol{\vartheta}} =  \boldsymbol{\mathcal{V}}  \dfrac{\partial \left( \textbf{A}^*  \boldsymbol{\Pi} \widehat{\boldsymbol{x}} (\boldsymbol{y} ; \boldsymbol{\Lambda})\right) }{\partial \boldsymbol{\zeta}} \boldsymbol{\mathcal{V}}^{-1} = \boldsymbol{\mathcal{V}}  \dfrac{\partial \left( \textbf{A}^*  \boldsymbol{\Pi} \widehat{\boldsymbol{x}} (\boldsymbol{y} ; \boldsymbol{\Lambda})\right) }{\partial \boldsymbol{y}} \boldsymbol{\mathcal{V}}^{-1},
\end{align*}
because $\partial_{\boldsymbol{\zeta}} \boldsymbol{y} = \boldsymbol{I}_P$ (the identity matrix of size $P\times P$).

Using the cyclicality of trace and the fact that $\boldsymbol{\mathcal{V}}$ is orthonormal, we finally obtain a closed-form expression of the degrees of freedom:
\begin{align*}
\mathbb{E}_{\boldsymbol{\vartheta}}\langle \textbf{A}^*  \boldsymbol{\Pi} \widehat{\boldsymbol{x}} (\boldsymbol{y} ; \boldsymbol{\Lambda}),  \boldsymbol{\vartheta}  \rangle  &= \mathbb{E}_{\boldsymbol{\zeta}} \left[ \mathrm{Tr} \left( \boldsymbol{\mathcal{D}}^{-1} \boldsymbol{\mathcal{V}}  \dfrac{\partial \left( \textbf{A}^*  \boldsymbol{\Pi} \widehat{\boldsymbol{x}} (\boldsymbol{y} ; \boldsymbol{\Lambda})\right) }{\partial \boldsymbol{y}} \boldsymbol{\mathcal{V}}^{-1}\right) \right],\\
&= \mathbb{E}_{\boldsymbol{\zeta}} \left[ \mathrm{Tr} \left( \boldsymbol{\mathcal{V}}^{-1}\boldsymbol{\mathcal{D}}^{-1} \boldsymbol{\mathcal{V}}  \dfrac{\partial \left( \textbf{A}^*  \boldsymbol{\Pi} \widehat{\boldsymbol{x}} (\boldsymbol{y} ; \boldsymbol{\Lambda})\right) }{\partial \boldsymbol{y}} \right) \right],\\
&= \mathbb{E}_{\boldsymbol{\zeta}} \left[ \mathrm{Tr} \left( \boldsymbol{\mathcal{S}} \dfrac{\partial \left( \textbf{A}^*  \boldsymbol{\Pi} \widehat{\boldsymbol{x}} (\boldsymbol{y} ; \boldsymbol{\Lambda})\right) }{\partial \boldsymbol{y}} \right) \right],\\
&= \mathbb{E}_{\boldsymbol{\zeta}} \left[ \mathrm{Tr} \left( \boldsymbol{\mathcal{S}}\textbf{A}^* \boldsymbol{\Pi} \partial_{\boldsymbol{y}}  \widehat{\boldsymbol{x}}(\boldsymbol{y} ; \boldsymbol{\Lambda}) \right) \right]
\end{align*}
\end{proof}

\section{Finite Difference Monte Carlo SURE}
\label{app:SURE_FDMC}

\begin{proof}
First, remark that since $\boldsymbol{y} \mapsto \widehat{\boldsymbol{x}}(\boldsymbol{y}; \boldsymbol{\Lambda})$ is Lipschitz continuous from Assumption~\ref{hyp:lip_L1}, it is Lebesgue differentiable almost everywhere and its Lebesgue derivative equals its weak derivative almost everywhere.
Then, based on Theorem~\ref{thm:SURE}, the only difficulty relies in dominating the degrees of freedom, since it is the only term depending on the Finite Difference step $\nu$.  \\
Applying successively both Monte Carlo and Finite Difference strategies presented in Section~\ref{sec:est_dof} we obtain
\begin{align}
\mathrm{Tr} \left( \boldsymbol{\mathcal{S}}\textbf{A}^* \boldsymbol{\Pi} \partial_{\boldsymbol{y}}  \widehat{\boldsymbol{x}}(\boldsymbol{y} ; \boldsymbol{\Lambda}) \right) &\overset{\text{Monte Carlo}}{=} \mathbb{E}_{\boldsymbol{\varepsilon}} \left\langle \boldsymbol{\mathcal{S}}\textbf{A}^* \boldsymbol{\Pi} \dfrac{\partial  \widehat{\boldsymbol{x}}(\boldsymbol{y} ; \boldsymbol{\Lambda}) }{\partial \boldsymbol{y}}\left[ \boldsymbol{\varepsilon}\right], \boldsymbol{\varepsilon} \right\rangle\\
& \overset{\text{Finite Difference}}{=} \mathbb{E}_{\boldsymbol{\varepsilon}} \left\langle \textbf{A}^* \boldsymbol{\Pi} \lim\limits_{\nu \rightarrow 0} \frac{\widehat{\boldsymbol{x}}(\boldsymbol{y} + \nu \boldsymbol{\varepsilon}; \boldsymbol{\Lambda}) - \widehat{\boldsymbol{x}}(\boldsymbol{y} ; \boldsymbol{\Lambda})}{\nu}, \boldsymbol{\mathcal{S}}\boldsymbol{\varepsilon} \right\rangle. \nonumber
\end{align}
Making use of the centered normalized Gaussian probability density function of $\boldsymbol{\varepsilon}$, the above expectation writes
\begin{align}
 &\mathbb{E}_{\boldsymbol{\varepsilon}} \left\langle \textbf{A}^* \boldsymbol{\Pi} \lim\limits_{\nu \rightarrow 0} \frac{\widehat{\boldsymbol{x}}(\boldsymbol{y} + \nu \boldsymbol{\varepsilon}; \boldsymbol{\Lambda}) - \widehat{\boldsymbol{x}}(\boldsymbol{y} ; \boldsymbol{\Lambda})}{\nu}, \boldsymbol{\mathcal{S}}\boldsymbol{\varepsilon} \right\rangle \\ \nonumber
\text{(p.d.f. of $\boldsymbol{\varepsilon}$)}& =  \int_{\mathbb{R}^P}  \lim\limits_{\nu \rightarrow 0} \left\langle \textbf{A}^* \boldsymbol{\Pi}  \frac{\widehat{\boldsymbol{x}}(\boldsymbol{y} + \nu \boldsymbol{\varepsilon}; \boldsymbol{\Lambda}) - \widehat{\boldsymbol{x}}(\boldsymbol{y} ; \boldsymbol{\Lambda})}{\nu}, \boldsymbol{\mathcal{S}}\boldsymbol{\varepsilon}  \right\rangle  \frac{\mathrm{e}^{-\frac{\lVert\boldsymbol{\varepsilon}\rVert^2}{2}} \,\mathrm{d}\boldsymbol{\varepsilon}}{(2\pi)^{P/2}} 
\end{align}
Then the following majorations hold
\begin{align}
\label{eq:maj_varepsilon}
&\left\lvert \left\langle \textbf{A}^* \boldsymbol{\Pi} \frac{\left( \widehat{\boldsymbol{x}}(\boldsymbol{y} + \nu \boldsymbol{\varepsilon}; \boldsymbol{\Lambda}) - \widehat{\boldsymbol{x}}(\boldsymbol{y} ; \boldsymbol{\Lambda})\right)}{\nu}, \boldsymbol{\mathcal{S}}\boldsymbol{\varepsilon}  \right\rangle \right\rvert \mathrm{e}^{-\frac{\lVert\boldsymbol{\varepsilon}\rVert^2}{2}}  \\ \nonumber
\text{(Cauchy-Schwarz) } & \leq \left\lVert\textbf{A}^* \boldsymbol{\Pi} \frac{\left( \widehat{\boldsymbol{x}}(\boldsymbol{y} + \nu \boldsymbol{\varepsilon}; \boldsymbol{\Lambda}) - \widehat{\boldsymbol{x}}(\boldsymbol{y} ; \boldsymbol{\Lambda})\right)}{\nu} \right\rVert \left\lVert \boldsymbol{\mathcal{S}}\boldsymbol{\varepsilon}\right\rVert\mathrm{e}^{-\frac{\lVert\boldsymbol{\varepsilon}\rVert^2}{2}}\\ \nonumber
\text{(Bounded operators) }  &  \leq \lVert \textbf{A}^* \rVert  \lVert \boldsymbol{\Pi} \rVert \left\lVert\frac{\widehat{\boldsymbol{x}}(\boldsymbol{y} + \nu \boldsymbol{\varepsilon}; \boldsymbol{\Lambda}) - \widehat{\boldsymbol{x}}(\boldsymbol{y} ; \boldsymbol{\Lambda})}{\nu} \right\rVert \lVert \boldsymbol{\mathcal{S}} \rVert \lVert \boldsymbol{\varepsilon} \rVert\mathrm{e}^{-\frac{\lVert\boldsymbol{\varepsilon}\rVert^2}{2}}\\ 
\text{(Hyp. \ref{hyp:lip_L1}: $L_1$-Lipschitz) } & \leq \lVert \textbf{A}^* \rVert  \lVert \boldsymbol{\Pi} \rVert  L_1 \lVert \boldsymbol{\varepsilon} \rVert \lVert \boldsymbol{\mathcal{S}} \rVert \lVert \boldsymbol{\varepsilon} \rVert\mathrm{e}^{-\frac{\lVert\boldsymbol{\varepsilon}\rVert^2}{2}}, \nonumber 
\end{align}
with $\lVert \boldsymbol{\varepsilon} \rVert^2 \mathrm{e}^{-\frac{\lVert\boldsymbol{\varepsilon}\rVert^2}{2}}$ integrable over $\mathbb{R}^P$.
Further, the domination being independent of $\nu$ the limit can be interchanged with the integral on variable $\boldsymbol{\varepsilon}$ which gives 
\begin{align}
\label{eq:int_varepsilon}
\int_{\mathbb{R}^P}  \lim\limits_{\nu \rightarrow 0}& \left\langle \textbf{A}^* \boldsymbol{\Pi}  \frac{\widehat{\boldsymbol{x}}(\boldsymbol{y} + \nu \boldsymbol{\varepsilon}; \boldsymbol{\Lambda}) - \widehat{\boldsymbol{x}}(\boldsymbol{y} ; \boldsymbol{\Lambda})}{\nu}, \boldsymbol{\mathcal{S}}\boldsymbol{\varepsilon}  \right\rangle  \frac{\mathrm{e}^{-\frac{\lVert\boldsymbol{\varepsilon}\rVert^2}{2}} \,\mathrm{d}\boldsymbol{\varepsilon}}{(2\pi)^{P/2}} \\
&=\lim\limits_{\nu \rightarrow 0} \int_{\mathbb{R}^P}   \left\langle \textbf{A}^* \boldsymbol{\Pi}  \frac{\widehat{\boldsymbol{x}}(\boldsymbol{y} + \nu \boldsymbol{\varepsilon}; \boldsymbol{\Lambda}) - \widehat{\boldsymbol{x}}(\boldsymbol{y} ; \boldsymbol{\Lambda})}{\nu}, \boldsymbol{\mathcal{S}}\boldsymbol{\varepsilon}  \right\rangle  \frac{\mathrm{e}^{-\frac{\lVert\boldsymbol{\varepsilon}\rVert^2}{2}} \,\mathrm{d}\boldsymbol{\varepsilon}}{(2\pi)^{P/2}}, \nonumber
\end{align}
and 
\begin{align}
\label{eq:int_zeta}
&\left\lvert \lim\limits_{\nu \rightarrow 0} \int_{\mathbb{R}^P}   \left\langle \textbf{A}^* \boldsymbol{\Pi}  \frac{\widehat{\boldsymbol{x}}(\boldsymbol{y} + \nu \boldsymbol{\varepsilon}; \boldsymbol{\Lambda}) - \widehat{\boldsymbol{x}}(\boldsymbol{y} ; \boldsymbol{\Lambda})}{\nu}, \boldsymbol{\mathcal{S}}\boldsymbol{\varepsilon}  \right\rangle  \frac{\mathrm{e}^{-\frac{\lVert\boldsymbol{\varepsilon}\rVert^2}{2}} \,\mathrm{d}\boldsymbol{\varepsilon}}{(2\pi)^{P/2}} \right\rvert  \nonumber \\
&\leq \lVert \textbf{A}^* \rVert  \lVert \boldsymbol{\Pi} \rVert  L_1  \lVert \boldsymbol{\mathcal{S}} \rVert  \int_{\mathbb{R}^P}\lVert \boldsymbol{\varepsilon} \rVert^2 \frac{\mathrm{e}^{-\frac{\lVert\boldsymbol{\varepsilon}\rVert^2}{2}} \,\mathrm{d}\boldsymbol{\varepsilon}}{(2\pi)^{P/2}} < \infty.
\end{align}
Then, Equation~\eqref{eq:int_varepsilon} means that 
\begin{align}
\mathrm{Tr} \left( \boldsymbol{\mathcal{S}}\textbf{A}^* \boldsymbol{\Pi} \partial_{\boldsymbol{y}}  \widehat{\boldsymbol{x}}(\boldsymbol{y} ; \boldsymbol{\Lambda}) \right)  = \lim\limits_{\nu \rightarrow 0} \mathbb{E}_{\boldsymbol{\varepsilon}}  \left\langle \textbf{A}^* \boldsymbol{\Pi}  \frac{\widehat{\boldsymbol{x}}(\boldsymbol{y} + \nu \boldsymbol{\varepsilon}; \boldsymbol{\Lambda}) - \widehat{\boldsymbol{x}}(\boldsymbol{y} ; \boldsymbol{\Lambda})}{\nu}, \boldsymbol{\mathcal{S}}\boldsymbol{\varepsilon}  \right\rangle.
\end{align}
Further, the majoration obtained in~Equation~\eqref{eq:int_zeta} not depending on $\boldsymbol{\zeta}$ (since $L_1$ does not depend on $\boldsymbol{y}$, as stated in Assumption~\ref{hyp:lip_L1}), neither on $\nu$, the limits on $\nu$ and the expected value with respect to Gaussian random noise $\boldsymbol{\zeta}$ can be interchanged so that 
\begin{align}
\mathbb{E}_{\boldsymbol{\zeta}}  \mathrm{Tr} \left( \boldsymbol{\mathcal{S}}\textbf{A}^* \boldsymbol{\Pi} \partial_{\boldsymbol{y}}  \widehat{\boldsymbol{x}}(\boldsymbol{y} ; \boldsymbol{\Lambda}) \right) 
 &= \mathbb{E}_{\boldsymbol{\zeta}}  \lim\limits_{\nu \rightarrow 0} \mathbb{E}_{\boldsymbol{\varepsilon}}  \left\langle \textbf{A}^* \boldsymbol{\Pi}  \frac{\widehat{\boldsymbol{x}}(\boldsymbol{y} + \nu \boldsymbol{\varepsilon}; \boldsymbol{\Lambda}) - \widehat{\boldsymbol{x}}(\boldsymbol{y} ; \boldsymbol{\Lambda})}{\nu}, \boldsymbol{\mathcal{S}}\boldsymbol{\varepsilon}  \right\rangle \\ \nonumber
&=  \lim\limits_{\nu \rightarrow 0}  \mathbb{E}_{\boldsymbol{\zeta},\boldsymbol{\varepsilon}}  \left\langle \textbf{A}^* \boldsymbol{\Pi}  \frac{\widehat{\boldsymbol{x}}(\boldsymbol{y} + \nu \boldsymbol{\varepsilon}; \boldsymbol{\Lambda}) - \widehat{\boldsymbol{x}}(\boldsymbol{y} ; \boldsymbol{\Lambda})}{\nu}, \boldsymbol{\mathcal{S}}\boldsymbol{\varepsilon}  \right\rangle.
\end{align}
giving the asymptotic unbiasedness of the Finite Difference Monte Carlo estimator of degrees of freedom and hence of the  Finite Difference Monte Carlo SURE~\eqref{eq:SURE_FDMC_def}.
\end{proof}

\section{Finite Difference Monte Carlo SUGAR}
\label{app:SUGAR_FDMC}

\begin{proof}
We remind that Finite Difference Monte Carlo SUGAR is composed of two terms, denoted $(\partial \boldsymbol{1})$ and $(\partial \boldsymbol{2})$ in the following:
\begin{align}
\label{eq:app_sugar}
&\nablaL_{\boldsymbol{\Lambda}}\widehat{R}_{\nu,\boldsymbol{\epsilon}}(\boldsymbol{y}; \boldsymbol{\Lambda} \lvert \boldsymbol{\mathcal{S}}) \triangleq \\
 \underset{ \text{\normalsize $(\partial \boldsymbol{1})$}}{2 \left( \textbf{A} \boldsymbol{\Phi} \nablaL_{\boldsymbol{\Lambda}} \widehat{\boldsymbol{x}}(\boldsymbol{y}; \boldsymbol{\Lambda}) \right)^*  \textbf{A}\left( \boldsymbol{\Phi}\widehat{\boldsymbol{x}} - \boldsymbol{y} \right) } +   &\underset{\text{\normalsize $(\partial \boldsymbol{2})$}}{\frac{2}{\nu} \left\langle \textbf{A}^* \boldsymbol{\Pi} \left( \nablaL_{\boldsymbol{\Lambda}}\widehat{\boldsymbol{x}}(\boldsymbol{y} + \nu \boldsymbol{\varepsilon}; \boldsymbol{\Lambda}) - \nablaL_{\boldsymbol{\Lambda}}\widehat{\boldsymbol{x}}(\boldsymbol{y} ; \boldsymbol{\Lambda})\right), \boldsymbol{\mathcal{S}}\boldsymbol{\varepsilon}  \right\rangle}. \nonumber
\end{align}
Since the estimator $\widehat{\boldsymbol{x}}(\boldsymbol{y}; \boldsymbol{\Lambda})$ is weakly differentiable with respect to $\boldsymbol{\Lambda}$, so is the true risk $R[\widehat{\boldsymbol{x}}](\boldsymbol{\Lambda})$. 
Thus, for any continuously differentiable test function $\varphi : \mathbb{R}^L \rightarrow \mathbb{R} \in C^1(\mathbb{V})$ with compact support denoted $\mathbb{V} \subset \mathbb{R}^L$, and any component $l \in \lbrace1, \hdots , L\rbrace$ of the gradient of the risk $\nablaL_{\boldsymbol{\Lambda}} R[\widehat{\boldsymbol{x}}](\boldsymbol{\Lambda})$ 
\begin{align}
\int_{\mathbb{R}^L} \left( \nablaL_{\boldsymbol{\Lambda}} R[\widehat{\boldsymbol{x}}](\boldsymbol{\Lambda}) \right)_l \varphi(\boldsymbol{\Lambda}) \, \mathrm{d}\boldsymbol{\Lambda} &= \int_{\mathbb{V}} \left( \nablaL_{\boldsymbol{\Lambda}} R[\widehat{\boldsymbol{x}}](\boldsymbol{\Lambda}) \right)_l \varphi(\boldsymbol{\Lambda}) \, \mathrm{d}\boldsymbol{\Lambda}\\
 \text{(Weak differentiability)} & = - \int_{\mathbb{V}} R[\widehat{\boldsymbol{x}}](\boldsymbol{\Lambda})\left( \nablaL_{\boldsymbol{\Lambda}} \varphi(\boldsymbol{\Lambda}) \right)_l \, \mathrm{d}\boldsymbol{\Lambda} \nonumber  \\
\text{(Definition of the risk~\eqref{eq:risk_def})} &=- \int_{\mathbb{V}}  \mathbb{E}_{\boldsymbol{\zeta}}  \left\lVert \boldsymbol{\Pi}\widehat{\boldsymbol{x}}(\boldsymbol{y} ; \boldsymbol{\Lambda}) - \boldsymbol{\Pi}\boldsymbol{x} \right\rVert_2^2\left( \nablaL_{\boldsymbol{\Lambda}} \varphi(\boldsymbol{\Lambda}) \right)_l  \, \mathrm{d}\boldsymbol{\Lambda}  \nonumber\\
\text{(Theorem~\ref{thm:SURE})} & = - \int_{\mathbb{V}}  \mathbb{E}_{\boldsymbol{\zeta}, \boldsymbol{\varepsilon}}    \lim\limits_{\nu \rightarrow 0} \widehat{R}_{\nu, \boldsymbol{\varepsilon}}(\boldsymbol{y}; \boldsymbol{\Lambda}\lvert \boldsymbol{\mathcal{S}}) \left( \nablaL_{\boldsymbol{\Lambda}} \varphi(\boldsymbol{\Lambda}) \right)_l  \, \mathrm{d}\boldsymbol{\Lambda}  \nonumber\\
\text{(Theorem~\ref{thm:SURE_FDMC})} & = - \int_{\mathbb{V}} \lim\limits_{\nu \rightarrow 0} \mathbb{E}_{\boldsymbol{\zeta}, \boldsymbol{\varepsilon}}     \widehat{R}_{\nu, \boldsymbol{\varepsilon}}(\boldsymbol{y}; \boldsymbol{\Lambda}\lvert \boldsymbol{\mathcal{S}}) \left( \nablaL_{\boldsymbol{\Lambda}} \varphi(\boldsymbol{\Lambda}) \right)_l  \, \mathrm{d}\boldsymbol{\Lambda} \nonumber \\
\text{(Dominated convergence)} & \overset{\text{\normalsize\color{bleu}(DC 1)}}{=} - \lim\limits_{\nu \rightarrow 0} \int_{\mathbb{V}}  \mathbb{E}_{\boldsymbol{\zeta}, \boldsymbol{\varepsilon}}     \widehat{R}_{\nu, \boldsymbol{\varepsilon}} (\boldsymbol{y}; \boldsymbol{\Lambda}\lvert \boldsymbol{\mathcal{S}}) \left( \nablaL_{\boldsymbol{\Lambda}} \varphi(\boldsymbol{\Lambda}) \right)_l  \, \mathrm{d}\boldsymbol{\Lambda} \nonumber\\
\text{(Fubini)} & \overset{\text{\normalsize\color{bleu}(Fu 1)}}{=}  - \lim\limits_{\nu \rightarrow 0}   \mathbb{E}_{\boldsymbol{\zeta}, \boldsymbol{\varepsilon}}     \int_{\mathbb{V}} \widehat{R}_{\nu, \boldsymbol{\varepsilon}}(\boldsymbol{y}; \boldsymbol{\Lambda}\lvert \boldsymbol{\mathcal{S}}) \left( \nablaL_{\boldsymbol{\Lambda}} \varphi(\boldsymbol{\Lambda}) \right)_l  \, \mathrm{d}\boldsymbol{\Lambda} \nonumber\\
\text{(Proposition~\ref{claim:SUGAR})} & = \lim\limits_{\nu \rightarrow 0}   \mathbb{E}_{\boldsymbol{\zeta}, \boldsymbol{\varepsilon}}     \int_{\mathbb{V}}  \left( \nablaL_{\boldsymbol{\Lambda}} \widehat{R}_{\nu, \boldsymbol{\varepsilon}}(\boldsymbol{y}; \boldsymbol{\Lambda}\lvert \boldsymbol{\mathcal{S}}) \right)_l \varphi(\boldsymbol{\Lambda}) \, \mathrm{d}\boldsymbol{\Lambda} \nonumber\\
& \overset{ \text{\normalsize\color{bleu}(Fu 2)}}{=}  \lim\limits_{\nu \rightarrow 0}  \int_{\mathbb{V}}  \mathbb{E}_{\boldsymbol{\zeta}, \boldsymbol{\varepsilon}}      \left( \nablaL_{\boldsymbol{\Lambda}} \widehat{R}_{\nu, \boldsymbol{\varepsilon}}(\boldsymbol{y}; \boldsymbol{\Lambda}\lvert \boldsymbol{\mathcal{S}}) \right)_l \varphi(\boldsymbol{\Lambda}) \, \mathrm{d}\boldsymbol{\Lambda} \nonumber\\
& \overset{\text{\normalsize\color{bleu}(DC 2)}}{=}  \int_{\mathbb{V}} \lim\limits_{\nu \rightarrow 0}  \mathbb{E}_{\boldsymbol{\zeta}, \boldsymbol{\varepsilon}}      \left( \nablaL_{\boldsymbol{\Lambda}} \widehat{R}_{\nu, \boldsymbol{\varepsilon}}(\boldsymbol{y}; \boldsymbol{\Lambda}\lvert \boldsymbol{\mathcal{S}}) \right)_l \varphi(\boldsymbol{\Lambda}) \, \mathrm{d}\boldsymbol{\Lambda}. \nonumber
\end{align}

\noindent {\color{bleu}(DC 1)} In order to apply dominated convergence theorem interchanging the limit on $\nu$ and the integral on $\mathbb{V}$, since $\varphi$ is a test function with compact domain $\mathbb{V}$, we derive a bound of $\mathbb{E}_{\boldsymbol{\zeta}, \boldsymbol{\varepsilon}}     \widehat{R}_{\nu, \boldsymbol{\varepsilon}}(\boldsymbol{y}; \boldsymbol{\Lambda}\lvert \boldsymbol{\mathcal{S}})$ which is independent of both $\nu$ and $\boldsymbol{\Lambda}$. 
Using the probability density functions we have 
\begin{align}
\mathbb{E}_{\boldsymbol{\zeta}, \boldsymbol{\varepsilon}}     \widehat{R}_{\nu, \boldsymbol{\varepsilon}}(\boldsymbol{y}; \boldsymbol{\Lambda}\lvert \boldsymbol{\mathcal{S}})  = \int_{\boldsymbol{\zeta}} \int_{\boldsymbol{\varepsilon}} \, \widehat{R}_{\nu, \boldsymbol{\varepsilon}}(\boldsymbol{y}; \boldsymbol{\Lambda}\lvert \boldsymbol{\mathcal{S}})  \,\mathcal{G}_{\boldsymbol{\mathcal{S}}}(\boldsymbol{\zeta}) \mathcal{G}_{\boldsymbol{I}}(\boldsymbol{\varepsilon}) \mathrm{d}\boldsymbol{\zeta} \mathrm{d}\boldsymbol{\varepsilon}
\end{align}
where $\mathcal{G}_{\boldsymbol{\mathcal{S}}}$ (resp. $\mathcal{G}_{\boldsymbol{I}}$) denotes the Gaussian probability density function with covariance matrix $\boldsymbol{\mathcal{S}}$ (resp. $\boldsymbol{I}$)
\begin{align}
\mathcal{G}_{\boldsymbol{\mathcal{S}}}(\boldsymbol{\zeta}) \triangleq \frac{\exp\left( -\lVert \boldsymbol{\zeta} \rVert^2_{\boldsymbol{\mathcal{S}}^{-1}}/2 \right)}{\sqrt{(2\pi)^P \lvert \det \boldsymbol{\mathcal{S}}\rvert}}  \quad \left( \text{resp. } \mathcal{G}_{\boldsymbol{I}}(\boldsymbol{\zeta}) \triangleq \frac{\exp\left( -\lVert \boldsymbol{\zeta} \rVert^2_2/2 \right)}{\sqrt{(2\pi)^P}}  \right).
\end{align} 
We remind that $\widehat{R}_{\nu, \boldsymbol{\varepsilon}}(\boldsymbol{y}; \boldsymbol{\Lambda}\lvert \boldsymbol{\mathcal{S}})$ is decomposed of three terms 
\begin{align}
\label{eq:SURE_FDMC_App}
& \widehat{R}_{\nu, \boldsymbol{\varepsilon}}(\boldsymbol{y}; \boldsymbol{\Lambda}\lvert \boldsymbol{\mathcal{S}}) \triangleq \nonumber \\
&\underset{\textbf{\normalsize(1)}}{\left\lVert  \textbf{A}\left( \boldsymbol{\Phi}\widehat{\boldsymbol{x}} - \boldsymbol{y} \right) \right\rVert_2^2} + \underset{\textbf{\normalsize(2)}}{\frac{1}{\nu} \left\langle \textbf{A}^* \boldsymbol{\Pi} \left( \widehat{\boldsymbol{x}}(\boldsymbol{y} + \nu \boldsymbol{\varepsilon}; \boldsymbol{\Lambda}\lvert \boldsymbol{\mathcal{S}}) - \widehat{\boldsymbol{x}}(\boldsymbol{y} ; \boldsymbol{\Lambda}\lvert \boldsymbol{\mathcal{S}})\right), \boldsymbol{\mathcal{S}}\boldsymbol{\varepsilon}  \right\rangle}- \underset{\textbf{\normalsize(3)}}{\mathrm{Tr}(\textbf{A} \boldsymbol{\mathcal{S}} \textbf{A}^*  )} .
\end{align}
which will be bounded separately.\\
{\textbf{(1)}} First, combining Assumptions~\textit{(i)}~and~\textit{(ii)} of Assumption~\ref{hyp:lip_L1}, we have
\begin{align}
\left\lVert \widehat{\boldsymbol{x}}(\boldsymbol{y}; \boldsymbol{\Lambda}) - \widehat{\boldsymbol{x}}(\boldsymbol{0}_P; \boldsymbol{\Lambda}) \right\rVert \leq L_1  \left\lVert \boldsymbol{y} - \boldsymbol{0}_P\right\rVert \Longrightarrow \left\lVert \widehat{\boldsymbol{x}}(\boldsymbol{y}; \boldsymbol{\Lambda})  \right\rVert \leq L_1  \left\lVert \boldsymbol{y}\right\rVert,
\end{align}
which can be used to bound first term (1) of~\eqref{eq:SURE_FDMC_App} as 
\begin{align}
\label{eq:dom(1)}
\left\lVert  \textbf{A}\left( \boldsymbol{\Phi}\widehat{\boldsymbol{x}} - \boldsymbol{y} \right) \right\rVert &\leq \lVert \textbf{A} \rVert \left\lVert \boldsymbol{\Phi}\widehat{\boldsymbol{x}} - \boldsymbol{y} \right\rVert \nonumber \\
& \leq \lVert \textbf{A} \rVert \left( \left\lVert \boldsymbol{\Phi}\widehat{\boldsymbol{x}}\right\rVert  + \left\lVert \boldsymbol{y} \right\rVert \right)\nonumber\\
& \leq \lVert \textbf{A} \rVert \left( \lVert \boldsymbol{\Phi} \rVert  \lVert \widehat{\boldsymbol{x}}\rVert  + \left\lVert \boldsymbol{y} \right\rVert \right)\nonumber\\
& \leq \lVert \textbf{A} \rVert \left( \lVert \boldsymbol{\Phi} \rVert L_1 \lVert \boldsymbol{y}\rVert  + \left\lVert \boldsymbol{y} \right\rVert \right)\nonumber\\
& \leq \lVert \textbf{A} \rVert \left( \lVert \boldsymbol{\Phi} \rVert L_1 + 1\right)\left\lVert \boldsymbol{y} \right\rVert.
\end{align}
Since by definition $\boldsymbol{\zeta} = \boldsymbol{y} - \boldsymbol{\Phi} \boldsymbol{x}$, $\boldsymbol{y} \mapsto \left\lVert \boldsymbol{y} \right\rVert$ is integrable against the Gaussian density $\mathcal{G}_{\boldsymbol{\mathcal{S}}}(\boldsymbol{\zeta}) $, and the above domination being independent of $\nu$, it enable us to apply dominated convergence.\\
{ \textbf{(2)}} Making use of the domination of Equation~\eqref{eq:maj_varepsilon} we have
\begin{align}
&\left\lvert \left\langle \textbf{A}^* \boldsymbol{\Pi} \frac{\left( \widehat{\boldsymbol{x}}(\boldsymbol{y} + \nu \boldsymbol{\varepsilon}; \boldsymbol{\Lambda}) - \widehat{\boldsymbol{x}}(\boldsymbol{y} ; \boldsymbol{\Lambda})\right)}{\nu}, \boldsymbol{\mathcal{S}}\boldsymbol{\varepsilon}  \right\rangle \right\rvert \\
& \leq \lVert \textbf{A}^* \rVert  \lVert \boldsymbol{\Pi} \rVert  L_1 \lVert \boldsymbol{\varepsilon} \rVert \lVert \boldsymbol{\mathcal{S}} \rVert \lVert \boldsymbol{\varepsilon} \rVert,\nonumber
\end{align}
and $\lVert \boldsymbol{\varepsilon} \rVert^2$ being integrable against $\mathcal{G}_{\boldsymbol{I}}(\boldsymbol{\varepsilon})$, dominated convergence applied.\\
{ \textbf{(3)}} The third term being constant, the domination is obvious.\\
Putting altogether the majoration of {\textbf{(1)}}, {\textbf{(2)}}  and {\textbf{(3)}} 
\begin{align}
\label{eq:dom_DC1}
\left\lvert \mathbb{E}_{\boldsymbol{\zeta}, \boldsymbol{\varepsilon}}     \widehat{R}_{\nu, \boldsymbol{\varepsilon}}(\boldsymbol{y}; \boldsymbol{\Lambda}\lvert \boldsymbol{\mathcal{S}})  \right\rvert \nonumber 
&\leq \int_{\boldsymbol{\zeta}} \int_{\boldsymbol{\varepsilon}} \, \lVert \textbf{A} \rVert \left( \lVert \boldsymbol{\Phi} \rVert L_1 + 1\right)\left\lVert \boldsymbol{y} \right\rVert  \,\mathcal{G}_{\boldsymbol{\mathcal{S}}}(\boldsymbol{\zeta}) \mathcal{G}_{\boldsymbol{I}}(\boldsymbol{\varepsilon}) \mathrm{d}\boldsymbol{\zeta} \mathrm{d}\boldsymbol{\varepsilon} \nonumber \\
&+  \int_{\boldsymbol{\zeta}} \int_{\boldsymbol{\varepsilon}} \, \lVert \textbf{A}^* \rVert  \lVert \boldsymbol{\Pi} \rVert  L_1  \lVert \boldsymbol{\mathcal{S}} \rVert \lVert \boldsymbol{\varepsilon} \rVert^2  \,\mathcal{G}_{\boldsymbol{\mathcal{S}}}(\boldsymbol{\zeta}) \mathcal{G}_{\boldsymbol{I}}(\boldsymbol{\varepsilon}) \mathrm{d}\boldsymbol{\zeta} \mathrm{d}\boldsymbol{\varepsilon} \nonumber \\
&+ \int_{\boldsymbol{\zeta}} \int_{\boldsymbol{\varepsilon}} \, \mathrm{Tr}(\textbf{A} \boldsymbol{\mathcal{S}} \textbf{A}^*  ) \,\mathcal{G}_{\boldsymbol{\mathcal{S}}}(\boldsymbol{\zeta}) \mathcal{G}_{\boldsymbol{I}}(\boldsymbol{\varepsilon}) \mathrm{d}\boldsymbol{\zeta} \mathrm{d}\boldsymbol{\varepsilon} \leq \infty,
\end{align}
the majoration being independent of $\nu$ and $\boldsymbol{\Lambda}$ dominated convergence applies.\\

\noindent {\color{bleu}(Fu 1)} The above domination of Equation~\eqref{eq:dom_DC1} being independent of $\boldsymbol{\Lambda}$, then Fubini's theorem applies.\\

\noindent {\color{bleu}(Fu 2)} 
The first term of~\eqref{eq:app_sugar}, denoted as~$(\nablaL\textbf{1})$ can be easily dominated by an integrable function. Indeed, Assumption~\ref{hyp:lip_L2} 
implies that $\nablaL_{\boldsymbol{\Lambda}} \widehat{\boldsymbol{x}}(\boldsymbol{y}; \boldsymbol{\Lambda})$ is uniformly bounded by $L_2$, independently of $\boldsymbol{y}$. 
Then it follows from Cauchy-Schwarz inequality and the domination of Equation~\eqref{eq:dom(1)}
\begin{align}
\nonumber 2 \left\lVert \left( \textbf{A} \boldsymbol{\Phi} \nablaL_{\boldsymbol{\Lambda}} \widehat{\boldsymbol{x}}(\boldsymbol{y}; \boldsymbol{\Lambda}) \right)^*  \textbf{A}\left( \boldsymbol{\Phi}\widehat{\boldsymbol{x}} - \boldsymbol{y} \right) \right\rVert &\leq 2 \lVert \textbf{A} \boldsymbol{\Phi} \nablaL_{\boldsymbol{\Lambda}} \widehat{\boldsymbol{x}}(\boldsymbol{y}; \boldsymbol{\Lambda}) \rVert \lVert \textbf{A}\left( \boldsymbol{\Phi}\widehat{\boldsymbol{x}} - \boldsymbol{y} \right)  \rVert,\\
&\leq 2 \lVert \textbf{A} \rVert \lVert \boldsymbol{\Phi} \rVert L_2 \lVert \textbf{A} \rVert \left( \boldsymbol{\Phi}L_1 +1\right) \lVert \boldsymbol{y} \rVert .
\end{align}
Hence, since $\lVert \boldsymbol{y} \rVert$ is integrable against $\mathcal{G}_{\boldsymbol{\mathcal{S}}}(\boldsymbol{y} - \boldsymbol{\Phi}\boldsymbol{x}) \mathcal{G}_{\boldsymbol{I}}(\boldsymbol{\varepsilon})$ and the domination being independent of $\nu$, both Fubini and dominated convergence theorems apply.\\
 
The second term of~\eqref{eq:app_sugar}, denoted as~$(\nablaL \boldsymbol{2})$, corresponding to the derivative of the estimation of degrees of freedom, can be rewritten as
\begin{align*}
\frac{2}{\nu} \left\langle \textbf{A}^* \boldsymbol{\Pi} \left( \nablaL_{\boldsymbol{\Lambda}}\widehat{\boldsymbol{x}}(\boldsymbol{y} + \nu \boldsymbol{\varepsilon}; \boldsymbol{\Lambda}) - \nablaL_{\boldsymbol{\Lambda}}\widehat{\boldsymbol{x}}(\boldsymbol{y} ; \boldsymbol{\Lambda})\right), \boldsymbol{\mathcal{S}}\boldsymbol{\varepsilon}  \right\rangle & \triangleq \frac{2}{\nu}\left\langle u(\boldsymbol{\zeta} + \nu \boldsymbol{\varepsilon} ; \boldsymbol{\Lambda}) - u(\boldsymbol{\zeta} ; \boldsymbol{\Lambda}),  \boldsymbol{\varepsilon}  \right\rangle
\end{align*}
where we set $u(\boldsymbol{z};\boldsymbol{\Lambda}) \triangleq  \boldsymbol{\mathcal{S}}\textbf{A}^* \boldsymbol{\Pi} \nablaL_{\boldsymbol{\Lambda}}\widehat{\boldsymbol{x}}(\boldsymbol{\Phi} \boldsymbol{x} +\boldsymbol{z}  ; \boldsymbol{\Lambda}) $. Since $\nablaL_{\boldsymbol{\Lambda}} \widehat{\boldsymbol{x}}(\boldsymbol{y}; \boldsymbol{\Lambda})$ is uniformly bounded by $L_2$, independently of $\boldsymbol{y}$, and all the linear operators are assumed to be bounded, then $\boldsymbol{\Lambda} \mapsto u(\boldsymbol{z};\boldsymbol{\Lambda})$ is bounded by some $L_u > 0$, independently of $\boldsymbol{z}$. Then
\begin{align}
\label{eq:nablaL_dof}
&\mathbb{E}_{\boldsymbol{\zeta},\boldsymbol{\varepsilon}} \left[ \frac{2}{\nu} \left\langle \textbf{A}^* \boldsymbol{\Pi} \left( \nablaL_{\boldsymbol{\Lambda}}\widehat{\boldsymbol{x}}(\boldsymbol{y} + \nu \boldsymbol{\varepsilon}; \boldsymbol{\Lambda}) - \nablaL_{\boldsymbol{\Lambda}}\widehat{\boldsymbol{x}}(\boldsymbol{y} ; \boldsymbol{\Lambda})\right), \boldsymbol{\mathcal{S}}\boldsymbol{\varepsilon}  \right\rangle \right] \\
 = \int_{\boldsymbol{\zeta}} \int_{\boldsymbol{\varepsilon}} &\frac{2}{\nu}\left\langle u(\boldsymbol{\zeta} + \nu \boldsymbol{\varepsilon} ; \boldsymbol{\Lambda}) - u(\boldsymbol{\zeta} ; \boldsymbol{\Lambda}),  \boldsymbol{\varepsilon}  \right\rangle \mathcal{G}_{\boldsymbol{\mathcal{S}}}(\boldsymbol{\zeta}) \mathcal{G}_{\boldsymbol{I}}(\boldsymbol{\varepsilon}) \, \mathrm{d}\boldsymbol{\zeta} \mathrm{d}\boldsymbol{\varepsilon}\nonumber\\
 = \int_{\boldsymbol{\zeta}} \int_{\boldsymbol{\varepsilon}} &\frac{2}{\nu}\left\langle u(\boldsymbol{\zeta} + \nu \boldsymbol{\varepsilon} ; \boldsymbol{\Lambda}) , \boldsymbol{\varepsilon}  \right\rangle \mathcal{G}_{\boldsymbol{\mathcal{S}}}(\boldsymbol{\zeta}) \mathcal{G}_{\boldsymbol{I}}(\boldsymbol{\varepsilon}) \, \mathrm{d}\boldsymbol{\zeta} \mathrm{d}\boldsymbol{\varepsilon} - \int_{\boldsymbol{\zeta}} \int_{\boldsymbol{\varepsilon}} \frac{2}{\nu}\left\langle u(\boldsymbol{\zeta}  ; \boldsymbol{\Lambda}),  \boldsymbol{\varepsilon}  \right\rangle \mathcal{G}_{\boldsymbol{\mathcal{S}}}(\boldsymbol{\zeta}) \mathcal{G}_{\boldsymbol{I}}(\boldsymbol{\varepsilon}) \, \mathrm{d}\boldsymbol{\zeta} \mathrm{d}\boldsymbol{\varepsilon}\nonumber\\
= \int_{\boldsymbol{\zeta}} \int_{\boldsymbol{\varepsilon}} &\frac{2}{\nu}\left\langle u(\boldsymbol{\zeta}  ; \boldsymbol{\Lambda}),  \boldsymbol{\varepsilon}  \right\rangle \left( \mathcal{G}_{\boldsymbol{\mathcal{S}}}(\boldsymbol{\zeta} - \nu \boldsymbol{\varepsilon}) - \mathcal{G}_{\boldsymbol{\mathcal{S}}}(\boldsymbol{\zeta}) \right) \mathcal{G}_{\boldsymbol{I}}(\boldsymbol{\varepsilon}) \, \mathrm{d}\boldsymbol{\zeta} \mathrm{d}\boldsymbol{\varepsilon}.\nonumber
\end{align}
Further, the following majoration holds
\begin{align}
\left\lVert\frac{2}{\nu}\left\langle u(\boldsymbol{\zeta}  ; \boldsymbol{\Lambda}),  \boldsymbol{\varepsilon}  \right\rangle \left( \mathcal{G}_{\boldsymbol{\mathcal{S}}}(\boldsymbol{\zeta} - \nu \boldsymbol{\varepsilon}) - \mathcal{G}_{\boldsymbol{\mathcal{S}}}(\boldsymbol{\zeta}) \right) \mathcal{G}_{\boldsymbol{I}}(\boldsymbol{\varepsilon}) \right\rVert & \leq \frac{2}{\nu} L_u \lVert \boldsymbol{\varepsilon} \rVert \left\lvert  \mathcal{G}_{\boldsymbol{\mathcal{S}}}(\boldsymbol{\zeta} - \nu \boldsymbol{\varepsilon}) - \mathcal{G}_{\boldsymbol{\mathcal{S}}}(\boldsymbol{\zeta}) \right\rvert \mathcal{G}_{\boldsymbol{I}}(\boldsymbol{\varepsilon}).
\end{align}
Up to a unitary variable change (see Appendix~\ref{app:SURE}, with $\boldsymbol{\vartheta} = \boldsymbol{\mathcal{V}}^{-1} \boldsymbol{\varepsilon}$), we can assume that the covariance matrix is diagonal, with diagonal terms $\left( s_i^2 \right)_{i = 1}^P$ so that 
\begin{align}
\mathcal{G}_{\boldsymbol{\mathcal{S}}}(\boldsymbol{\zeta}) = \prod_{i  = 1}^P \frac{\exp\left(- \lvert \zeta_i \rvert^2/2 s_i^2 \right)}{\sqrt{2\pi s_i^2}}
\end{align}
and we define the one-dimensional Gaussian densities as
\begin{align}
g_{s_i^2}(\zeta_i) \triangleq \frac{\exp\left(- \lvert \zeta_i \rvert^2/2 s_i^2 \right)}{\sqrt{2\pi s_i^2}}.
\end{align} 
From Taylor inequality,
\begin{align}
\lvert g_{s_i^2}(\zeta_i-\nu \varepsilon_i) -  g_{s_i^2}(\zeta_i) \rvert \leq \int_{(0, \nu  \varepsilon_i)} \lvert g_{s_i^2}'(\zeta_i - \tau) \rvert \, \mathrm{d}\tau,
\end{align}
where $(0, \nu \varepsilon_i)$ denotes the ordered interval, taking into account that $\varepsilon_i$ might be negative that is
\begin{align}
(0,\nu \varepsilon_i) = \left\lbrace \begin{array}{cc}
\left[0, \nu \varepsilon_i\right] & \text{ if }\varepsilon_i \geq 0 \\
\left[ \nu \varepsilon_i, 0\right] & \text{ else }
\end{array}\right.
\end{align}
then 
\begin{align}
\int_{\zeta_i} \lvert g_{s_i^2}(\zeta_i-\nu \varepsilon_i) -  g_{s_i^2}(\zeta_i) \rvert \, \mathrm{d}\zeta_i 
&\leq \int_{\zeta_i} \int_{(0, \nu  \varepsilon_i)} \lvert g_{s_i^2}'(\zeta_i - \tau) \rvert \, \mathrm{d}\tau \mathrm{d}\zeta_i \nonumber\\
& \leq  \int_{(0, \nu  \varepsilon_i)} \left( \int_{\zeta_i} \lvert g_{s_i^2}'(\zeta_i - \tau) \rvert \, \mathrm{d}\zeta_i \right) \, \mathrm{d}\tau \nonumber \\ 
& = \left( \int_{\mathbb{R}} \lvert g_{s_i^2}'(t) \rvert \, \mathrm{d}t \right) \nu \lvert \varepsilon_i \rvert < + \infty\nonumber
\end{align}
since the derivative of the Gaussian density is integrable over $\mathbb{R}$.\\
Going back to the integrals over variables $\boldsymbol{\zeta}, \boldsymbol{\varepsilon} \in \mathbb{R}^P$ of Equation~\eqref{eq:nablaL_dof}
\begin{align}
&\left\lVert \int_{\boldsymbol{\zeta}} \int_{\boldsymbol{\varepsilon}} \frac{2}{\nu}\left\langle u(\boldsymbol{\zeta}  ; \boldsymbol{\Lambda}),  \boldsymbol{\varepsilon}  \right\rangle \left( \mathcal{G}_{\boldsymbol{\mathcal{S}}}(\boldsymbol{\zeta} - \nu \boldsymbol{\varepsilon}) - \mathcal{G}_{\boldsymbol{\mathcal{S}}}(\boldsymbol{\zeta}) \right) \mathcal{G}_{\boldsymbol{I}}(\boldsymbol{\varepsilon}) \, \mathrm{d}\boldsymbol{\zeta} \mathrm{d}\boldsymbol{\varepsilon} \right\rVert\\
& \leq \int_{\boldsymbol{\zeta}} \int_{\boldsymbol{\varepsilon}} \frac{2}{\nu} L_u \lVert \boldsymbol{\varepsilon} \rVert \left\lvert  \mathcal{G}_{\boldsymbol{\mathcal{S}}}(\boldsymbol{\zeta} - \nu \boldsymbol{\varepsilon}) - \mathcal{G}_{\boldsymbol{\mathcal{S}}}(\boldsymbol{\zeta}) \right\rvert \mathcal{G}_{\boldsymbol{I}}(\boldsymbol{\varepsilon}) \, \mathrm{d}\boldsymbol{\zeta} \mathrm{d}\boldsymbol{\varepsilon}\nonumber\\
& =   \int_{\boldsymbol{\varepsilon}} \frac{2}{\nu} L_u \lVert \boldsymbol{\varepsilon} \rVert \prod_{i = 1}^P \left(  \int_{\zeta_i} \lvert g_{s_i^2}(\zeta_i-\nu \varepsilon_i) -  g_{s_i^2}(\zeta_i) \rvert \, \mathrm{d}\zeta_i \right) \mathcal{G}_{\boldsymbol{I}}(\boldsymbol{\varepsilon}) \, \mathrm{d}\boldsymbol{\varepsilon}\nonumber\\
& \leq \int_{\boldsymbol{\varepsilon}} \frac{2}{\nu} L_u \lVert \boldsymbol{\varepsilon} \rVert \prod_{i = 1}^P \left( \left( \int_{\mathbb{R}} \lvert g_{s_i^2}'(t) \rvert \, \mathrm{d}t \right) \nu \lvert \varepsilon_i \rvert  \right) \mathcal{G}_{\boldsymbol{I}}(\boldsymbol{\varepsilon}) \, \mathrm{d}\boldsymbol{\varepsilon}\nonumber\\
\lvert \varepsilon_i \rvert \leq \lVert \boldsymbol{\varepsilon} \rVert \quad &\leq \int_{\boldsymbol{\varepsilon}} 2 \nu^{P-1}  \lVert  L_u \boldsymbol{\varepsilon} \rVert \lVert \boldsymbol{\varepsilon}^P \rVert \prod_{i = 1}^P  \left( \int_{\mathbb{R}} \lvert g_{s_i^2}'(t) \rvert \, \mathrm{d}t    \right) \mathcal{G}_{\boldsymbol{I}}(\boldsymbol{\varepsilon}) \, \mathrm{d}\boldsymbol{\varepsilon}\nonumber\\
(0  < \nu \leq 1) \quad &\leq \int_{\boldsymbol{\varepsilon}} 2  L_u   \lVert \boldsymbol{\varepsilon} \rVert^{P+1} \prod_{i = 1}^P \left(  \int_{\mathbb{R}} \lvert g_{s_i^2}'(t) \rvert \, \mathrm{d}t \right)   \mathcal{G}_{\boldsymbol{I}}(\boldsymbol{\varepsilon}) \, \mathrm{d}\boldsymbol{\varepsilon} < + \infty \nonumber
\end{align}
Indeed, $\lVert \boldsymbol{\cdot} \rVert$ being the euclidean norm $(\forall i) \, \lvert \varepsilon_i \rvert \leq \lVert \boldsymbol{\varepsilon} \rVert$. 
Moreover, since we are interested in the limit $\nu \rightarrow 0$, we can assume without loss of generality that $0 >\nu \leq 1$ and thus $\nu^{P-1} \leq 1$. 
We conclude using the fact that any power of $  \lVert \boldsymbol{\varepsilon} \rVert$ is integrable against $\mathcal{G}_{\boldsymbol{I}}(\boldsymbol{\varepsilon})$, combined to the fact that the support $\mathbb{V}$ of $\varphi$ is compact, which enable to apply Fubini's theorem to exchange $\int_{\mathbb{V}}$ and $\mathbb{E}_{\boldsymbol{\zeta},\boldsymbol{\varepsilon}}$.\\

\noindent {\color{bleu}(DC 2)} The above majoration does not depends on $\nu$. Further, the Lipschitzianity assumptions provides the existence P-$a.s.$ of 
\begin{align*}
\lim\limits_{\nu \rightarrow 0}  \mathbb{E}_{\boldsymbol{\zeta}, \boldsymbol{\varepsilon}}      \left( \nablaL_{\boldsymbol{\Lambda}} \widehat{R}_{\nu, \boldsymbol{\varepsilon}}(\boldsymbol{y}; \boldsymbol{\Lambda}\lvert \boldsymbol{\mathcal{S}}) \right)_l
\end{align*} 
then dominated convergence theorem applies to invert $\lim\limits_{\nu \rightarrow 0}$ and $\int_{\mathbb{V}}$ which completes the proof.

\end{proof}

\clearpage

\section{Constant term of Stein Unbiased Risk Estimate}
\label{app:const_term}
\begin{proof}
Because of the cyclicality of the trace, one has $\mathrm{Tr}(\textbf{A} \boldsymbol{\mathcal{S}} \textbf{A}^*  )= \mathrm{Tr}( \textbf{A}^* \textbf{A} \boldsymbol{\mathcal{S}} )$, then using the definition of $\textbf{A} \triangleq \boldsymbol{\Pi} \left( \boldsymbol{\Phi}^* \boldsymbol{\Phi}\right)^{-1} \boldsymbol{\Phi}^*$,
\begin{align*}
&\textbf{A}^*  \textbf{A} =\boldsymbol{\Phi}  \left( \boldsymbol{\Phi}^* \boldsymbol{\Phi}\right)^{-1}\boldsymbol{\Pi}^* \boldsymbol{\Pi} \left( \boldsymbol{\Phi}^* \boldsymbol{\Phi}\right)^{-1} \boldsymbol{\Phi}^*\\
\end{align*}
then turning to the block-matrix form to perform the products
\begin{align*}
&\left( \boldsymbol{\Phi}^* \boldsymbol{\Phi}\right)^{-1}\boldsymbol{\Pi}^* \boldsymbol{\Pi} \left( \boldsymbol{\Phi}^* \boldsymbol{\Phi}\right)^{-1}  \\
& = \frac{1}{\left( F_0 F_2 - F_1^2 \right)^2} \begin{pmatrix}
 F_0 \boldsymbol{I}_{N/2} & - F_1 \boldsymbol{I}_{N/2} \\
- F_1 \boldsymbol{I}_{N/2} & F_2 \boldsymbol{I}_{N/2} 
  \end{pmatrix} \begin{pmatrix}
  \boldsymbol{I}_{N/2} & \boldsymbol{Z}_{N/2}\\
  \boldsymbol{Z}_{N/2} & \boldsymbol{Z}_{N/2}
  \end{pmatrix} \begin{pmatrix}
 F_0 \boldsymbol{I}_{N/2} & - F_1 \boldsymbol{I}_{N/2} \\
- F_1 \boldsymbol{I}_{N/2} & F_2 \boldsymbol{I}_{N/2} 
  \end{pmatrix} \\
  &= \frac{1}{\left( F_0 F_2 - F_1^2 \right)^2} \begin{pmatrix}
 F_0 \boldsymbol{I}_{N/2} & - F_1 \boldsymbol{I}_{N/2} \\
- F_1 \boldsymbol{I}_{N/2} & F_2 \boldsymbol{I}_{N/2} 
  \end{pmatrix} \begin{pmatrix}
  F_0\boldsymbol{Z}_{N/2} & -F_1 \boldsymbol{I}_{N/2}\\
\boldsymbol{Z}_{N/2} &  \boldsymbol{Z}_{N/2} 
  \end{pmatrix}\\
  &= \frac{1}{\left( F_0 F_2 - F_1^2 \right)^2} \begin{pmatrix}
 F_0^2 \boldsymbol{I}_{N/2} & - F_0 F_1 \boldsymbol{I}_{N/2} \\
- F_0F_1 \boldsymbol{I}_{N/2} & F_1^2 \boldsymbol{I}_{N/2} 
  \end{pmatrix} .
\end{align*}
Using again cyclicality of the trace 
\begin{align*}
\mathrm{Tr}(\textbf{A} \boldsymbol{\mathcal{S}} \textbf{A}^*  ) &= \mathrm{Tr} \left( \boldsymbol{\Phi}  \left( \boldsymbol{\Phi}^* \boldsymbol{\Phi}\right)^{-1}\boldsymbol{\Pi}^* \boldsymbol{\Pi} \left( \boldsymbol{\Phi}^* \boldsymbol{\Phi}\right)^{-1} \boldsymbol{\Phi}^* \boldsymbol{\mathcal{S}} \right)\\
&= \mathrm{Tr} \left( \boldsymbol{\Phi}  \frac{1}{\left( F_0 F_2 - F_1^2 \right)^2}\begin{pmatrix}
 F_0^2 \boldsymbol{I}_{N/2} & - F_0 F_1 \boldsymbol{I}_{N/2} \\
- F_0F_1 \boldsymbol{I}_{N/2} & F_1^2 \boldsymbol{I}_{N/2} 
  \end{pmatrix} \boldsymbol{\Phi}^* \boldsymbol{\mathcal{S}} \right)\\
  &= \mathrm{Tr} \left(   \frac{1}{\left( F_0 F_2 - F_1^2 \right)^2}\begin{pmatrix}
 F_0^2 \boldsymbol{I}_{N/2} & - F_0 F_1 \boldsymbol{I}_{N/2} \\
- F_0F_1 \boldsymbol{I}_{N/2} & F_1^2 \boldsymbol{I}_{N/2} 
  \end{pmatrix}  \boldsymbol{\Phi}^* \boldsymbol{\mathcal{S}}\boldsymbol{\Phi}\right)\\
\end{align*}

Then using the action of $\boldsymbol{\Phi}$ and $\boldsymbol{\Phi}^*$, explicited in Formula~\eqref{eq:def_Phi} we have the matrix representation
\begin{align}
\boldsymbol{\Phi} = \begin{pmatrix}
 1 \boldsymbol{I}_{N/2} &  \boldsymbol{I}_{N/2}\\
  2 \boldsymbol{I}_{N/2} &  \boldsymbol{I}_{N/2}\\
  \vdots & \vdots\\
J   \boldsymbol{I}_{N/2} & \boldsymbol{I}_{N/2}\\
\end{pmatrix} \in \mathbb{R}^{JN_1N_2 \times 2N_1N2}.
\end{align}  
Using of the decomposition of $\boldsymbol{\mathcal{S}}$ into $J^2$ blocks $\boldsymbol{\mathcal{S}}_{j}^{j'} = \mathcal{C}_j^{j'} \boldsymbol{\Xi}_j^{j'}\in \mathbb{R}^{N/2\times N/2}$, we obtain
\begin{align}
\boldsymbol{\Phi}^* \boldsymbol{\mathcal{S}} \boldsymbol{\Phi} = \begin{pmatrix}
\sum_{j,j' } j j' \boldsymbol{\mathcal{S}}_{j}^{j'} & \sum_{j,j' } j' \boldsymbol{\mathcal{S}}_{j}^{j'}\\
\sum_{j,j' } j\boldsymbol{\mathcal{S}}_{j}^{j'}& \sum_{j,j' }  \boldsymbol{\mathcal{S}}_{j}^{j'}
\end{pmatrix}, \quad 1 \leq j, j'\leq J.
\end{align}
It follows
\begin{align}
&\mathrm{Tr}(\textbf{A} \boldsymbol{\mathcal{S}} \textbf{A}^*  ) \nonumber\\
&= \mathrm{Tr}\left(\frac{1}{\left( F_0 F_2 - F_1^2 \right)^2} \begin{pmatrix}
 F_0^2 \boldsymbol{I}_{N/2} & - F_0 F_1 \boldsymbol{I}_{N/2} \\
- F_0 F_1 \boldsymbol{I}_{N/2} & F_1^2 \boldsymbol{I}_{N/2} 
  \end{pmatrix} \begin{pmatrix}
\sum_{j,j' }  j j' \boldsymbol{\mathcal{S}}_{j}^{j'} & \sum_{j,j' } j' \boldsymbol{\mathcal{S}}_{j}^{j'}\\
\sum_{j,j' } j\boldsymbol{\mathcal{S}}_{j}^{j'} & \sum_{j,j' } \boldsymbol{\mathcal{S}}_{j}^{j'}
\end{pmatrix} \right)\nonumber\\
&=\frac{1}{\left( F_0 F_2 - F_1^2 \right)^2}  \mathrm{Tr}\begin{pmatrix}
\sum_{j,j' } F_0^2 j j' \boldsymbol{\mathcal{S}}_{j}^{j'} - F_0 F_1  j' \boldsymbol{\mathcal{S}}_{j}^{j'} & \sum_{j,j' }  F_0^2  j' \boldsymbol{\mathcal{S}}_{j}^{j'}  - F_0 F_1  \boldsymbol{\mathcal{S}}_{j}^{j'}\\
\sum_{j,j' } F_1^2 j\boldsymbol{\mathcal{S}}_{j}^{j'} - F_0 F_1 j j' \boldsymbol{\mathcal{S}}_{j}^{j'}  & \sum_{j,j' }  F_1^2   \boldsymbol{\mathcal{S}}_{j}^{j'}  - F_0 F_1   j'\boldsymbol{\mathcal{S}}_{j}^{j'}
\end{pmatrix} \nonumber \\
&= \frac{1}{\left( F_0 F_2 - F_1^2 \right)^2} \mathrm{Tr}\left(\sum_{j,j' } \left(F_0^2  j j' \boldsymbol{\mathcal{S}}_{j}^{j'} - F_0F_1  j' \boldsymbol{\mathcal{S}}_{j}^{j'} + F_1^2 \boldsymbol{\mathcal{S}}_{j}^{j'}   - F_0F_1   j'\boldsymbol{\mathcal{S}}_{j}^{j'} \right)\right).
\end{align}
Then, one can remark that 
\begin{align}
\mathrm{Tr}(\boldsymbol{\mathcal{S}}_{j}^{j'}) \triangleq \sum_{\underline{n} \in \Omega} \mathcal{S}_{j, \underline{n}}^{j', \underline{n}} = \sum_{\underline{n} \in \Omega} \mathcal{C}_j^{j'} = \frac{N}{2} \mathcal{C}_j^{j'}.
\end{align}
since the filter $\boldsymbol{\Xi}_j^{j'}$ is supposed to be normalized, in the sense that its maximum value equals 1.
Finally 
\begin{align}
&\mathrm{Tr}(\textbf{A} \boldsymbol{\mathcal{S}} \textbf{A}^*  ) = \frac{N/2}{\left( F_0 F_2 - F_1^2 \right)^2} \sum_{j,j' } \left(F_0^2  j j' \mathcal{C}_j^{j'} - F_0F_1  j' \mathcal{C}_j^{j'}+ F_1^2 \mathcal{C}_j^{j'}  - F_0F_1   j'\mathcal{C}_j^{j'}\right). 
\end{align}
\end{proof}

\bibliographystyle{plain}
\bibliography{abbr,siims20}

\end{document}